%% file: main.tex
\documentclass[10pt]{article}
\usepackage{fullpage}
\usepackage{microtype}
\usepackage[utf8]{inputenc}
\usepackage[T1]{fontenc}
\usepackage{lmodern}
\usepackage{natbib}
\usepackage{tabularx}

\usepackage{stmaryrd}
\usepackage{amsmath}
\usepackage{amsfonts}
\usepackage{amssymb}
\usepackage{dsfont}
\usepackage{euscript}
\usepackage{enumitem}
\usepackage{algpseudocode}
\usepackage{algorithm}

\usepackage{xcolor}

\usepackage{tikz}
\usepackage{tikz-cd}
\usepackage{mathtools}
\usepackage[bibliography=common]{apxproof}
\usepackage[normalem]{ulem}
\usepackage[english]{babel}
\usepackage{soul}

\input{commands}

\newcommand{\la}{\lambda}

\newcommand{\norm}[2][K]{\left\|#2\right\|_{#1}}
\newcommand{\nor}[1]{\left\|#1\right\|}

\newcommand{\gnorm}[1]{\left\|#1\right\|}

\newcommand{\inner}[2]{\left\langle#1,#2\right\rangle}
\providecommand{\scal}[2]{\left\langle{#1},{#2}\right\rangle}

\newcommand{\tr}[1]{\operatorname{Tr}\left[ #1 \right]}

\newcommand{\cA}{\mathcal{A}}
\newcommand{\cB}{\mathcal{B}}

\newcommand{\cE}{\mathcal{E}}
\newcommand{\cF}{\mathcal{F}}
\newcommand{\cG}{\mathcal{G}}
\newcommand{\cH}{\mathcal{H}}
\newcommand{\cI}{\mathcal{I}}

\newcommand{\cL}{\mathcal{L}}
\newcommand{\cM}{\mathcal{M}}
\newcommand{\cN}{\mathcal{N}}
\newcommand{\cP}{\mathcal{P}}

\newcommand{\cR}{\mathcal{R}}

\newcommand{\cU}{\mathcal{U}}

\newcommand{\cX}{\mathcal{X}}
\newcommand{\cY}{\mathcal{Y}}

\newcommand{\wh}{\widehat}

\newcommand{\bbE}{\mathbb{E}}
\newcommand{\bbP}{\mathbb{P}}

\usepackage[pdftex,colorlinks=true,linkcolor=blue, citecolor=blue,]{hyperref}
\usepackage{cleveref}

\theoremstyle{plain}
\newtheorem{theorem}{Theorem}[section]
\newtheorem{proposition}[theorem]{Proposition}
\newtheorem{lemma}[theorem]{Lemma}
\newtheorem{corollary}[theorem]{Corollary}
\newtheorem{remark}[theorem]{Remark}

\theoremstyle{definition}
\newtheorem{definition}[theorem]{Definition}
\newtheorem{assumption}[theorem]{Assumption}
\newtheorem{example}[theorem]{Example}

\theoremstyle{remark}

\newcommand{\be}{\begin{equation}}
\newcommand{\ee}{\end{equation}}

\usepackage{authblk}

\title{Learning ergodic dynamical systems from a finite trajectory}
\author[1]{Oleksii Kachaiev}
\author[1]{Silvia Villa}
\author[2,3]{Lorenzo Rosasco}

\affil[1]{MaLGa center, DIMA, Università degli Studi di Genova, Genoa, Italy}
\affil[2]{Istituto Italiano di Tecnologia, Genoa, Italy}
\affil[3]{MaLGa center, DIBRIS, Università degli Studi di Genova, Genoa, Italy}

\date{}

\begin{document}

\maketitle
\begin{abstract}
We consider the problem of learning from a single finite trajectory of an ergodic stochastic dynamical system. More precisely, we study discrete-time autonomous stochastic systems {defining} time-homogeneous Markov processes. We first focus on estimating the optimal {one-step prediction} function by nonlinear least squares, and derive high-probability guarantees measured with respect to the invariant measure of the process. These results make explicit how the non-independent and non-identically distributed nature of trajectory data modifies the classical statistical learning analysis. We then extend the framework to higher-order systems and finite-state spaces. Finally, we show that the same least squares and concentration arguments naturally extend to learning Koopman operators. Our approach combines tools from statistical learning theory and quantitative ergodic theory for Markov chains. It relies, in particular, on a concentration inequality for Hilbert-space-valued additive functionals of uniformly geometrically ergodic Markov chains.
\end{abstract}

\tableofcontents

\section{Introduction}

Learning from sequential data is a central problem, with applications spanning the analysis, identification, and control of dynamical systems \citep{brunton}, time-series forecasting \citep{lim2021time}, and language modeling \citep{bengio2003neural}, to name a few examples.
Sequential data pose a  challenge to classical statistical learning theory, which is developed primarily for independent and identically distributed (i.i.d.) data.
While theoretical frameworks avoiding explicit modeling assumptions are possible \citep{lugosicesa,orabona}, in this paper we postulate an explicit data-generating process in the form of a dynamical system. We further assume that the data are given by a single finite trajectory of the system. One objective of our study is to highlight the main challenges of this setting compared with classical statistical learning theory. Most works departing from the i.i.d. setting focus on relaxing independence, while often assuming stationarity. In these studies, dependence is quantified in terms of suitable mixing conditions, which enable the concentration results needed to derive learning guarantees; see, e.g., \citep{blanchard2019concentration} and Section~\ref{rel_wrk} for further references.

Instead, we consider ergodic systems and explicitly account for the transient phase before convergence to stationarity. The key observation enabling our study is that stochastic dynamical systems naturally induce Markov processes, providing a framework in which learning from a single trajectory can be analyzed. To make the comparison with classical statistical learning transparent, we begin with the {one-step prediction} problem for real-valued system states. While simplified, this setting is closest to classical supervised learning and already allows us to tackle the main statistical issues.
In this context, we study a nonlinear least squares estimator over possibly nonparametric model classes. We derive both universal consistency and finite-sample bounds, building on the rich literature on kernel methods, from foundational works \citep{devito2005learning,devito2005model,smale2005shannon,caponnetto2007optimal,smale2007learning} to more recent contributions; see, e.g., \citep{blanchard2018optimal,fischer2020sobolev,li2024towards,zhang2025optimal}.
We consider  natural assumptions adapted from supervised learning while also imposing assumptions specific to the dynamical-systems/Markov-process setting.
Notably, uniform geometric ergodicity (UGE) allows us to obtain non-asymptotic guarantees.
The main technical contribution is a set of concentration results for Markov processes, including a covariance-operator concentration bound.
These results leverage concentration inequalities for Hilbert-space-valued martingales \citep{pinelis1994optimum}, combined with a martingale decomposition enabled by the existence of a solution to the Poisson equation for Markov processes satisfying UGE.
This approach extends analogous real-valued results to Hilbert-space-valued random variables; see \citep{glynn2002hoeffding,gordin1969central}.

Using real-valued state {one-step prediction} as a starting point, we consider a number of extensions.
We study the more practical question of one-step prediction of vector-valued states and show it can be treated with minimal adjustments to the scalar-state framework and the corresponding analysis. Infinite-dimensional state spaces are also covered naturally. We then study higher-order Markov processes, where the transition to the next state depends on the previous \( p \) states rather than only on the current one. This scenario, which is close to order-\( p \) autoregressive models \citep{brockwell}, can be analyzed under a natural UGE condition after rewriting the process as a first-order Markov process on \( p \)-tuples of consecutive states. We then consider the special case of {one-step prediction} on finite-state spaces, which naturally relates to next-token prediction problems relevant to simple language models \citep{shannon1948mathematical,manning1999foundations}.
Here, we approach the problem through a least squares surrogate approach, akin to classical approaches to multi-category classification \citep{bartlett2006convexity,mroueh2012multiclass} and structured prediction \citep{ciliberto2020general,nowak2019sharp}. While many developments are possible in this setting, we derive and discuss some basic illustrative estimates and results.

Finally, beyond {one-step prediction}, we discuss how the approach and analysis developed can be extended to learning the Koopman operator from a finite trajectory.
The least squares approach we consider is equivalent to kernel Koopman regression, as studied for example in \citep{kostic2022learning}.
Our derivation is slightly different, and perhaps simpler, since it is based only on feature maps \citep{scholkopf2002learning} avoiding introducing the corresponding reproducing kernel Hilbert spaces (RKHS). Again, under the UGE assumption, we adapt the {one-step prediction} analysis to derive learning guarantees in operator norm. These estimates can be contrasted with previous results derived for i.i.d. or stationary data; see e.g. \citep{kostic2022learning} and Section~\ref{rel_wrk}.
We note that those works also consider norms other than the operator norm, possibly to derive spectral learning guarantees, see e.g. \citep{kostic2023sharp}.
While we focus here on operator-norm estimates, other norms could be considered adapting standard results from supervised learning, possibly under further assumptions. More generally, throughout the paper we strive to simplify the proofs and exposition, highlighting the main challenges in the dynamical systems scenario. We intentionally defer to future work a number of possible developments that are straightforward using known results in classic statistical learning theory for i.i.d. data but somewhat involved.

Before summarizing our contributions and describing the plan of the paper, we provide a detailed overview of relevant works.

\subsection{Related work}\label{rel_wrk}
In this section we provide a broader overview of works related to our study.
As noted above, a key aspect is the study of learning problems beyond the i.i.d. setting.
We begin by reviewing works that relax independence while retaining stationarity.

\paragraph{Learning from dependent data.} The topic is mostly built around the notion of \emph{mixing}, i.e., the mechanism which governs random variables to become more independent or decorrelated going into the future \citep{bradley1986basic,bradley2005basic}. Consistency and generalization properties under various mixing conditions using the empirical risk minimization (ERM) framework are studied in \citet{white1984nonlinear,zou2009generalization,steinwart2009learning,steinwart2009fast,hang2016learning,hang2017bernstein,roy2021empirical,NEURIPS2022_1dc9fbdb}. \citet{agarwal2012generalization} derives the bound for online learning from dependent data. Methodologically close to our work, \citet{blanchard2019concentration} presents learning sampling rate for kernel ridge regression (KRR) under \( \mathcal C \)-mixing assumption via kernel integral operator. Moving beyond mixing-type assumptions, \citet{massiani2024consistency} establish asymptotic consistency of kernel methods under empirical weak convergence. The notion of mixing is classically defined for stationary processes. Generalization results for non-stationary processes are studied in \citet{kuznetsov2017generalization}.

\paragraph{Nonparametric estimation from time series.} Nonparametric prediction and estimation for dependent data in the context of time series goes back to early results on strong uniform convergence rates for kernel regressor estimator for stationary real-valued Markov processes of a fixed order \citet{collomb1986strong}.
Another line of research dedicated to consistency under minimal dependence assumptions: \citet{yakowitz1999strongly} construct partition and kernel-series estimators for stationary ergodic sequences, while \citet{biau2010nonparametric} develop various sequential aggregation schemes, including the one based on kernels; these results avoid mixing assumptions, but are mainly consistency or universal-consistency statements, i.e., do not provide non-asymptotic rates. As for Markov chains, \citet{wu2010kernel} develops asymptotic normality and uniform rates for kernel density and regression estimators of nonlinear causal time series using predictive dependence measures rather than mixing in its classical sense. Extensions to non-stationary settings include \citet{vogt2012nonparametric} with the focus on locally stationary regressors and time-varying regression functions with kernel estimation to mitigate the curse of dimensionality. Learning-theoretic approaches such as \citet{mcdonald2017nonparametric} provide high-probability risk bounds for time series forecasting under \(\beta\)-mixing via capacity control. Closer to statistical learning theory for time series, \citet{kuznetsov2018theory} derive data-dependent generalization bounds for non-stationary and non-mixing processes via sequential complexity and a suitable discrepancy between source and target distributions, leading to discrepancy-weighted forecasting algorithms.

\paragraph{Learning dynamical systems.} The recovery of dynamical models from observed trajectories is the central objective of the classical field of \emph{system identification}. Classical theory studies the estimation of parametric or state-space models and establishes consistency and asymptotic normality under suitable assumptions; see, e.g.~\citet{soderstrom1988system,ljung1998system}. More recent work provides non-asymptotic guarantees from a single trajectory in various settings; see, e.g., \citet{stanhope2014identifiability,simchowitz2018learning,oymak2019non,fattahi2019learning,ziemann2022single,sattar2022finite,block2023smoothed,li2023non,musavi2024identification,zhang2025online} to mention a few. In contrast to this literature, our primary objective is not to recover a finite-dimensional parametrization of the transition rule. Instead, we formulate one-step prediction as a possibly nonparametric least squares problem and study the prediction error with respect to the invariant measure of the process.

\paragraph{Koopman operator framework.} A related operator-theoretic approach seeks to learn the so-called \emph{Koopman} operator associated with the dynamics; see, e.g., \citet{giannakis2019data,mezic2020koopman,bevanda2021koopman,brunton2021modern,colbrook2025introductory}. For finite-dimensional Koopman approximations (a.k.a. \emph{extended dynamic mode decomposition}), finite-data error bounds  under i.i.d.\ or ergodic sampling are studied in \citet{nuske2023finite}. Closer to the direction of our work, kernel-based estimators of Koopman operators and their statistical properties are studied in \citet{mollenhauer2022kernel,kostic2022learning,kostic2023sharp,philipp2024error}. In particular, \citet{kostic2022learning} considers learning from a single trajectory of a stationary ergodic Markov process under mixing assumptions, while convergence in expectation for an online nonparametric learning algorithm is established in \citet{hou2024nonparametric}. In the present work, learning the Koopman operator is studied as an extension of the one-step prediction framework. We show that the same least squares formulation and the same concentration arguments can be applied under uniform geometric ergodicity, including when the observed trajectory is not initialized at stationarity.

\paragraph{Learning from Markov trajectories.}
The literature most directly related to our setting studies regression and statistical learning when the observations are generated by a Markov process. Generalization and consistency results for regression from uniformly ergodic Markov chains are derived, for example, in \citet{zou2009learning,zou2012generalization}. In these works, the Markov process primarily acts as a dependent sampling mechanism for an otherwise standard learning problem. In our setting, by contrast, the input-output pairs, \( (X_t,X_{t+1}) \), are transitions of the same process whose dynamics are to be learned. Moreover, we do not assume that the process is initialized at stationarity. We define the population risk with respect to the invariant measure \(\pi\), and use uniform geometric ergodicity to quantify both temporal dependence and the transient discrepancy between the distribution of \(X_t\) and \(\pi\). This also distinguishes our formulation from single-trajectory bounds in which the prediction error is defined with respect to a finite-horizon, trajectory-dependent distribution, as in \citet{ziemann2022single}. Related but different statistical targets include the estimation of invariant and transition densities. For example, \citet{lacour2008nonparametric} derives minimax \(L^2\)-rates over Besov classes under minorization and geometric ergodicity.

\medskip
Finally, the remainder of the section reviews the literature concerning the main analytical tools used in this work rather than alternative formulations of the learning problem.

\paragraph{Kernel methods and reproducing kernel Hilbert spaces.}
Our non-linear least squares analysis builds on the classical theory of regularized learning in reproducing kernel Hilbert spaces (RKHS). Foundational results establish consistency and finite-sample rates for kernel least squares through the spectral properties of covariance and kernel integral operators; see, e.g., \citet{devito2005learning,devito2005model,smale2005shannon,caponnetto2007optimal,smale2007learning}. More recent developments include for example, \citet{blanchard2018optimal,fischer2020sobolev,li2024towards,zhang2025optimal}. In this paper, we use the same operator-theoretic formulation of regularized least squares. The main difference is that empirical covariance operators and feature averages are computed along a single Markov trajectory rather than from i.i.d. observations. Then the standard deterministic parts of the kernel least squares analysis remain applicable, while the probabilistic arguments require concentration results adapted to the trajectory setting.

\paragraph{Concentration inequalities.} The key technical tool for obtaining our result is a novel concentration inequality for the Hilbert-valued observables of a uniformly ergodic Markov chain. Hence we provide a quick overview of the related results available in the literature. While a large body of work is available for the concentration of partial sums on Markov chains, the focus is primarily on real-valued bounded observables. Notable mentions. \citet{a2004optimal,kontorovich2008concentration,kontorovich2014uniform} study finite and countably large state space chains. Hoeffding type inequality under minorization condition \citet{glynn2002hoeffding} , under \emph{spectral gap} (or generalization of thereof) \citet{miasojedow2014hoeffding,fan2021hoeffding}, under integral probability measure (IPM) contraction \citet{chen2023hoeffding}, and under Wasserstein contraction \citet{sandric2023hoeffding}. Bernstein-type inequalities \citet{huang2024bernstein,jiang2024bernsteinsinequalitiesgeneralmarkov} are derive under spectral gap condition. The case of more general vector- or Hilbert-valued observables remains rarely mentioned. In this block we can mention \citet{katselis2021concentration} for Lipschitz observables over finite-state space and \citet{adamczak2008tail} which derives tail inequality for suprema of empirical processes with application to geometrically ergodic Markov chains.

\subsection{Summary of contributions and plan of the paper}

In this work we study learning from a finite trajectory of an ergodic stochastic dynamical system, highlighting similarities and differences to standard supervised learning from i.i.d. data.
In summary, the main contributions of the paper are:

\begin{itemize}
\item \textbf{One-step prediction of real-valued states from a finite trajectory.}
We formulate forecasting for ergodic stochastic dynamical systems as a nonlinear least squares learning problem from a single finite trajectory. In this setting, we derive universal consistency and finite-sample bounds for possibly nonparametric model classes.
Ergodicity motivates defining the population risk with respect to the invariant measure, while uniform geometric ergodicity provides quantitative control of deviations between empirical averages along the finite trajectory and the corresponding invariant-measure expectations. Indeed, a key result is a concentration inequality for Hilbert-space-valued observables of uniformly geometrically ergodic Markov processes.

\item \textbf{Extensions of the setting.}
We show that the scalar-state analysis extends with minimal modifications to vector-valued, and possibly infinite-dimensional, state spaces.
We also consider higher-order Markov processes, by reducing them to first-order Markov processes on tuples of consecutive states.
Finally, we study one-step prediction on finite-state spaces through a least squares surrogate approach, connecting this setting to {multiclass classification and structured prediction}.

\item \textbf{Learning Koopman operators.}
Beyond {one-step prediction}, we show how the same least squares viewpoint can be used to study the problem of learning Koopman operators from a finite trajectory.
The resulting estimator is equivalent to kernel Koopman regression, for which we provide  a simplified derivation using feature maps rather than reproducing kernel Hilbert spaces and kernels. Under the same uniform geometric  ergodicity assumption, we derive learning guarantees in operator norm.

\item \textbf{Numerical illustration.}
We complement the theoretical results with simple numerical experiments illustrating the behavior of the proposed estimators and the role of temporal dependence, mixing, and the geometry of the invariant distribution.

\end{itemize}

Finally, the paper is organized as follows.
Section~2 recalls the basic notions on stochastic dynamical systems, Markov processes, invariant measures, and ergodicity.
Section~3 introduces the main {one-step prediction} problem for real-valued states, defines the least squares estimator, and states the corresponding learning guarantees.
Sections~4--6 develop the main extensions: vector-valued states, higher-order Markov processes, and finite-state spaces.
Section~7 studies learning the Koopman operator from a finite trajectory.
Section~8 reports numerical experiments, and Section~9 concludes the paper.
The appendices collect background material, operator-theoretic facts, concentration inequalities for Markov processes, and the proofs of the main results.

\section{Stochastic dynamical systems and Markov processes}

In this section, we introduce the notion of stochastic dynamical systems and their connection with Markov processes. 
Several useful notions and results are recalled, in particular invariant measures and ergodicity, playing a central role for learning.

Throughout, we let $\cX$ be a  measurable space  with  $\sigma$-algebra $\cB(\cX)$ and denote by $\cP(\cX)$ the set of all probability measures on $\cX$.

\subsection{Stochastic dynamical systems}

Let \( \cX \) be a measurable space, \( \cN \) a probability space with measure $\nu$ and consider $f:\cX\times \cN\to \cX$. For a sequence  $\eta_0, \eta_1 \dots,$ of i.i.d. samples of $\nu$, and $x_0\in \cX$, let, 
\begin{equation}
x_{t+1} = f(x_t, \eta_t), \quad t=0, 1, \dots
\label{eq:ds-def}
\end{equation}
The above recursion describes a sample of a stochastic evolution given an initial condition. {We will refer to the corresponding system as an \emph{autonomous stochastic dynamical system}.} Here, \( \cX \) is the space of visited states, the function $f$ drives the evolution, while the nuisance space \( (\cN, \nu) \) encodes the randomness which makes the evolution stochastic. 
Such randomness can have different nature. 

\begin{example}[Noisy deterministic evolution]
\label{ex:noisy-evolution}
Let $f_\star:\cX \to \cX$, $\cN=\cX=\mathbb R^d$, and for all $x, \eta \in \cX$,
\[
f(x, \eta)= f_\star(x)+\eta.
\]
In this case, the stochastic evolution can be interpreted as a deterministic evolution perturbed by additive stochastic noise. 
\end{example}
The above example provides perhaps the most classic perspective in the dynamical systems literature in engineering \citep{astrom1970stochastic,ljung1999system}. However, the recursion~\eqref{eq:ds-def} can also describe systems in which stochasticity is intrinsic rather than representing epistemic noise.

\begin{example}[Language model]
Let $\cX$ be a finite set of symbols, such as words or tokens.
The evolution map specifies how the next token is generated from the current one, with the randomness encoded by $\eta_t$ accounting for variability in the generation process. More realistic models allow the previous $p$ tokens, forming the context, to influence the generation of the next token; see Section~\ref{sec:high_ord}.
\end{example}
The above example corresponds to a first-order language model, commonly described in terms of a Markov process \citep{shannon1948mathematical,manning1999foundations}.

More generally, stochastic dynamical systems of the form~\eqref{eq:ds-def} induce Markov processes and, under suitable measurability assumptions, Markov processes admit an analogous representation, as discussed next.

\subsection{Markov processes \& dynamical systems}
\label{sec:MP-dynamical}

Let \( \cX \) be a measurable space and \( (\Omega, \cA, \bbP) \) a probability space.
A stochastic process is a family of random variables $(X_t)_{t\in \bbN_0}$ with \( X_t : \Omega \to \cX \), $t=0, 1, \dots$. It is called a \emph{Markov process} if, for all \( t \in \bbN \) and all \( B \in \cB(\cX) \),
\[
\bbP\big( X_{t+1} \in B \mid X_0, \dots, X_t \big)
=
\bbP\big( X_{t+1} \in B \mid X_t \big).
\]
The process is called homogeneous if the above conditional distribution does not depend on \( t \).

As shown in Proposition~\ref{prop:SDS_MP} in Appendix~\ref{app:SP}, a stochastic dynamical system of the form~\eqref{eq:ds-def} defines a Markov process under mild conditions, in particular when the initial condition is random and independent of the sequence of nuisances $(\eta_t)_t$.
 Conversely, {on a  Borel state space,} every Markov process can be realized by a stochastic dynamical system~\eqref{eq:ds-def}. 

An advantage of considering Markov processes is that they carry additional structure that helps characterizing the dynamics, see Appendix~\ref{app:SP} for references. In particular, Markov processes (hence stochastic dynamical systems) can be equivalently defined in terms of transition kernels. 

\subsection{Transition kernels and Kolmogorov-Chapman equations}
\label{sec:kernel-and-KC}

A transition kernel is a measurable map
\( P : \cX \times \cB(\cX) \to [0,1] \)
such that
\begin{itemize}
\item for every \( x \in \cX \), the map \( B \mapsto P(x,B) \) defines a probability measure on \( \cB(\cX) \);
\item for every \( B \in \cB(\cX) \), the map \( x \mapsto P(x,B) \) is measurable.
\end{itemize}
It is well known that, under Borel-type regularity assumptions on the space $\cX$ (Appendix~\ref{sec:backMP}), equivalences can be established between 
time-homogeneous Markov processes and  transition kernels, see {e.g.~\citet[Section~1]{douc2018markov}}.

Indeed, any time-homogeneous Markov process
$(X_t)_{t\in \bbN_0}$ on a measurable space $\cX$ which is Borel regular (see Appendix~\ref{sec:backMP}) admits an associated transition
kernel $P$, defined for all $x\in\cX$ and all $B\in\cB(\cX)$ by
\[
P(x,B) = \bbP\big( X_{t+1} \in B \mid X_t = x \big),
\]
where the right-hand side does not depend on $t$. Conversely, given a transition kernel $P$ and an initial distribution
$\mu_0 \in \cP(\cX)$, there exists a Markov process $(X_t)_{t\in \bbN_0}$ such that
\[
\bbP(X_0 \in B) = \mu_0(B),
\]
and
\[
\bbP(X_{t+1} \in B \mid X_t = x) = P(x,B),
\]
for all $t \in \bbN$, all $x \in \cX$, and all $B \in \cB(\cX)$.

Given a transition kernel $P$ on $\cX$, one can define recursively a family of
kernels $(P^t)_{t\in\bbN}$ by setting $P^1 = P$ and
\[
P^{t+1}(x,B)
=
\int_{\cX} P^t(x,dy)\,P(y,B),
\qquad
x\in\cX,\ B\in\cB(\cX).
\]
It can be shown that each $P^t$ is again a transition kernel.
Moreover, the family $(P^t)_{t\in\bbN}$ satisfies the Chapman--Kolmogorov equations,
that is, for all $s,t\in\bbN$,
\[
P^{t+s}(x,B)
=
\int_{\cX} P^t(x,dy)\,P^s(y,B),
\qquad
x\in\cX,\ B\in\cB(\cX).
\]

The above equations provide an explicit expression for the marginal laws of each variable $X_t$ in the Markov process. More precisely, for each $t\in\bbN$, each {marginal distribution} $\mu_t\in\cP(\cX)$ of the process at time $t$ is
defined as 
\[
\mu_t(B) = \bbP(X_t\in B),
\qquad
B\in\cB(\cX).
\]
Using the transition kernel, it can be shown that 
\[
\mu_t(B)
=
\int_{\cX} \mu_0(dx)\,P^t(x,B),
\qquad
t\in\bbN,\ B\in\cB(\cX).
\]
In the following, we use the shorthand notation
\be\label{FP_notations}
\mu_t = \mu_0 P^t
\ee
to denote the above relation.

\begin{remark}[Finite-state spaces and transition matrix]\label{trans_mat}
If \( \cX \) is finite with cardinality $N$, it is common to represent the transition kernel by a matrix \( P \in \bbR^{N \times N} \), called the transition matrix. More precisely, for \( \cX = \left\{ \mathbf x_1, \ldots, \mathbf x_N \right\} \), the entries of \( P \) are given by
\[
P_{ij} = \bbP(X_{t+1} = \mathbf x_j \mid X_t = \mathbf x_i),
\]
and the following properties hold
\(P_{ij} \ge 0\),
and 
\(\sum_{j=1}^N P_{ij} = 1\),
see e.g., \citet[Chapter~2.1]{bremaud2013markov}.
Later, when working on finite-state spaces, we will use the notation
\[
P(x,x')=\bbP(X_{t+1}=x' \mid X_t=x),
\qquad x,x'\in\cX,
\]
which is closer to the notation used for general state spaces.
\end{remark}

\subsection{Invariant measure and ergodicity}
\label{sec:ergodic-measure}

To learn dynamical systems with guarantees akin to classical supervised learning,
some regularity properties are needed. In particular, we restrict our attention
to Markov processes satisfying the following assumption.

\begin{assumption}[Uniform geometric ergodicity]\label{UGE}
We assume there exists $\pi\in \cP(\cX)$ and constants \( C < \infty \), \( \rho < 1 \) such that
\[
\sup_{x \in \cX} \| P^t(x, \cdot) - \pi \|_{\mathrm{TV}} \leq C \rho^{t},
\quad \forall \, t \ge 0.
\]
\end{assumption}

The above condition is standard in the Markov process literature, where other instances, e.g. considering norms other than total variation, are also considered; see e.g. \citet{roberts1997geometric,meyn2012markov, levin2017markov,gallegos2024equivalences}.
In the following, by ergodic we usually mean satisfying the above assumption.
We add several comments. First, note that $\pi$ is a unique invariant measure for the process, that is
\[
\pi = \pi P,
\]
where we used the notation~\eqref{FP_notations}. Second, we note that several conditions are necessary for Assumption~\ref{UGE} to hold; in particular, the process must be aperiodic and irreducible. Third, recalling~\eqref{FP_notations}, it is easy to see that Assumption~\ref{UGE} implies 
\[
\|\mu_t - \pi\|_{\mathrm{TV}} \to 0, \qquad \text{as } t\to\infty,
\]
for any initial distribution $\mu_0 \in \cP(\cX)$. 
Fourth, another less straightforward consequence of Assumption~\ref{UGE} is the ergodic theorem
(see, e.g.,~\citealp[Chapter 5]{douc2018markov} and \citealp[Chapter 2.1]{eberle2009markov}), 
stating that for every \(\pi\)-integrable function \(\varphi \), along \(\bbP\)-almost every trajectory,
\be\label{birkhoff}
\frac{1}{T}\sum_{t=1}^T \varphi(X_t)
\rightarrow \int \varphi\,d\pi, \quad \text{as } T\to\infty.
\ee 
Intuitively, this result shows that time averages along a sufficiently long trajectory approximate averages over the invariant distribution. This is the key property needed to derive learning guarantees.
Finally, we note that Assumption~\ref{UGE} is related to the notion of mixing time, see, e.g., \citet{montenegro2006mathematical,levin2017markov}, which, in total variation norm, is defined, for \( \varepsilon>0 \), by
\[
\tau(\varepsilon) = \min\left\{t\in \bbN_0: \sup_{x\in\cX} \nor{P^t(x,\cdot)-\pi}_{\mathrm{TV}} \le \varepsilon \right\}.
\]
The mixing time can be used to define the so-called burn-in time, i.e., the number of initial samples discarded before the chain is sufficiently close to stationarity. Assumption~\ref{UGE} implies
\[
\tau(\varepsilon) \le \left\lceil \frac{\log C+\log(1/\varepsilon)}{- \log \rho} \right\rceil.
\]
Finally, we provide a simple example of a process satisfying Assumption~\ref{UGE}.

\begin{example}[Bounded drift with gaussian noise]\label{the_example}
Let \( \cX=\bbR^d \) and define
\[
X_{t+1} = b(X_t) + \sigma \xi_t,
\qquad \text{with } \xi_t \sim \cN(0,1), \, \sigma > 0,
\]
where \( \cN(0, I_d) \) is a standard Gaussian and \( b: \cX \to \cX \) is bounded. Denote \( a := \norm[\infty]{b}/\sigma \). Then, the chain is uniformly ergodic and Assumption~\eqref{UGE} holds with \( C = 1 \). Moreover, the following approximation could be given,
\begin{equation}
\frac{1}{1 - \rho} \approx \begin{cases}
\sqrt{2 \pi} a \exp\!\left(a^2/2\right), & d = 1 \\[0.3em]
\exp\!\left( a\sqrt{d} + a^2/2 \right), & d \gg 1
\end{cases}
\label{eq:bounded-drift-tau-rel}
\end{equation}
For the details, see Appendix~\ref{sec:MPExamples}.
\end{example}

\section{Learning dynamical systems with scalar-valued states}
\label{sec:learning}

We begin by considering the most fundamental learning problem, namely prediction. For dynamical systems, this corresponds to one-step prediction, that is predicting the state at the next time step given the current state. This  problem is  commonly referred to as forecasting in the time-series literature. In later sections, we will consider more complex settings and tasks. In this section, to highlight the main elements of novelty compared to classical supervised learning, we start by considering scalar-valued states. This choice allows for a treatment close to standard regression, postponing complications arising from vector-valued states. In this context, we consider a basic nonlinear least squares estimator, for which we derive learning guarantees. A central point is the role of the ergodicity assumption, which is not only instrumental from a technical point of view but also plays a conceptual role in defining a notion of risk in this dynamical setting.

\subsection{Setting}
\label{sec:stats-learning}
Let \( \cX=\bbR \), and let \( (X_t)_{t\in\bbN_0} \) be a homogeneous Markov process with initial distribution \( \mu_0 \) and transition kernel \( P \). Assume that there exists a constant \( R<\infty \) such that, for all \( x\in\cX \),
\be\label{2mom}
\int |x'|^2 \, P(x, dx') \le R^2.
\ee
The problem we consider is that of estimating the optimal prediction function
\( f_\star : \cX \to \cX \),
defined by
\be\label{fstar}
f_\star(x) = \int x' \, P(x, dx'), \qquad x \in \cX,
\ee
given a truncated sample trajectory \( \{ x_0, x_1, \dots, x_{T} \} \).
Assuming the process is ergodic with invariant measure \( \pi \), in particular under Assumption~\eqref{UGE},
the quality of an empirical estimator \( \wh f \) is measured by
\begin{equation}
\label{eq:excess-risk-l2}
\cE(\wh f, f_\star) = \int (\wh f(x) - f_\star(x))^2 \, \pi(dx).
\end{equation}
In the following we will derive high-probability bounds for the above error measure. 

Before presenting the least squares estimator we consider, we add a few comments on the above setting.

\paragraph{Optimal prediction function.}
The optimal prediction function~\eqref{fstar} can be written as
\[
f_\star(x)= \bbE[X_{t+1}~|~X_t=x]
\]
and  is akin to the regression function in supervised learning. In particular, for every \( x \in \cX \), \( f_\star(x) \) minimizes the \emph{instantaneous} least squares risk
\[
\bbE[(X_{t+1}-f(X_t))^2~\mid~X_t=x]
= \int (x' - f(x))^2 \, P(x, dx').
\]
This risk and \( f_\star \) are well defined in view of Assumption~\ref{2mom}. This assumption is taken for simplicity to avoid introducing additional conditions, but the weaker assumption
\be\label{fin2mom}
\int \abs{x'}^2 P(x, dx')<\infty
\ee
suffices here.

\paragraph{Data models.}
The data consist of a realization of a truncated trajectory of the process, that is
\be\label{singletraj}
\{ x_0 = X_0(\omega), x_1 = X_1(\omega), \dots, x_{T} = X_{T}(\omega) \},
\qquad \text{for a fixed } \omega \in \Omega .
\ee
This is a natural observation model, that can be compared to other models considered in the literature. 

\begin{remark}[Stationary data]
A common simplification is to assume that the process is initialized at the invariant measure, that is \(X_0 \sim \pi\). In this case, \(X_t \sim \pi\) for every \(t \geq 0\), and \( (X_t)_t\) is a stationary process. 
Unlike in the i.i.d. case, the data are not independent, {since} the observations along a given trajectory still retain temporal dependence. The main role of the stationarity assumption is technical, since it simplifies the probabilistic analysis. However, exact initialization from \(\pi\) is often unrealistic. Indeed, the invariant distribution is typically unknown or not directly available for sampling. In this work, we do not assume stationarity, and our main result applies to trajectories started from an arbitrary initial distribution.
\end{remark}

\paragraph{Expected risk.}
Data from a single trajectory are neither identically distributed nor independent, which raises the question of how to define the risk. For a process with unique invariant measure \( \pi \), we let
\begin{equation}
\label{eq:risk-def}
\cR(f) = \int \left( x' - f(x) \right)^2 \, P(x, dx')\, \pi(dx).
\end{equation}
Assumption~\ref{2mom} ensures the risk  is well defined on \(
L^2(\pi) = \left\{ f : \cX \to \cX \; : \; \|f\|_\pi^2 := \int \abs{f(x)}^2 \pi(dx) < \infty \right\}
\)
and $f_\star\in L^2(\pi)$. Again, Assumption~\ref{2mom} is taken for simplicity, but the weaker assumption
\be\label{pi2mom}
\int|x|^2 \pi(dx)<\infty
\ee
would be sufficient. Moreover, it is easy to verify that~\eqref{eq:excess-risk-l2} is the excess risk, that is 
\[
\cE( f, f_\star) = \cR(f) - \cR(f_\star), \qquad f \in L^2(\pi) .
\]
The above relations find a direct analogy in supervised learning. 
Arguably, this setting provides the simplest and most natural extension of the classical notion of risk to dynamical systems which are not stationary.

\paragraph{Risks and ergodicity.}
Ergodicity offers a way to interpret the expected risk definition \eqref{eq:risk-def}.
Given data, the empirical risk computed along a fixed truncated trajectory is
\begin{equation}
\label{eq:emp-risk-def}
\wh \cR(f)=\frac 1 T \sum_{t=0}^{T-1} (x_{t+1}-f(x_t))^2.
\end{equation}
For i.i.d. data, both the definition of $\cR$ and its empirical approximation by $\wh \cR$ are essentially justified by the law of large numbers. In the setting considered in this paper, the ergodic theorem~\eqref{birkhoff} provides a corresponding justification for data generated by a Markov process. 

\subsection{Least squares estimator}

We consider possibly nonlinear estimators that are linearly parametrized, namely estimators of the form
\be\label{lin_est}
f(x) = \inner{w}{\Phi(x)}, \qquad x \in \cX,
\ee
where \( (\cH, \inner{\cdot}{\cdot}) \) is a separable Hilbert space, \( \Phi : \cX \to \cH \) a measurable map, and \( w \in \cH \).
The estimator coefficients are computed via regularized least squares by minimizing
\[
\frac{1}{T} \sum_{t=0}^{T-1} \big( x_{t+1} - \inner{w}{\Phi(x_t)} \big)^2
+ \la \|w\|^2, \qquad \la > 0 .
\]
The unique solution of the above problem is denoted by \( \wh w_\la \), and the corresponding estimator by \( \wh f_\la \). We add a few comments before describing the associated learning properties.

\paragraph{Linearly parametrized models.}
The class of estimators we consider includes, for example, linear models or models defined by a dictionary of atoms \citep{mallat2008wavelet}, e.g. a trigonometric (Fourier) basis where \( \cH = \ell^2 \times \ell^2 \) and
\[
\Phi(x)
= \big( \alpha_j \cos(2\pi j x),\; \alpha_j \sin(2\pi j x) \big)_{j \ge 1},
\]
for some fixed sequence \( (\alpha_j)_{j\in \bbN_0} \in \ell^2 \).
More generally, we may consider \( \cH \) to be the feature space of a reproducing kernel Hilbert space (RKHS) with kernel \( k : \cX \times \cX \to \bbR \) \citep{aronszajn1950theory}. In this case, one can take \( \Phi(x) = k(x,\cdot) \), the so-called canonical feature map.
A classical example is the RKHS induced by the Gaussian kernel
\[
k(x,y) = \exp\!\left( - \frac{|x-y|^2}{2\sigma^2} \right).
\]
The above Fourier model is also equivalent to an RKHS with kernel
\[
k(x,y) = \sum_{j\ge1} \alpha_j^2 \cos\!\big( 2\pi j (x-y) \big).
\]
We refer to \citet{scholkopf2002learning} for further discussion and examples, as well as for details about the computational aspects considered next.

\paragraph{Batch and online  computations.}
By a standard computation the solution can be written as
\be\label{ridgesol}
\wh w_\la = (\wh \Sigma + \la I)^{-1} \wh h ,
\ee
where
\be\label{sigmah_hat}
\wh \Sigma = \frac{1}{T} \sum_{t=0}^{T-1} \Phi(x_t) \otimes \Phi(x_t),
\qquad
\wh h = \frac{1}{T} \sum_{t=0}^{T-1} \Phi(x_t) \, x_{t+1}.
\ee
The above expression can be computed if \( \cH \) is finite dimensional. Moreover, for any kernel \( k:\cX\times \cX\to \bbR \) such that $k(x, y)= \inner{\Phi(x)}{\Phi(y)}$ for every $x, y \in \cX$, the solution can  also be written for all $ x \in \cX$ as
\begin{equation}
\label{eq:krr-estimator-r}
\wh f_\la(x) = \sum_{t=0}^{T-1} c_t \, k(x_t, x),
\quad x \in \cX,
\end{equation}
where \( c = (c_0,\dots,c_{T-1}) \in \bbR^T \) is given by
\(
c = (\wh K + T \la I)^{-1} \wh y
\)
with \( \wh y = (x_1, x_2, \dots, x_{T}) \in \bbR^T \),
and \( \wh K \in \bbR^{T \times T} \) is defined by
\( (\wh K)_{ts} = k(x_t, x_s) \), for \( t,s = 0,\dots,T-1 \).

We recall that in both cases the above \emph{batch} computations can be made \emph{online} when data are presented sequentially. For finite-dimensional models the Sherman–Morrison–Woodbury identity can be used to derive the recursive least squares updates \citep{kushner2003stochastic}
\[
w_{t+1}
:= w_t + \Gamma_{t+1} \, \Phi(x_t)\big( x_{t+1} - \inner{w_t}{\Phi(x_t)} \big),
\qquad w_0 := 0,
\]
where
\[
\Gamma_{t+1}
:= \Gamma_t - \frac{\Gamma_t \Phi(x_t)\otimes \Phi(x_t)\Gamma_t}{1 + \inner{\Phi(x_t)}{\Gamma_t \Phi(x_t)}} ,
\qquad \Gamma_0 := \frac{1}{T \la} I.
\]
It can be checked that, at step \(t=T-1\), the above recursion is equivalent to computing the regularized least squares estimator~\eqref{eq:krr-estimator-r} with \(T\) samples from the trajectory.

More generally, in the kernel setting analogous online updates can be derived \citep{engel2004kernel}. In this case, the estimator is represented as
\[
f_t(x) = \sum_{s=0}^{t-1} c_s^{(t)} k(x_s,x),
\quad x \in \cX,
\]
where the coefficient vector \( c^{(t)} \in \bbR^t \) satisfies \( c^{(t)} = (\wh K_t + T \la I)^{-1} \wh y_t \). The algorithm maintains the inverse regularized kernel matrix
\[
A_t = \left(\wh K_t + T \la I\right)^{-1}
\]
computed on the first $t$ observations $\left(x_0, x_1, \dots, x_{t-1}\right)$. Given a new observation \( x_{t} \), define
\[
k_t = \big( k(x_s,x_{t}) \big)_{s=0}^{t-1},
\quad \text{and} \quad
\beta_t = \left( T\la + k(x_{t},x_{t}) - k_t^\top A_t k_t \right)^{-1}.
\]
Then the update is performed as
\[
A_{t+1} = \begin{pmatrix}
A_t + \beta_t A_t k_t k_t^\top A_t & -\beta_t A_t k_t \\
-\beta_t k_t^\top A_t & \beta_t
\end{pmatrix},
\quad \text{and} \quad
c^{(t+1)} = \begin{pmatrix}
c^{(t)} \\
0
\end{pmatrix}
+ \beta_t
\begin{pmatrix}
- A_t k_t \\
1
\end{pmatrix}
\big(x_{t+1} - f_t(x_t)\big).
\]
Similarly to the finite dimensional case, \( f_{T} \) coincides with the batch estimator~\eqref{eq:krr-estimator-r}.

We end noting that the above computations require knowing a priori the number of points that will be presented sequentially. Such algorithms are called incremental fixed horizon algorithms in the online learning literature \citep{lugosicesa}.

\subsection{Learning guarantees}
We begin by collecting basic conditions to derive learning guarantees.

\begin{assumption}[Moment condition and boundedness] \label{bound} 
Assume that there exist constants $M, R > 0$ and $\sigma > 0$ such that for all $x \in \cX$, condition~\eqref{2mom} holds and, moreover,
\be\label{mom_bound}
\int \abs{x' - f_\star(x)}^m P(x,dx')
\le
\frac{m!}{2}\,\sigma^2 M^{m-2},
\qquad \forall m \ge 2 .
\ee
Finally, assume that there exists $\kappa > 0$ such that for all $x \in \cX$,
\be\label{bound_phi}
\nor{\Phi(x)} \le \kappa .
\ee
\end{assumption}
The above assumptions have direct analogues in supervised learning. In particular,
 the boundedness Assumption~\ref{bound_phi} on $\Phi$ is standard and satisfied by many models, for example feature maps defined by translation invariant kernels \citep{scholkopf2002learning}. 
The moment condition~\eqref{mom_bound} is also standard in statistical learning theory; see, for example, \cite{caponnetto2007optimal,fischer2020sobolev}. Here, it is naturally adapted to the dynamical systems setting.
As seen before, the second moment bound~\eqref{2mom} implies Conditions~\eqref{fin2mom} and~\eqref{pi2mom} needed for the least squares risks to be well defined. Further, it implies that 
$$
\nor{f_\star}_\pi\le \nor{f_\star}_\infty\le R,
$$
which is needed, together with Condition~\eqref{mom_bound}, to derive probabilistic estimates. The above assumptions suffice to derive general convergence results provided the following condition holds.
\begin{assumption}[Universal models] \label{universal_phi} 
Assume \( \Phi:\cX \to \cH \) is such that 
\[
\inf_{w\in \cH}\int \big(\scal{w}{\Phi(x)}-f_\star(x)\big)^2 \, \pi(dx) = 0.
\]
\end{assumption}
The above assumption can be interpreted in two ways. It may be viewed as a condition on \( f_\star \) for a given feature map \( \Phi \), or one may choose \( \Phi \) so that the condition holds for every \( f_\star \). We adopt the latter perspective and this is the  reason why we call this the universal model assumption. This is the case, for instance, for feature maps associated with \emph{universal} kernels \citep{christmann2008support}.
Examples include feature maps defining translation-invariant kernels whose Fourier transform has full support \citep{carmeli2010vector,sriperumbudur2011universality}. 

The first result we present is asymptotic and establishes universal consistency for one-step prediction under Assumption~\ref{universal_phi}.

\begin{theorem}[Universal consistency]\label{thm:universal}
Let Assumptions~\ref{bound},~\ref{UGE} and~\ref{universal_phi} hold. Let $(\lambda_T)_{T\geq 1}$ satisfy 
\[
\lim_{T\to \infty}\la_T=0
\quad \text{and} \quad
\lim_{T\to \infty}\frac{\log T}{\la_T^2 T}=0,
\]
then
\[
\bbP \left\{
\lim_{T \to \infty} \cE\big(\wh f_{\la_T}, f_\star\big) = 0
\right\} = 1.
\]
\end{theorem}
The proof of the theorem follows from the results in the following subsection, which will be proved in Appendix~\ref{sec:learning_guar}. 

The almost-sure convergence in Theorem~\ref{thm:universal} is usually referred to as \emph{strong consistency}. Since Assumption~\ref{universal_phi} ensures that the model can approximate any regression function, the result can more precisely be viewed as a strong universal consistency statement within the class of processes considered here; see \citet{gyorfi2002distribution} for the corresponding terminology in supervised learning. Analogous consistency results for regularized least squares with i.i.d. observations are classical and can be derived for example by viewing learning as a regularized inverse problem \citep{devito2005learning}. Strong consistency results are also available under stationary ergodic sampling. In particular, \citet{yakowitz1999strongly} consider nonparametric regression and forecasting without mixing assumptions, under regularity conditions on the regression function, while \citet{biau2010nonparametric} establish universal consistency for sequential aggregation schemes. Consistency of kernel methods under more general dependent sampling conditions is studied in \citet{steinwart2009learning}. More directly related results for regression from uniformly ergodic Markov samples are given in \citet{zou2009learning,zou2012generalization}. Compared with these works, Theorem~\ref{thm:universal} concerns prediction of the dynamics generating the trajectory itself, allows a non-stationary initial distribution, and measures the prediction error with respect to the invariant measure~$\pi$, see Section~\ref{sec:stats-learning} and discussions therein.

Next, we present non-asymptotic finite sample bounds. To derive them, we make the following assumption.
\begin{assumption}[well specified model] \label{wellspec}
Assume that there exists $w_\star\in\cH$ such that
\[
f_\star(x)=\scal{w_\star}{\Phi(x)},
\qquad x \in \cX.
\]
\end{assumption}
Assumption~\ref{wellspec} implies that the infimum in Assumption~\ref{universal_phi} is attained. It corresponds to the well specified setting in regression, here considered for the optimal one-step prediction function. This assumption can be strengthened or weakened; see, for example, \cite{caponnetto2007optimal,fischer2020sobolev}. We adopt it because it provides a simple setting to derive and illustrate non-asymptotic learning bounds and leave the study of improved rates to  future work.
\begin{theorem}[Rates for well specified systems]
\label{thm:wellspecsys}
 Let Assumptions~\ref{bound},~\ref{wellspec}, and~\ref{UGE} hold. Let \( \delta \in (0,1) \), if 
 \[
\la \ge \Lambda_{T,\delta} = \frac{8 C \kappa^2}{1-\rho} \left( \sqrt{\frac{2 \log(6/\delta)}{T}}+\frac{1}{T}\right), 
\]
  then with probability at least \( 1 - \delta \), 
\begin{equation}
\label{eq:sample-rate-bound-simplified}
\cE\big(\wh f_{\la}, f_\star\big)
\lesssim \frac{ C^2 D }{(1-\rho)^2} \left(\frac{2\log (6/\delta)}{\la T}+\frac{1}{\la T^2}\right) + \lambda\nor{w_\star}^2,
\end{equation}
where the constant \( D \) can be derived from  Proposition~\ref{prop:est-error}. In particular, setting \( \la = \la_T \asymp T^{-1/2} \), with probability at least \( 1 - \delta \),
\[
\cE\big(\wh f_{\la_T}, f_\star\big)
\lesssim \frac{ C^2 D }{(1-\rho)^2} \left(\frac{2\log (6/\delta)}{\sqrt{T}}+\frac{1}{T^{3/2}}\right)
+ \frac{\nor{w_\star}^2}{\sqrt{T}}.
\]
\end{theorem}

The proof of this theorem follows from the results in the following sections, proved in Appendix~\ref{sec:learning_guar}.

The first finite-sample bound holds for any fixed \(\lambda\) above a threshold depending on the number of observations and the confidence level. In particular, this condition excludes the minimum-norm interpolation regime, corresponding to \(\lambda\downarrow 0\). This restriction could potentially be relaxed using recent results on benign overfitting and interpolation; see, for example, \cite{bartlett2020benign,bartletttsigler}. We do not pursue this direction and leave it for future work. The second bound is obtained by choosing \(\lambda\) to balance the leading estimation and approximation terms. The resulting rate matches the optimal supervised-learning rate under analogous assumptions; see e.g. \cite{caponnetto2007optimal}.

We add a few more comments about the comparison with results for i.i.d. data.
Our result for a single finite trajectory of length \(T\) highlights the presence of an additional term  \(1/(\lambda T^2)\), which appears as a boundary term from truncating a certain martingale over the considered trajectory. For fixed \(\lambda\), this boundary term decays as \(T^{-2}\);  it is smaller than the leading contribution \(1/(\lambda T)\) by a factor \(T^{-1}\), and therefore has no effect on the learning rate. It remains an open question whether such a boundary correction is intrinsic to the Markovian error bound or simply an artifact of the chosen decomposition.

A more notable difference is the multiplicative constant factor
\[
\left(\frac{C}{1-\rho}\right)^2,
\]
which ultimately reflects the Markovian nature of the observations. This factor is tied to the uniform geometric ergodicity Assumption~\ref{UGE}. Although it does not change the asymptotic rate in \(T\), it diverges as \(\rho\to 1\), as happens, for example, in the high-dimensional example~\eqref{eq:bounded-drift-tau-rel}. The presence of a multiplicative factor of this form is expected from the viewpoint of Markov-chain theory. The quantity \((1-\rho)^{-1}\), often interpreted as a \emph{relaxation-time} scale, controls the persistence of correlations along the trajectory. {The quadratic dependence on this time scale is unlikely to be optimal}. 
Indeed, squared deviations of empirical averages typically scale proportionally to the integrated correlation time, suggesting a dependence of order \((1-\rho)^{-1}\) rather than \((1-\rho)^{-2}\); see, for instance, \citet[Theorem~1.2]{chatterjee2025spectral}. Hence, the appearance of \((1-\rho)^{-2}\) may be a consequence of the proof technique. Similar quadratic dependences occur in related concentration estimates for empirical measures of Markov chains; see, e.g., \citet[Theorem~5.4]{kloeckner2020empirical}.

Finally, we note that the closest result in the literature is given in \citet{zou2012generalization}, where generalization bounds are derived for least squares  regularized regression with uniformly ergodic Markov samples. Their setting is closely related to ours, since one-step prediction from a Markov trajectory can be formulated as a regression problem with dependent observations. However, their bounds are derived under different assumptions and expressed in terms of different complexity quantities, making a direct comparison with Theorem~\ref{thm:wellspecsys} difficult. As mentioned in the introduction, further results are available if stationarity is assumed.

The derivation of the above results relies on a classical decomposition into approximation and estimation errors, which we recall and analyze next.

\subsection{Approximation and estimation errors} 

First we introduce the basic decomposition into estimation and approximation error and then provide the corresponding estimates. As we will see it is in the analysis of the estimation error that the dynamic nature of the data plays a role and novel technical aspects are found.

Let \(w_\la\) be the unique solution of the problem
\[
\min_{w\in\cH}
\int \big( x' - \inner{w}{\Phi(x)} \big)^2 \, P(x,dx') \, \pi(dx)
\;+\; \la \|w\|^2,
\qquad \la > 0 
\]
and let $f_\la$ be the corresponding estimator.
Consider the decomposition 
\be\label{err_dec}
\wh f_\la - f_\star = (\wh f_\la -f_\la)+ (f_\la - f_\star ) .
\ee
The first term on the right-hand side is the estimation error, while the second is the approximation error. The approximation can be analyzed exactly as in supervised learning. We recall the following basic results. 
\begin{proposition}\label{prop:apprx}
If Assumptions~\eqref{bound} and~\eqref{universal_phi} hold then
\be\label{apprx_conv}
\lim_{\la\to 0}
\cE(f_\la, f_\star)
=0.
\ee
Moreover, if Assumption~\ref{wellspec} holds then 
\be\label{apprx_bound}
\cE(f_\la, f_\star)
\le \la \nor{w_\star}^2.
\ee
\end{proposition}
\noindent The proof is a direct adaptation of standard results in supervised learning \cite{devito2005model} and inverse problems \cite{engl1996regularization}. Since it is short, we report it in Appendix for completeness. For the estimation error we  prove the following result.

\begin{proposition} \label{prop:est-error}
Let  Assumptions~\ref{bound} and~\ref{UGE} hold. Let \( \delta \in (0,1) \) if
\[
\la \ge \frac{8 C \kappa^2}{1-\rho} \left( \sqrt{\frac{2 \log(6/\delta)}{T}}+\frac{1}{T}\right).
\]
then with probability at least \( 1 - \delta \),
\[
\cE\big(\wh f_\la, f_\la\big)
\le \frac{C^2 D_1 }{(1-\rho)^2} \left(
1
+ \frac{\kappa^2}{\lambda}
\right) \left(\frac{2\log (6/\delta)}{\la T}+\frac{1}{\la T^2}\right)
+ \frac{D_2 \log^2(6/\delta)}{\la T^2}
+\frac{D_3 \log(6/\delta)}{\la T},
\]
where \( D_1 := 192 \kappa^2 R^2 \), \( D_2 := 12 \kappa^2 M^2 \), and \( D_3 := 6 \kappa^2 \sigma^2 \). Moreover, if in addition Assumption~\ref{wellspec} holds, then
\[
\cE\big(\wh f_\la, f_\la\big)
\le 
\frac{ C^2 D_4 }{(1-\rho)^2} \left(\frac{2\log (6/\delta)}{\la T}+\frac{1}{\la T^2}\right)
+ \frac{D_2 \log^2(6/\delta)}{\la T^2}
+ \frac{D_3 \log(6/\delta)}{\la T},
\]
where \( D_4 := 96 \kappa^2 \, (R + \kappa \nor{w_\star})^2 \).
\end{proposition}

The proof is given in Appendix~\ref{sec:learning_guar}. The proposition provides two bounds. The first does not require the model to be well specified, but has a worse dependence on \(\lambda\), of order \(1/(\lambda^2 T)\), in the leading term. This bound is sufficient to establish the asymptotic consistency result in Theorem~\ref{thm:universal}. Under Assumption~\ref{wellspec}, the second bound improves this dependence to \(1/(\lambda T)\), which is needed to obtain the finite-sample rate in Theorem~\ref{thm:wellspecsys}.

Both bounds consist of four main terms. The last two arise from the analysis of a quantity analogous to the noise term in supervised learning, here induced by the stochastic transitions between states, see Lemma~\ref{lem:probbounds}. This quantity has a martingale structure, which is reflected in the form of the bound. 
Under the moment condition~\eqref{mom_bound}, a Bernstein-type martingale concentration inequality yields the classical variance and scale contributions, of order \(\log(1/\delta)/T\) and \(\log^2(1/\delta)/T^2\), corresponding respectively to the sub-Gaussian and sub-exponential parts of the bound.

The first two terms on the right-hand side arise from probabilistic estimates, in particular from concentration of the empirical covariance operator; see again Lemma~\ref{lem:probbounds}. This is the part of the estimation error most substantially affected by the dynamical nature of the data and, specifically, by the uniform geometric ergodicity property of the process. The term proportional to \(\log(6/\delta)/(\lambda T)\) is obtained from a Hoeffding-type bound under the boundedness condition~\eqref{bound_phi}, while the term proportional to \(1/(\lambda T^2)\) is the aforementioned boundary term. Both carry the factor \(C^2/(1-\rho)^2\), reflecting the uniform geometric ergodicity Assumption~\ref{UGE}.

\section{Learning dynamical systems with vector-valued states}
\label{sec:vvds}

The extension to vector-valued dynamical systems is straightforward provided the state space $\cX$ has a linear structure. In particular, we let $\cX$ be a Hilbert space, which allows us to consider $\cX = \bbR^d$ for some $d \in \mathbb N$. Considering possibly infinite-dimensional spaces can cover a wider range of settings and will be useful when discussing Koopman operators. We next describe how to adapt the discussion in the previous section to vector-valued states.

\subsection{Setting}

Let \( (\cX, \scal{\cdot}{\cdot}_\cX) \) be a real separable Hilbert space, and let \( (X_t)_{t \in \bbN_0} \) be a {time-}homogeneous Markov process with initial distribution \(\mu_0\) and transition kernel \(P\). Assume that there exists \(R<\infty\) such that, for all \(x \in \cX\),
\begin{equation}\label{2mom_vv}
\int \nor{x'}_{\cX}^2 \, P(x, dx') \le R^2.
\end{equation}
The optimal prediction function in this setting is 
\( f_\star : \cX \to \cX \) given by
 \[ f_\star(x) = \int x' \, P(x, dx'), \quad  x \in \cX,
  \] 
where the integral is in the sense of Bochner.
The goal is to estimate the optimal prediction function from a single trajectory $x_0, \dots, x_T$. Assuming the process is ergodic with invariant measure \( \pi \), see Assumption~\ref{UGE}, the quality of an empirical estimator \( \wh f \) is measured by
\begin{equation}
\label{eq:excess-risk-l2_vv}
\cE(\wh f, f_\star) = \int \nor{\wh f(x) - f_\star(x)}_{\cX}^2 \, \pi(dx).
\end{equation}

The comments we made for scalar valued states adapt to the vector states setting with the obvious changes. The optimal prediction function can be written as  $f_\star(x)=\bbE[X_{t+1}~|~X_t=x]$. 
For every \(x\in \cX\), $f_\star(x)$ minimizes the instantaneous least squares risk
\[
\bbE[\nor{X_{t+1}-f(X_t)}_\cX^2~|~X_t=x]= \int \nor{x' - f(x)}_\cX^2 \, P(x, dx').
\]
Both $f_\star$ and the {instantaneous risk} are well defined by Condition~\eqref{2mom_vv}, which is imposed for simplicity and could be relaxed, see~\eqref{fin2mom}.
Moreover, if the process is ergodic with invariant measure \( \pi \), we define the expected risk
\[
\cR(f) = \int \nor{x' - f(x) }_\cX^2 \,  P(x, dx')\, \pi(dx)
\]
which is well defined on  \(
L^2(\pi) := \left\{ f : \cX \to \cX \; : \; \nor{f}_\pi^2 := \int \nor{f(x)}_\cX^2 \, \pi(dx) < \infty \right\}
\) by Condition~\eqref{2mom_vv}, and again a weaker condition like~\eqref{pi2mom} would suffice. Then it is easy to check that $f_\star$ is the minimizer of the expected risk in $L^2(\pi)$ and moreover 
\[
\cE( f, f_\star) = \cR(f) - \cR(f_\star), \qquad f \in L^2(\pi) .
\]
\begin{remark}[On the  \(L^2(\pi)\) space notation]
The space \(L^2(\pi)\) introduced above is the Bochner
space \(L^2(\pi;\cX)\) of strongly measurable, square-integrable \(\cX\)-valued functions, identified up to \(\pi\)-almost-everywhere equality. Since \(\cX\) is separable, Borel \(\cB(\cX)\)-measurability implies
strong measurability. We use the
shorter notation \(L^2(\pi)\), since the meaning will always be clear from the context.
\end{remark}
As shown next, the least squares estimator can also be naturally generalized to the case of vector-valued states.

\subsection{Least squares estimator}

In this setting, we consider nonlinear estimators of the form
\be\label{lin_est_vv}
f(x) = W^*\Phi(x), \qquad x \in \cX,
\ee
where \( (\cH, \inner{\cdot}{\cdot}) \) is a separable Hilbert space, \( \Phi : \cX \to \cH \) a measurable map, and \( W \in \cL_2(\cX, \cH)\). The estimator coefficients are computed via regularized least squares by minimizing
\be\label{rls_vv}
\frac{1}{T} \sum_{t=0}^{T-1} \nor{x_{t+1} - W^*\Phi(x_t)}_{\cX}^2
+ \la \nor{W}_{\cL_2(\cX, \cH)}^2, \qquad \la > 0 .
\ee
The error is measured by the norm of $\cX$ and the regularization term is the (squared) Hilbert--Schmidt norm. The unique solution of the above problem is denoted by \( \wh W_\la \), and the corresponding estimator by \( \wh f_\la \). By standard computations \( \wh W_\la \) can be written as
\be\label{ridgesol_vv}
\wh W_\la = (\wh \Sigma + \la I)^{-1} \wh h ,
\ee
where
\be\label{sigmah_vv}
\wh \Sigma = \frac{1}{T} \sum_{t=0}^{T-1} \Phi(x_t) \otimes \Phi(x_t),
\qquad
\wh h = \frac{1}{T} \sum_{t=0}^{T-1} \Phi(x_t) \otimes x_{t+1}.
\ee
Note that the only difference with the scalar case is the expression for $\wh h$. As in the scalar case, the above expression can be computed if \( \cH \) is finite dimensional.
Moreover, for any kernel \( k:\cX\times \cX\to \bbR \) such that $k(x, y)= \inner{\Phi(x)}{\Phi( y)}$ for every $x, y \in \cX$, the solution can also be written for all $ x \in \cX$ as
\[
\wh f_\la(x) = \sum_{t=0}^{T-1} C_t \, k(x_t, x), 
\]
where \( C = (C_0,\ldots,C_{T-1}) \in \cX^T \) is given by \( C = (\wh K + T \la I)^{-1} \wh y \), with \( \wh y = (x_1, x_2, \dots, x_T) \in \cX^T \) and \( \wh K\) is a linear operator from $\cX^T$ to $\cX^T$ defined, for every \(C\in\cX^T\), by
\[
(\wh K C)_t=\sum_{s=0}^{T-1} k(x_t,x_s) \, C_s,
\qquad t=0,\ldots,T-1 .
\]
If $\cX= \bbR^d$, then $\cX^T= (\bbR^d)^T$, \( \wh y , C \in (\bbR^d)^T\),
and \( \wh K \in \bbR^{dT \times dT} \) is a block matrix with 
\[
\wh K_{ts}=k(x_t,x_s)I_d,
\qquad t,s=0,\ldots,T-1 .
\]
Importantly, the representation of the estimator remains valid when \(\cX\) is infinite dimensional; see, e.g., \citet{micchelli2005learning,grunewalder2012modelling} for details.

Online computations can also be derived analogously to the scalar setting but we omit this discussion. We add instead a few comments on possible choices for the map $\Phi$ to connect with ideas in time series analysis and multi-task learning. 

\paragraph{Component-wise {one-step prediction}.}

We specialize the discussion to the case of finite-dimensional state spaces, that is, $\cX= \bbR^d$.
In this case, each state is a vector $x= (x^1, \dots, x^d)$. The estimator {\em coefficient} $W$ can be identified with  $w_1, \dots, w_d $ in $\cH$, and each function of the form~\eqref{lin_est_vv} can be equivalently described in terms of its components
\be\label{vv_comp}
f_j(x)=\scal{w_j}{\Phi(x)}, \quad \quad j=1, \dots, d,
\ee
where $f_j(x)= e_j^T f(x)$ and $(e_j)_{j=1}^d$ is the canonical orthonormal basis of $\bbR^d$. Indeed, for every \(j=1,\dots,d\), if \(w_j = W e_j\), then for every \(x\in\bbR^d\),
\[
f_j(x)= e_j^T W^*\Phi(x)
= \scal{W e_j}{\Phi(x)}
= \scal{w_j}{\Phi(x)} .
\]
This general situation can be contrasted with a more restrictive model where each component is treated separately. Indeed, this corresponds to considering maps
\(
\Phi_j:\bbR\to \cH_j\), \( j=1,\dots,d,
\)
where each \(\cH_j\) is a Hilbert space, and letting
\[
\cH=\bigoplus_{j=1}^d \cH_j,
\qquad
\Phi(x)=\big(\Phi_1(x^1),\dots,\Phi_d(x^d)\big).
\]
In this case, choosing \(w_j\in \cH_j\) yields
\(
f_j(x)=\scal{w_j}{\Phi_j(x^j)}_{\cH_j},
\)
so that the \(j\)-th component of the prediction depends only on the \(j\)-th component of the current state. 

\paragraph{Joint estimation and multitask learning.}
{Returning to the formulation~\eqref{vv_comp}, we can observe that the estimation of each \(w_j\), \(j=1,\dots,d\) is done separately.} Indeed,
\be\label{indep_reg}
\nor{W}_{\cL_2(\bbR^d,\cH)}^2
=
\sum_{j=1}^d \nor{w_j}^2,
\ee
so that Problem~\eqref{rls_vv} is equivalent to
\[
\min_{w_1,\dots,w_d\in\cH}
\frac{1}{T}\sum_{t=0}^{T-1}
\sum_{j=1}^d
\left( x_{t+1}^j-\scal{w_j}{\Phi(x_t)}\right)^2
+
\lambda \sum_{j=1}^d \nor{w_j}^2 ,
\]
which is itself equivalent to the collection of \(d\) separate problems
\[
\min_{w_j\in\cH}
\frac{1}{T}\sum_{t=0}^{T-1}
\left(x_{t+1}^j-\scal{w_j}{\Phi(x_t)}\right)^2
+
\lambda \nor{w_j}^2,
\qquad j=1,\dots,d .
\]
It is natural to ask whether these different problems can be coupled, so that the estimation of each component is influenced by the others. This is a classical question in multitask learning \cite{micchelli2004kernels,alvarez2012kernels}, which we next revisit within the {one-step prediction} setting.

The most intuitive idea is to couple the estimation problems by designing suitable regularizers. In particular, it is natural to consider regularizers of the form
\be\label{mtl_reg}
\operatorname{Tr}(W A W^*)
=
\sum_{i,j=1}^d A_{ij}\scal{w_i}{w_j}_\cH,
\ee
for a suitable symmetric positive semidefinite matrix \(A\in\bbR^{d\times d}\). Different choices of \(A\) induce different forms of coupling. {For example, \(A=I\) gives the independent regularizer~\eqref{indep_reg}.} For \(a\in[0,1]\), the choice
\[
A=(1-a)\left(I-\frac1d\mathbf 1\mathbf 1^\top\right)+aI
\]
gives
\[
(1-a)\sum_{j=1}^d \nor{w_j-\bar w}^2
+
a\sum_{j=1}^d \nor{w_j}^2,
\qquad
\bar w=\frac1d\sum_{j=1}^d w_j .
\]
Thus \(a=0\) corresponds to maximal coupling, while \(a=1\) recovers the independent regularization. More generally, if \(C_1,\dots,C_m\) is a partition of \(\{1,\dots,d\}\), the choice
\[
A=(1-a)\left(
I-\sum_{\ell=1}^m |C_\ell|^{-1}
\mathbf 1_{C_\ell}\mathbf 1_{C_\ell}^\top
\right)+aI
\]
gives the regularizer
\[
(1-a)\sum_{\ell=1}^m\sum_{j\in C_\ell}
\nor{w_j-\bar w_\ell}^2
+
a\sum_{j=1}^d \nor{w_j}^2,
\qquad
\bar w_\ell
=
\frac{1}{|C_\ell|}\sum_{j\in C_\ell} w_j .
\]
This couples only components belonging to the same cluster.

As a final example, consider a symmetric matrix \(S\in\bbR^{d\times d}\) with nonnegative entries, and let \(L_S=D_S-S\) be the corresponding graph Laplacian, with \((D_S)_{ii}=\sum_{j=1}^d S_{ij}\). Taking \(A=L_S+\Gamma\), where \(\Gamma\) is diagonal with nonnegative entries, gives the regularizer
\[
\frac12\sum_{i,j=1}^d S_{ij}\nor{w_i-w_j}^2
+ \sum_{j=1}^d \Gamma_{jj}\nor{w_j}^2 .
\]

Least squares estimators using regularization of the form~\eqref{mtl_reg} can be rewritten to highlight their relation with estimators using the independent regularization~\eqref{indep_reg}. Since we do not develop this point further, we restrict the discussion to \(A\succ 0\), for which the relation is particularly simple. Indeed, consider the least squares problem 
\be\label{mtl_ridge}
\min_{W\in \cL_2(\cX,\cH)}
\frac1T\sum_{t=0}^{T-1}
\nor{x_{t+1}-W^*\Phi(x_t)}_{\cX}^2
+
\lambda\,\operatorname{Tr}(WAW^*).
\ee
Recall that \( A = A^\ast \), therefore it admits the eigendecomposition
\[
A=UDU^*,
\qquad D=\operatorname{diag}(d_1,\dots,d_d),
\qquad d_j>0,
\]
where \(U\) is orthogonal. Define
\[
V=WU,
\qquad \widetilde x_t=U^*x_t,
\qquad t=0,\dots,T.
\]
Then \(W=VU^*\), and therefore
\[
W^*\Phi(x_t)=UV^*\Phi(x_t),
\qquad \nor{x_{t+1}-W^*\Phi(x_t)}_{\cX}^2
= \nor{\widetilde x_{t+1}-V^*\Phi(x_t)}_{\cX}^2.
\]

Writing \(V=(v_1,\dots,v_d)\), we have,
\[
\operatorname{Tr}(WAW^*)=\operatorname{Tr}(VDV^*)=\sum_{j=1}^d d_j\nor{v_j}^2 .
\]
Thus Problem~\eqref{mtl_ridge} is equivalent to
\[
\min_{v_1,\dots,v_d\in\cH}
\frac{1}{T}\sum_{t=0}^{T-1}
\sum_{j=1}^d
\left(\widetilde x_{t+1}^j-\scal{v_j}{\Phi(x_t)}\right)^2
+ \lambda \sum_{j=1}^d d_j\nor{v_j}^2 ,
\]
which is itself equivalent to the collection of \(d\) separate problems
\[
\min_{v_j\in\cH}
\frac{1}{T}\sum_{t=0}^{T-1}
\left(\widetilde x_{t+1}^j-\scal{v_j}{\Phi(x_t)}\right)^2
+ \lambda d_j\nor{v_j}^2,
\qquad j=1,\dots,d .
\]
Consequently, the optimization decouples in the eigenbasis of \(A\), with coordinate-dependent regularization parameters \(\lambda d_j\), while the prediction in the original coordinates is recovered by rotating the resulting vector through \(U\).

Finally we note that the learning guarantees below are stated for \(A=I\). For a fixed \(A \succ 0\), the same arguments can be applied in its eigenbasis, but the resulting bounds would explicitly depend on the spectrum of \(A\). We leave this extension for future work.

\paragraph{Vector-valued reproducing kernel Hilbert spaces.}
The class of estimators introduced above can be equivalently described in terms of vector-valued reproducing kernel Hilbert spaces and operator-valued kernels
\[
\Gamma:\cX\times\cX\to \cL(\cX).
\]
Our basic setting corresponds to kernels of the form \(\Gamma(x,x')=k(x,x')I\), with \(k(x,x')=\scal{\Phi(x)}{\Phi(x')}\). Regularizers of the form~\eqref{mtl_reg} correspond to kernels of the form \(\Gamma(x,x')=k(x,x')A^{-1}\); see again \cite{micchelli2004kernels,alvarez2012kernels}. Considering more general vector-valued reproducing kernel Hilbert spaces is another possible direction for future work.

Given the above discussion, in the next section we discuss learning guarantees for least squares estimators with vector-valued states.

\subsection{Learning guarantees}

The assumptions for learning and the corresponding guarantees remain the same as in the scalar case, up to straightforward adjustments. We therefore omit the universal consistency result, and focus on the rates for well specified systems. The boundedness assumption is given below.

\begin{assumption}[Moment condition and boundedness] \label{bound_vv} 
Assume that there exist constants $M, R > 0$ and $\sigma>0$ such that for all $x \in \cX$, 
condition~\eqref{2mom_vv} holds and moreover
\be\label{res_bound_vv}
\int_{\cX}\nor{x' - f_\star(x)}^m\,P(x,dx')
\le
\frac{m!}{2}\,\sigma^2 M^{m-2},
\qquad \forall m \ge 2 .
\ee
Finally, assume that there exists $\kappa > 0$ such that for all $x \in \cX$,
\be\label{bound_phi_vv}
\nor{\Phi(x)} \le \kappa .
\ee
\end{assumption}

Well specified systems are described by the following assumption. Let $\cL_2(\cX, \cH)$ be the space of Hilbert--Schmidt operators from $\cX$ to $\cH$, see Appendix~\ref{app_operators}.

\begin{assumption}[Well specification] \label{wellspec_vv} 
Assume there exists $W_\star\in \cL_2(\cX, \cH)$ such that
\[
f_\star(x)
= W_\star^\ast \Phi(x),
\qquad \forall x\in \cX .
\]
\end{assumption}
The above assumption is a direct analogue of  Assumption~\ref{wellspec} for scalar-valued states,  here extended to the  vector-valued states setting. We are ready to state the learning guarantees for the least-squares estimator in this setting.  Universal  consistency  can be obtained by the very same argument as in the scalar-valued case; see Theorem~\ref{thm:universal}. We omit this derivation and focus on finite sample bounds for well-specified models.

\begin{theorem}[Rates for well specified systems]
\label{thm:wellspecsys_vv}
Let Assumptions~\ref{bound_vv}~\ref{wellspec_vv}, and~\ref{UGE} hold.
Let \( \delta \in (0,1) \), if 
\[
\la \ge \frac{8 C \kappa^2}{1-\rho} \left( \sqrt{\frac{2 \log(6/\delta)}{T}}+\frac{1}{T}\right), 
\]
then with probability at least \( 1 - \delta\), 
\begin{equation}
\label{eq:sample-rate-bound-simplified_vv}
\cE\big(\wh f_{\la}, f_\star\big)
\lesssim \frac{ C^2 D }{(1-\rho)^2} \left(\frac{2\log (6/\delta)}{\la T}+\frac{1}{\la T^2}\right) + \lambda\nor{W_\star}_{\cL_2(\cX,\cH)}^2,
\end{equation}
where the constant \( D \) can be derived from the proof. In particular, setting \( \la = \la_T \asymp T^{-1/2} \), with probability at least \( 1 - \delta \),
\[
\cE\big(\wh f_{\la_T}, f_\star\big)
\lesssim \frac{ C^2 D }{(1-\rho)^2} \left(\frac{2\log (6/\delta)}{\sqrt{T}}+\frac{1}{T^{3/2}}\right)
+ \frac{\nor{W_\star}_{\cL_2(\cX,\cH)}^2}{\sqrt{T}}.
\]
\end{theorem}
The proof of the above result is given in Appendix~\ref{sec:furth_proofs} and follows from a straightforward adaptation of the scalar-valued analysis. The required modifications concern only the vector-valued nature of the variables and are unrelated to the main difficulties arising from dynamical systems and single-trajectory data. This is why we chose to present the scalar-valued setting first for clarity, although it could also be obtained as a special case of the vector-valued analysis. The above bound has the same structure discussed for  the bound for scalar valued states up-to the constants definition, and the dependence on  $\nor{W_\star}_{\cL_2(\cX,\cH)}$. Since we already discussed this structure and  the nature of the different contributions, we don't develop this discussion further. 

\section{Learning higher-order systems dynamical systems }\label{sec:high_ord}

Thus far, we considered dynamical systems whose evolution at each time step depends only on the current state, and saw that they are equivalent to Markov chains. Next, we generalize the discussion to higher-order dynamical systems and the corresponding higher-order Markov chains. That is, we consider dynamical systems where the evolution at each time step is governed by the previous $p$ visited states.

\subsection{Higher-order dynamical systems \& Markov chains}

We return to the dynamical systems description in terms of evolution maps~\eqref{eq:ds-def}.
Let \( \cX \) be a measurable space, \( \cN \) a probability space with measure \(\nu\), and consider \(f:\cX^p\times \cN\to \cX\). For a sequence \(\eta_0,\eta_1,\dots\) of i.i.d. samples from \(\nu\), and initial states \(x_0,\dots,x_{p-1}\in\cX\), let
\begin{equation}
x_{t+p} = f(x_t, \ldots, x_{t+p-1}, \eta_t), \qquad t=0,1,\dots.
\label{eq:ds-def_p}
\end{equation}
The above recursion describes a sample of a stochastic evolution given an initial condition, but now the evolution at each time step is governed by the previous \(p\) states. The corresponding system is called an autonomous stochastic dynamical system of order \(p\). Choosing \(p=1\), the so called first-order system, we recover the setting  in~\eqref{eq:ds-def}.

\begin{remark}
An alternative indexing convention, commonly used in the time-series and forecasting literature, is
\[
x_{t+1}
= f(x_t,\ldots,x_{t-p+1},\eta_t),
\qquad t=p-1,p,\ldots.
\]
The difference is ultimately notational. We use the convention in~\eqref{eq:ds-def_p} for consistency with the subsequent Markov process formulation and the subsequent definition of the estimator.
\end{remark}

Again, we can  consider a corresponding Markov process. More precisely, if \( (\Omega, \cA, \bbP) \) is a probability space, an order-\( p \) Markov process is a family of random variables \((X_t)_{t\in \bbN_0}\), with \(X_t:\Omega\to\cX\), such that, for all \(t \in \bbN_0\) and all \(B\in\cB(\cX)\),
\[
\bbP\big( X_{t+p} \in B \mid X_0, \dots, X_{t+p-1} \big)
= \bbP\big( X_{t+p} \in B \mid X_t, \dots, X_{t+p-1}\big).
\]
We take the process to be time-homogeneous, that is, the above conditional distribution does not depend on \(t\). The corresponding transition kernel is a measurable map
\(
P:\cX^p\times \cB(\cX)\to [0,1]
\)
such that
\begin{itemize}
\item for every \(\overline x\in \cX^p\), the map
\(B\mapsto P(\overline x,B)\) defines a probability measure on \(\cB(\cX)\);
\item for every \(B \in \cB(\cX)\), the map
\(\overline x\mapsto P(\overline x,B)\) is measurable.
\end{itemize}

For a time-homogeneous order-\( p \) Markov process, the kernel is given by
\[
P(x_t,\ldots,x_{t+p-1},B)
= \bbP\big(X_{t+p} \in B \,\big|\, X_t=x_t,\ldots,X_{t+p-1}=x_{t+p-1} \big),
\qquad B \in \cB(\cX),
\]
where the right-hand side does not depend on \(t\). Conversely, given a transition kernel \(P:\cX^p\times\cB(\cX)\to[0,1]\) and an initial distribution \(\overline \mu_0\in\cP(\cX^p)\), there exists a Markov
process of order \(p\) such that
\[
\bbP\big((X_{p-1},\dots,X_0)\in A\big)=\overline \mu_0(A),
\qquad A\in\cB(\cX^p),
\]
and
\[
\bbP\big( X_{t+p}\in B \,\big|\, X_t,\ldots,X_{t+p-1} \big)
= P(X_t,\ldots,X_{t+p-1},B),
\qquad B \in \cB(\cX).
\]
For an order-\(p\) Markov chain it is possible to consider a corresponding \emph{lifted} first-order Markov process on \(\cX^p\); see, e.g., \citet{doob1990,xu2026can} for details. More precisely, the process
\be\label{liftd_MP}
\overline X_t=(X_t,\ldots,X_{t+p-1})\in\cX^p, \qquad t\in\bbN_0
\ee
is a first-order Markov process with transition kernel
\(Q:\cX^p \times \cB(\cX^p) \to [0,1]\) given by
\be\label{lift_tk}
Q(\overline x,A)
= \int_{\cX}
\mathds{1}_A(x_1,\ldots,x_{p-1}, x')\,
P(\overline x,dx'),
\ee
for every \(\overline x=(x_0, \ldots,  x_{p-1})\in\cX^p\) and \(A\in\cB(\cX^p)\).

Since \(Q\) is a \emph{proper} Markov transition kernel on \( \cX^p \), all the standard constructions and identities developed in Section~\ref{sec:kernel-and-KC} apply. We briefly recall the relevant ones here.

The family of iterated kernels is defined by \(Q^1=Q\) and
\[
Q^{t+1}(\overline x,A)
= \int_{\cX^p} Q^t(\overline x,d\overline y)\,Q(\overline y,A),
\qquad
\overline x\in\cX^p,\ A\in\cB(\cX^p),
\]
and satisfies the Chapman--Kolmogorov equations
\[
Q^{t+s}(\overline x,A)
=
\int_{\cX^p} Q^t(\overline x,d\overline y)\,Q^s(\overline y,A),
\qquad
\overline x\in\cX^p,\ A\in\cB(\cX^p).
\]
If \(\overline \mu_t\in\cP(\cX^p)\) denotes the marginal distribution of \(\overline X_t\), then
\[
\overline \mu_t(A)=\bbP(\overline X_t\in A),
\qquad A\in\cB(\cX^p),
\]
and
\[
\overline\mu_t(A)
=
\int_{\cX^p}\overline \mu_0(d\overline x)\,Q^t(\overline x,A).
\]
We write this relation shortly as
\be\label{FPL_notations}
\overline \mu_t= \overline \mu_0 Q^t .
\ee

To learn higher-order dynamical systems we impose regularity properties on the lifted process.
More precisely, we restrict our attention to higher-order Markov processes satisfying the following assumption.
\begin{assumption}[Uniform geometric ergodicity]\label{UGE_p}
Let \(Q\) be the transition kernel~\eqref{lift_tk}.
We assume there exists \(\overline \pi\in \cP(\cX^p)\) and constants \(C>0\), \(\rho<1\)
such that, for all \(t\in\bbN\),
\[
\sup_{\overline x\in\cX^p}
\big\| Q^t(\overline x,\cdot)-\overline \pi \big\|_{\mathrm{TV}}
\le C\rho^t .
\]
\end{assumption}
The above condition is the uniform geometric ergodicity Assumption~\ref{UGE} for the lifted Markov process~\eqref{liftd_MP}. Here, \(\overline \pi\) is the unique invariant measure for the lifted process, that is,
\[
\overline\pi=\overline\pi Q, 
\]
where we used the notation~\eqref{FPL_notations}.
Moreover, Assumption~\ref{UGE_p} implies that
\[
\| \overline \mu_t- \overline \pi\|_{\mathrm{TV}}\to 0
\qquad \text{as } t\to\infty,
\]
for any initial distribution \(\overline \mu_0\in\cP(\cX^p)\).

\subsection{Setting}

Let \( (\cX, \scal{\cdot}{\cdot}_\cX) \) be a real separable Hilbert space, and let \( (X_t)_{t \in \bbN_0} \) be a time-homogeneous Markov process of order \(p\), with transition kernel \(P:\cX^p\times\cB(\cX)\to[0,1]\). Let
\( (\overline X_t)_{t\in\bbN_0} \) be the corresponding lifted process~\eqref{liftd_MP}, with transition kernel \(Q\) and initial distribution \(\overline \mu_0\in\cP(\cX^p)\). Assume that there exists \(R<\infty\) such that, for all \(\overline x\in\cX^p\),
\begin{equation}\label{2mom_p}
\int \nor{x'}_{\cX}^2 \, P(\overline x, dx') \le R^2.
\end{equation}
The optimal prediction function in this setting is \( f_\star : \cX^p \to \cX \), given by
\[
f_\star(\overline x) = \int x' \, P(\overline x, dx'),
\qquad \overline x\in\cX^p,
\]
where the integral is in the sense of Bochner. The goal is to estimate the optimal prediction function from a single trajectory \(x_0,\dots,x_T\) of the process \( (X_t)_{t \in \bbN_0} \).  Assuming the lifted process is ergodic with invariant measure 
\(\overline \pi\in\cP(\cX^p)\), see Assumption~\ref{UGE_p}, the quality of an empirical estimator \(\wh f\) is measured by
\begin{equation}
\label{eq:excess-risk-l2_p}
\cE(\wh f, f_\star)
=
\int_{\cX^p}
\nor{\wh f(\overline x)-f_\star(\overline x)}_{\cX}^2
\,\overline \pi(d\overline x).
\end{equation}

The observations made for first-order systems adapt to this setting with the obvious changes. The optimal prediction function can be written as 
\[
f_\star(\overline x)
=
\bbE[X_{t+p}\mid \overline X_t=\overline x].
\]
For every \(\overline x\in\cX^p\),
 \(f_\star(\overline x)\) minimizes the instantaneous least squares risk
\[
\bbE\!\left[
\nor{X_{t+p}-f(\overline X_t)}_\cX^2
\,\middle|\,
\overline X_t=\overline x
\right]
=
\int \nor{x'-f(\overline x)}_\cX^2\,P(\overline x,dx').
\]
Both \(f_\star\) and the {instantaneous risk}  are well defined by Condition~\eqref{2mom_p}, which is taken for simplicity and could be relaxed by extending the reasoning explained for scalar states. 

Moreover, if the lifted process has invariant measure \(\overline \pi\), we define the expected risk
\[
\cR(f)
= \int \nor{x'-f(\overline x)}_\cX^2 \,P(\overline x,dx')\,\overline \pi(d\overline x).
\]
This functional is well defined on
\[
L^2(\overline\pi, \cX^p;\cX)
:= \left\{
f:\cX^p \to \cX:
\nor{f}_{\overline \pi}^2
:= \int \nor{f(\overline x)}_\cX^2\,\overline \pi(d\overline x)
<\infty \right\}
\]
by Condition~\eqref{2mom_p}, and again a weaker moment condition would suffice, see~\eqref{pi2mom}. Then \(f_\star\) is the minimizer of the expected risk in \(L^2(\overline \pi, \cX^p;\cX)\), and
\[
\cE(f,f_\star)=\cR(f)-\cR(f_\star),
\qquad f\in L^2(\overline \pi, \cX^p;\cX).
\]
As shown next, the least squares estimator naturally generalizes to this higher-order setting.

\subsection{Least squares estimator}

The discussion closely follows that for first order systems with vector-valued states. 
We consider nonlinear estimators of the form
\be\label{lin_est_vv_op}
f(\overline x) = W^*\Phi(\overline x), \qquad \overline x \in \cX^p,
\ee
where \( (\cH, \inner{\cdot}{\cdot}) \) is a separable Hilbert space, \( \Phi : \cX^p \to \cH \) a measurable map, and \( W \in \cL_2(\cX, \cH)\). Note that the main difference here is that the input is a  tuple $\overline x\in \cX^p$ and the domains and ranges of $\Phi$ and $W$ are updated accordingly. 
Given a  truncated trajectory $x_0, \dots, x_T$, for each $t= 0, \dots T-p$ let 
\be\label{augm_state}
\overline x_t = (x_t, \ldots, x_{t+p-1}). 
\ee
Then, estimator coefficients are computed via regularized least squares by minimizing
\be\label{rls_vv_high_ord}
\frac{1}{T-p+1} \sum_{t=0}^{T-p} \nor{x_{t+p} - W^*\Phi(\overline x_t)}_{\cX}^2
+ \la \nor{W}_{\cL_2(\cX, \cH)}^2, \qquad \la > 0 .
\ee
The error is measured by the norm of $\cX$ and the regularization term is the (squared) Hilbert--Schmidt norm. The unique solution of the above problem is denoted by \( \wh W_\la \), and the corresponding estimator by \( \wh f_\la \). By standard computations \( \wh W_\la \) can be written as
\be\label{ridgesol_p}
\wh W_\la = (\wh \Sigma + \la I)^{-1} \wh h ,
\ee
where
\be\label{sigmah_p}
\wh \Sigma = \frac{1}{T-p+1} \sum_{t=0}^{T-p} \Phi(\overline x_t) \otimes \Phi( \overline x_t),
\qquad
\wh h = \frac{1}{T-p+1} \sum_{t=0}^{T-p} \Phi( \overline x_t) \otimes x_{t+p}.
\ee
As in the scalar case, the above expression can be computed if  \( \cH \) is finite dimensional.
Moreover, for any kernel \( k:\cX^p\times \cX^p\to \bbR  \) such that $k(\overline x, \overline y)= \inner{\Phi(\overline x)}{\Phi( \overline y)}$ for every $\overline x, \overline y \in \cX^p$, the solution can also be written  for all $ \overline x \in \cX^p$ as
\[
\wh f_\la(\overline x)
= \sum_{t=0}^{T-p} C_t \, k(\overline x_t, \overline x), 
\]
where \( C = (C_0,\ldots,C_{T-p}) \in \cX^{T-p+1} \) is given by \( C = (\wh K + (T-p+1) \la I)^{-1} \wh y \), with  \( \wh y = (x_p,x_{p+1}, \dots,x_T) \in \cX^{T-p+1} \) and \( \wh K\)  is a linear operator from $\cX^{T-p+1}$ to $\cX^{T-p+1}$ defined, for every \(C\in\cX^{T-p+1}\), by
\[
(\wh K C)_t=\sum_{s=0}^{T-p} k(\overline x_t, \overline x_s)C_s,
\qquad t=0,\ldots,T-p.
\]

As before, online computations can be derived straightforwardly. We conclude, as in the vector-valued setting, with a few remarks on possible choices of the feature map \(\Phi\).

\paragraph{Order-\(p\) autoregressive models}

The basic observation is that the structure of the augmented states~\eqref{augm_state} offers an opportunity to design ad hoc models.

An augmented state $\overline x= (x_{1}, \dots, x_p)\in \cX^p$ can be seen as one vector, but also as a collection of $p$ vectors in $\cX$.
Following this latter perspective, one can consider maps
$\Phi_j:\cX\to \cH_j$, \( j=1,\dots,p\), where each \(\cH_j\) is a Hilbert space. Then we can  let
\[
\cH=\bigoplus_{j=1}^p \cH_j, 
\]
and define $\Phi: \cX^p\to \cH$ by
\[
\Phi(\overline x)=\big(\Phi_1(x_1),\dots,\Phi_p(x_p)\big),
\qquad \overline x=(x_1,\dots,x_p).
\]
For $W_j\in \cL_2(\cX, \cH_j)$, \( j=1,\dots,p\), define $W:\cX^p\to\cH$ by
\[
W\overline x=(W_1x_1,\dots,W_px_p).
\]
Then \(W\in \cL_2(\cX^p,\cH)\) and 
\[
f(\overline x)
= W^*\Phi(\overline x)
= \sum_{j=1}^p W_j^* \Phi_j(x_j).
\]
When, for all $j=1,\dots,p$, $\cH_j=\cX=\bbR^d$ and the maps $\Phi_j$ are identities, {the above model becomes the prediction function associated with a vector autoregressive model of order \(p\), or \( \mathrm{VAR}(p) \)} \citep{lim2021time}
\[
f(\overline x)
= \sum_{j=1}^p W_j^* x_j.
\]
It becomes nonlinear when more general maps are considered.

After discussing  specific models for estimating higher-order systems, we  go back to the general model to state and discuss corresponding learning guarantees.

\subsection{Learning guarantees}

The learning guarantees for higher-order systems have the same structure as those obtained for first-order systems with vector-valued states. The main difference is that the  inputs are the augmented states
\(\overline X_t\in\cX^p\), which form a first-order Markov process with transition kernel \(Q\), and that a trajectory \(x_0,\dots,x_T\) provides \(T-p+1\) input-output pairs.

\begin{assumption}[Moment condition and boundedness]
\label{bound_p}
Assume that there exist constants \(M,R>0\), \(\sigma>0\), and \(\kappa>0\) such that Condition~\eqref{2mom_p} holds and, for every \(\overline x\in\cX^p\),
\be\label{res_bound_p}
\int_{\cX}
\nor{x'-f_\star(\overline x)}_{\cX}^m
,P(\overline x,dx')
\le
\frac{m!}{2}\sigma^2M^{m-2},
\qquad \forall m\ge 2.
\ee
Moreover, assume that
\be\label{bound_phi_p}
\nor{\Phi(\overline x)}\le\kappa,
\qquad \forall \overline x\in\cX^p.
\ee
\end{assumption}

Well specified systems are described by the following assumption.

\begin{assumption}[Well specification]
\label{wellspec_p}
Assume that there exists \(W_\star\in\cL_2(\cX,\cH)\) such that
\[
f_\star(\overline x)
= W_\star^*\Phi(\overline x),
\qquad \forall \overline x\in\cX^p.
\]
\end{assumption}

The above assumptions are the direct analogues of Assumptions~\ref{bound_vv} and~\ref{wellspec_vv}, with the state \(x\) replaced by the augmented state \(\overline x\).

\begin{theorem}[Rates for higher-order systems]
\label{thm:wellspecsys_p}
Let Assumptions~\ref{UGE_p}, \ref{bound_p}, and~\ref{wellspec_p} hold. For every \(\delta\in(0,1)\), if \(T \ge p \) and 
\[
\lambda
\ge \frac{8C\kappa^2}{1-\rho}
\left( \sqrt{\frac{2\log(6/\delta)}{T-p+1}} + \frac{1}{T-p+1} \right),
\]
then, with probability at least \(1-\delta\),
\begin{equation}
\label{eq:sample-rate-bound-simplified_p}
\cE\bigl(\wh f_\lambda,f_\star\bigr)
\lesssim
\frac{C^2D}{(1-\rho)^2}
\left(\frac{2\log(6/\delta)}{\lambda(T-p+1)} + \frac{1}{\lambda(T-p+1)^2}\right)
+ \lambda \nor{W_\star}_{\cL_2(\cX,\cH)}^2,
\end{equation}
where the constant \(D\) is defined in the proof.
In particular, setting
\[
\lambda=\lambda_T\asymp (T-p+1)^{-1/2},
\]
with probability at least \(1-\delta\),
\[
\cE\bigl(\wh f_{\lambda_T},f_\star\bigr)
\lesssim
\frac{C^2D}{(1-\rho)^2}
\left(\frac{2\log(6/\delta)}{\sqrt{T-p+1}}+\frac{1}{(T-p+1)^{3/2}}\right)
+ \frac{\nor{W_\star}_{\cL_2(\cX,\cH)}^2}{\sqrt{T-p+1}}.
\]
\end{theorem}
The proof of the above result is given in Appendix~\ref{sec:furth_proofs} and follows from a straightforward adaptation of the scalar- and vector-valued analysis.
The bound has the same structure as the one obtained for first-order systems with vector-valued states, with the number of available observations reduced from \(T\) to \(T-p+1\). The order \(p\) also enters implicitly through the constants \(C\) and \(\rho\), which describe the uniform geometric ergodicity of the lifted process, and possibly through the constants in Assumptions~\ref{bound_p} and~\ref{wellspec_p}. Thus, for fixed \(p\) and uniformly controlled constants, the rate remains of order \((T-p+1)^{-1/2}\).

We conclude by noting that while we only state the finite-sample rate,  universal consistency  results can be obtained by the same argument as in the scalar-valued case; see Theorem~\ref{thm:universal}.

\section{Learning dynamical systems with finite-state spaces}
In this section, we focus on the setting where the state space is finite, that is 
\( \cX = \{\mathbf x_1, \dots, \mathbf x_N\} \). We consider a {time-}homogeneous Markov process \( (X_t)_{t\in \bbN_0} \) with initial distribution \( \mu_0 \), where the transition kernel can be seen as a map \( P:\cX \times \cX \to [0,1] \) or alternatively as an \( N \times N \) transition matrix.

\subsection{Setting}
If for all \( x \in \cX \) and all \( x' \neq x'' \in \cX \),
\be\label{unique}
P(x,x') \neq P(x,x''),
\ee
then the optimal prediction function \( f_\star : \cX \to \cX \) is given, for each \(x\in \cX\), by 
\be\label{bayes_rule}
f_\star(x)=
\arg\max_{x' \in \cX } P(x, x').
\ee
The goal is to estimate $f_\star$ from a single trajectory $x_0, x_1, \dots, x_{T}$. Assuming the process to be ergodic with invariant measure \( \pi \), see Assumption~\ref{UGE}, the quality of an empirical estimator \( \wh f \) is measured by
\begin{equation}
\label{eq:excess-risk-ind}
\cE(\wh f, f_\star) = \sum_{x\in \cX} 
\mathds{1}\{ \wh f(x) \neq f_\star(x) \} \,
\pi(x).
\end{equation}
Note that here the invariant measure $\pi$ is a probability mass function. We add a few comments.

\paragraph{Margin.}
Condition~\eqref{unique} is a simplifying assumption. In particular, since $\cX$ is finite, it implies that there exists \( \gamma>0 \) such that for all \( x\in\cX \),
\be\label{gap}
P(x,f_\star(x)) - \max_{x'\neq f_\star(x)} P(x,x') \ge \gamma.
\ee
This is akin to margin conditions in multi-category classification \citep{tsybakov2004optimal,vigogna2022multiclass} and gap conditions in bandits 
Condition~\eqref{gap} can be restrictive for some common examples like random walk. \citep{lattimore2020bandit}.
As discussed later, it can be avoided by considering an error metric other than~\eqref{eq:excess-risk-ind}.

\paragraph{Optimal prediction function.}
The optimal prediction function \( f_\star \) is the analogue of the Bayes classification rule in multi-category classification. In particular, it minimizes the \emph{instantaneous} risk
\[
\sum_{x' \in \cX} \mathds{1}\{ x' \neq f(x) \} \, P(x, x'),
\qquad x \in \cX.
\]

\paragraph{Risks and ergodicity.}
For a  process  with invariant measure \(\pi\), let
\[
\cR(f) = \sum_{x \in \cX} \sum_{x' \in \cX}
\mathds{1}\{ x' \neq f(x) \} \, \pi(x)\, P(x, x').
\]
It is easy to see that \( f_\star \) is a minimizer of \( \cR \). Under Condition~\eqref{unique} (and hence~\eqref{gap}), it is also easy to see that 
\[
 \cE(f, f_\star) \le \frac 1 \gamma \left(\cR(f) - \cR(f_\star)\right).
\]
Moreover, in general, for all \( f : \cX \to \cX \),
\[
\cR(f) - \cR(f_\star)\le \cE(f, f_\star).
\]
Hence, considering \( \cR(f) - \cR(f_\star) \) provides a weaker way to measure the error of a prediction function \( f \), and allows avoiding Condition~\eqref{unique} and hence~\eqref{gap}. We will further comment on this choice later.

The above comments highlight many analogies between one-step prediction with finite-state spaces and multi-category classification. The main difference is that in one-step prediction with finite states both input and output coincide with the same finite-state space.

A possible way to derive an empirical estimator is to use the data to obtain an approximation \( \wh P \) of the transition matrix and plug it into Equation~\eqref{bayes_rule} to obtain a prediction function. Here, we follow a different route, exploring a surrogate loss function approach common in multiclass learning. Again, we focus on least squares.

\subsection{Surrogate least squares approach}
\label{sec:surrogate-least-squares}

Next, we describe an approach inspired by multicategory classification and structured prediction; see, e.g., \citet{mroueh2012multiclass, ciliberto2020general}.

To adopt a least-squares approach, we introduce an encoder--decoder pair.
\begin{equation}
\label{eq:finite-states-encoder-decoder}
E: \cX \to \bbR^q,\quad D: \bbR^q\to \cX,
\end{equation}
where $q$ is the encoder dimension.
The simplest example is the {\em one-hot encoding}, that corresponds to \(q=N\) and for \(j=1, \dots, N\)
\[
E(\mathbf x_j)= e_j,
\]
where \(e_1, \dots, e_N\) is the canonical basis in \(\bbR^N\). The corresponding decoding is given, for any \(s \in \bbR^N\), by
\[
D(s)\in \argmax_{x\in \cX} \langle E(x), s\rangle_{\bbR^N},
\]
where one maximizer is taken arbitrarily in case of ties. Given an encoding and following the discussion for vector-valued states, the optimal prediction function is defined as
\be\label{sur_gstar}
g_\star(x)= \sum_{x'\in \cX} E(x') P(x,x'),
\qquad x \in \cX.
\ee
Using a surrogate least squares approach we aim at estimating $g_\star$ from an encoded truncated trajectory $E(x_0), \dots, E(x_{T})$. We add several comments.

\paragraph{Optimal prediction functions and risks.}
Given the encoder, from the scalar and vector-valued settings discussed in previous sections, we know that the prediction function~\eqref{sur_gstar} minimizes the \emph{instantaneous} least squares risk
\[
\sum_{x'\in\cX} \nor{E(x') - g(x)}_{\bbR^q}^2 \, P(x,x'),
\qquad x\in\cX.
\]
Moreover, for an uniformly geometrically ergodic process (see Assumption~\ref{UGE}), it minimizes the risk
\[
\cI(g)
:= \sum_{x\in\cX}\sum_{x'\in\cX} \nor{E(x') - g(x)}_{\bbR^q}^2 \, \pi(x)\, P(x,x'),
\qquad g: E(\cX) \to \bbR^q.
\]
and the following relation holds
\[
\cI(g) - \cI(g_\star)
= \nor{g - g_\star}_\pi^2,
\]
where \(\nor{g}_\pi^2 = \sum_{x\in\cX} \nor{g(x)}_{\bbR^q}^2 \,\pi(x)\).

\paragraph{Explicit formulas via the transition matrix.}
In the finite-state setting we can derive more explicit formulas using the transition matrix. In particular, considering the one-hot encoding, then for every \(x\in \cX\),
\be\label{WS_finitestate}
g_\star(x) = P^\top E(x),
\ee
where \(P\) is the transition matrix. As we discuss later, this observation ensures that Assumption~\ref{wellspec_vv} can always be satisfied in this setting taking \(W_\star= P\). Further insights can be derived restricting the minimization of the risk \(\cI\) to functions of the form \(g(x)= W^\top E(x)\), \(x \in \cX\). It is easy to see that the optimal $W_\star$ satisfies the normal equation 
\[
\Sigma W_\star= h,
\]
where
\[
\Sigma = \operatorname{diag}(\pi), \qquad h=\operatorname{diag}(\pi)\,P,
\]
where we identified \(\pi\) with the vector \((\pi(\mathbf x_1),\dots,\pi(\mathbf x_N)) \in \bbR^N\). Moreover,
the matrix \( P \) solves the normal equation \( \Sigma W = h \). If, in addition, \( \pi \) has full support, for example, if the finite-state chain is \emph{irreducible}, then \( \Sigma \) is invertible and \( P \) is the unique solution.

\paragraph{Surrogate consistency and comparison inequalities}
The following lemma relates the optimal prediction functions and corresponding error measures when considering one-hot encoding.

\begin{lemma}
\label{comp_ineq1}
If \(E\) is the one-hot encoding and \(D\) the corresponding decoding, then for every \( x \in \cX \),
\be\label{fisher}
f_\star(x) = D \circ g_\star(x).
\ee
Moreover, if Condition~\eqref{unique}, and hence~\eqref{gap}, holds, then for all \( g:\cX\to \bbR^N \), and \( f=D\circ g \), the following inequality holds
\begin{equation}\label{eq:boundE}
\cE(f,f_\star)
\le \frac{2}{\gamma^2} \bigl(\cI(g)-\cI(g_\star)\bigr).
\end{equation}
\end{lemma}

\begin{remark}[Comparison inequality without a margin condition]
If we drop Condition~\eqref{unique} and ~\eqref{gap} the following inequality holds for every function \( g: \cX \to \bbR^q \) and \( f := D \circ g \),
\begin{equation}\label{eq:boundR}
\cR(f)-\cR(f_\star)
\le \sqrt{2}\,\nor{g-g_\star}_\pi
= \sqrt{2 \bigl(\cI(g)-\cI(g_\star)\bigr)}.
\end{equation}
This inequality has the same structure as  known comparison inequalities between the excess prediction risk and the excess least squares surrogate risk in multicategory classification and structured prediction; see, e.g.,~\citet{mroueh2012multiclass, ciliberto2020general}.
\end{remark}

\subsection{Least squares estimator}

Given a truncated trajectory \(x_0,\dots,x_T\), consider an encoder/decoder pair \(E:\cX\to \bbR^q\), \(D:\bbR^q\to \cX\), a feature map \(\Phi:\bbR^q\to \cH\), where \(\cH\) is a Hilbert space. 

We consider surrogate estimators of the form \(g(E(x))=W^*\Phi(E(x))\), \(x\in\cX\), where \(W\in \cL_2(\bbR^q,\cH)\). The corresponding decoded estimator is \(f=D\circ g\).

The estimator is computed via regularized least squares by minimizing
\be\label{ridge_finite}
\frac1T\sum_{t=0}^{T-1}\|E(x_{t+1})-W^*\Phi(E(x_t))\|_{\bbR^q}^2
+\lambda \|W\|_{\cL_2(\bbR^q,\cH)}^2,
\qquad \lambda>0.
\ee
We denote by \(\wh W_\lambda\) the unique minimizer, with corresponding estimators \(\wh g_\lambda(E(x))=\wh W_\lambda^*\Phi(E(x))\) and \(\wh f_\lambda=D\circ \wh g_\lambda\). Before discussing learning guarantees for the above estimator, we discuss how it can be written in a more explicit way that allows us to draw comparisons with basic transition matrix estimators.

\paragraph{Explicit formulas and comparison with empirical transition matrix estimate.}
Given a truncated trajectory \(x_0,\dots,x_{T}\), for \(i,j\in\{1,\dots,N\}\) consider the quantities
\[
n(i)= \sum_{t=0}^{T-1} \mathds{1}\{x_t=\mathbf x_i\},
\qquad
n(i,j)= \sum_{t=0}^{T-1} \mathds{1}\{x_t=\mathbf x_i,\ x_{t+1}=\mathbf x_j\}.
\]
Then we can define an empirical estimate of the transition matrix \(P\) by
\[
\wh P_{ij}=
\begin{cases}
\displaystyle \frac{n(i,j)}{n(i)}, & \text{if } n(i)>0,\\[1ex]
0, & \text{if } n(i)=0.
\end{cases}
\]
A simple estimator is then 
\( \wh f (\mathbf x_i)= \argmax_{j=1, \dots, N} \wh P _{ij} \), with ties broken arbitrarily. Next we contrast this estimator with the least squares estimator defined by~\eqref{ridge_finite}, when considering one-hot encoding and $\Phi$ to be the identity in $\bbR^N$. Towards this end recall that from the discussion of vector-valued states we know that, the minimizer of Problem~\eqref{ridge_finite} can be written as
\be\label{ridgesol_vv_fs}
\wh W_\la = (\wh \Sigma + \la I)^{-1} \wh h ,
\ee
and when choosing $\Phi$ to be the identity
\be\label{sigmah_finite}
\wh \Sigma = \frac{1}{T} \sum_{t=0}^{T-1} E(x_t) \otimes E(x_t),
\qquad
\wh h = \frac{1}{T} \sum_{t=0}^{T-1} E(x_t) \otimes E(x_{t+1}).
\ee
Then it is easy to see that 
\be\label{sigmah_finite_count}
\wh \Sigma = \frac1T \operatorname{diag}\bigl(n(1),\dots,n(N)\bigr),
\qquad
\wh h = \frac1T (n(i,j))_{i,j=1}^N, 
\ee
so that for all $i,j=1, \dots, N$
\[
(\wh W_\la)_{ij}=\frac{n(i,j)}{n(i)+\la T},
\]
or equivalently 
\[
(\wh W_\lambda)_{ij}=
\frac{n(i)}{n(i)+\lambda T}\,\wh P_{ij}.
\]
A few comments are in order. If \(n(i)>0\) for all \(i=1,\dots,N\), and we take \(\la=0\), then \(\wh W_\la=\wh P\). The comparison is more interesting when some \(n(i)\) can be zero, as may in particular happen when \(N\) is larger than \(T\). Indeed, in this case, the entries of \(\wh P\) corresponding to \(n(i)=0\) are set to zero by convention, whereas \(\wh W_\la\) is still well defined for every \(\la>0\). By replacing the empirical normalization \(n(i)\) by \(n(i)+\la T\), states visited only a few times are shrunk more strongly, while states with many observations are affected less. 

\subsection{Learning guarantees}
In this section we derive learning guarantees for the least squares estimator in the finite-state setting.
For simplicity, we consider a one-hot encoding and the corresponding decoding, and take \(\Phi\) to be the identity map. While the analysis extends beyond this setting, these are basic and natural choices that allow us to identify some features of the finite-state case.
We leave to future work more general treatments based on more informed algorithmic choices.

As discussed in the previous sections, the surrogate least squares approach yields guarantees in two steps: first, comparison inequalities relate the finite-state one-step prediction error to the surrogate least squares error, see Lemma~\ref{comp_ineq1}; second, the least squares error is controlled by error estimates analogous to those for vector-valued state prediction. We begin by reviewing how the assumptions in this latter setting adapt to the finite-state setting. 

Recall that, with one-hot encoding we have \( q = N \). Then, note that Assumption~\ref{2mom_vv} is always satisfied with \(R=1\), since, for all \(x\in \cX\), \(\nor{E(x)}=1\), so that
\[
\sum_{x'\in \cX} \nor{E(x')}_{\bbR^N}^2P(x,x')=1.
\]
Similarly, Assumption~\ref{bound_phi_vv} becomes
\[
\nor{\Phi(E(x))}_{\bbR^N} = \nor{E(x)}_{\bbR^N} \le \kappa
\]
so that it is always satisfied with \(\kappa=1\). We still need the following moment condition.

\begin{assumption}[Moment condition] \label{bound_finite} 
Assume that there exist constants \(M,\sigma>0\) such that, for all \(x \in \cX\),
{
\be\label{res_bound}
\sum_{x' \in \cX}\nor{E(x') - g_\star(x)}_{\bbR^N}^m\,P(x, x')
\le
\frac{m!}{2}\,\sigma^2 M^{m-2},
\qquad \forall m \ge 2 .
\ee
}
\end{assumption}
One can immediately derive some basic bounds on \(M\) and \(\sigma\). Since, for all \(x,x'\in\cX\),
\[
\nor{E(x')-g_\star(x)}_{\bbR^N}^2\le 2,
\]
and
\[
\sum_{x' \in \cX}\nor{E(x')-g_\star(x)}_{\bbR^N}^2\,P(x,x')
= 1-\sum_{x'\in\cX}P(x,x')^2
\le 1,
\]
one can always take \(M=\sqrt 2\) and \(\sigma=1\). The above equality follows by noting that
\[
g_\star(x)
=\bigl(P(x,\mathbf x_1),\ldots,P(x,\mathbf x_N)\bigr).
\]
These worst-case bounds can be refined depending on the transition kernel. For example, if \(P(x,x')=1/N\) for every \(x,x'\in\cX\), then one can take
\[
M=\sigma=\sqrt{1-\frac1N}.
\]
Keeping these possible choices in mind, we next keep the values generic. Finally, as seen from Equation~\eqref{WS_finitestate}, Assumption~\ref{wellspec_vv} is satisfied with \(W_\star=P\), the transition matrix.
Note that the Hilbert--Schmidt/Frobenius norm of \(P\) satisfies the basic bounds
\begin{equation}
\label{eq:p-hs-upper}
1 \le \nor{P}_{\cL_2(\bbR^N)} \le \sqrt N .
\end{equation}
The lower bound is attained when all transition probabilities are uniform, while the upper bound is attained when each transition is deterministic. This follows by noting that 
\[
\nor{P}_{\cL_2(\bbR^N)}^2
= \sum_{x,x'\in \cX} P(x,x')^2
= \sum_{x\in\cX} \nor{P(x,\cdot)}_{\bbR^N}^2,
\]
and \( 1/N \le \nor{P(x,\cdot)}_{\bbR^N}^2 \le 1 \).

Given the above discussion we can state the main result of this section.

\begin{theorem}[Rates for finite-state spaces] \label{thm:fs}
Let \( \wh g_\lambda \) be an estimator corresponding to the solution of~\eqref{ridge_finite} obtained using the one-hot encoding \(E:\cX\to\bbR^N\) and the identity feature map \(\Phi(x) = x\). Set \( \wh f_\lambda := D \circ \wh g_\lambda \), where \( D \) is the corresponding decoder.

Suppose that Assumptions~\ref{UGE} and~\ref{bound_finite} hold. Assume, in addition, Condition~\eqref{unique} holds and \(\gamma>0\) is the corresponding margin defined in~\eqref{gap}. Let \(\delta\in(0,1)\), if
\[
\lambda
\geq \frac{8C}{1-\rho} \left( \sqrt{\frac{2\log(6/\delta)}{T}} + \frac{1}{T} \right),
\]
then, with probability at least \(1-\delta\),
\[
\cE(\wh f_\lambda,f_\star)
\lesssim
\frac{ C^2 D }{\gamma^2 (1-\rho)^2} \left(\frac{2\log (6/\delta)}{\la T}+\frac{1}{\la T^2}\right)
+ \frac{\lambda}{\gamma^2} \nor{P}_{\cL_2(\bbR^N)}^2.
\]
In particular, setting \(\lambda=\lambda_T\asymp T^{-1/2}\) sufficiently large so that the condition above is satisfied gives with probability at least \( 1 - \delta \),
\[
\cE\big(\wh f_{\la_T}, f_\star\big)
\lesssim
\frac{ C^2 D }{\gamma^2 (1-\rho)^2} \left(\frac{2\log (6/\delta)}{\sqrt{T}}+\frac{1}{T^{3/2}}\right)
+ \frac{\nor{P}_{\cL_2(\bbR^N)}^2}{\gamma^2 \sqrt{T}}.
\]
\end{theorem}

The proof of the above result is given in Appendix~\ref{sec:appendix-finite-states} and leverages the result for vector-valued state spaces. Indeed, the bound strongly resembles the one given in Theorem~\ref{thm:wellspecsys_vv}. In particular, it yields the same rate of the order \(T^{-1/2}\). A first   observation regards the  (hidden) dependence on the cardinality of the state space. Indeed, under one-hot encoding, the well-specified parameter is the transition matrix \(P\), and both the approximation term and the constant \(D\) depend on \(\nor{P}_{\cL_2(\bbR^N)}\); more precisely, \(D\) contains a term of order \(\bigl(1+\nor{P}_{\cL_2(\bbR^N)}\bigr)^2\). According to Equation~\eqref{eq:p-hs-upper}, \( \nor{P}_{\cL_2(\bbR^N)}^2 \leq N \), so
the resulting bound has, in the worst case, a linear dependence on \(N\). This dependence is however adaptive to the structure of the transition matrix: it is largest for nearly deterministic transitions and smaller when the transition probabilities are more diffuse, i.e., closer to uniform. Unlike guarantees for estimating the full transition matrix uniformly over its rows~\citep{wolfer2019minimax,chan2021learning}, our bound explicitly involves neither the minimum invariant probability (i.e., \( \min_{x \in \cX} \pi(x) \)), nor a \emph{cover time} (i.e., the first time by which the trajectory has visited every state at least once). The reason is simply that the prediction error \( \cE(f,f_\star) \) is averaged with respect to \(\pi\), so that states that are rarely visited are assigned less weight. Closely related results on optimal next-state prediction are given by \citet{han2021optimal}, although they consider stationary trajectories, prediction risk using Kullback–Leibler divergence, and estimators which are tailored specifically to finite-state Markov chains. Finally, we note that while we only discuss finite sample bounds, the universal learning results in  Theorem~\ref{thm:universal} can be adapted to the finite-state setting by exploiting the derived comparison inequalities.

\section{Learning  Koopman operators}

Thus far we have considered only the {one-step prediction} problem in different settings. In this section, we consider how to learn the so-called Koopman, or Markov, operator, which can be shown to provide a complete description of the stochastic dynamics of a system. In particular, we revisit the approach proposed in \cite{kostic2022learning}, and derive guarantees for learning the Koopman operator of a uniformly ergodic system with a least squares estimator using a single trajectory.

\subsection{Setting}

Let \( (\cX, \cB(\cX)) \) be a measurable space, and let \( (X_t)_{t\in \bbN_0} \) be a homogeneous Markov process with initial distribution \(\mu_0\) and transition kernel \(P\).
Let \(\cM_b(\cX)\) be the space of real-valued bounded measurable functions on \(\cX\) with the uniform norm. The Koopman operator \(A_\star : \cM_b(\cX) \to \cM_b(\cX)\) is defined by
\be\label{koop_op}
A_\star f(x) = \int f(x') \, P(x, dx'), \qquad f \in \cM_b(\cX), \ x \in \cX.
\ee
 The goal is to estimate the Koopman operator from a single trajectory $x_0, \dots, x_T$. Assuming the process is ergodic with invariant measure \( \pi \), see Assumption~\ref{UGE}, the quality of an empirical estimator \( \wh A \) is measured by
\begin{equation}
\label{eq:excess-risk-koop_vv}
\cE_{\cF}(\wh A, A_\star) = \sup_{f\in \cF} \int (\wh A f(x) - A_\star f(x))^2 \, \pi(dx),
\end{equation}
where $\cF\subset \cM_b(\cX)$ such that the above supremum is well defined and finite.   Later, we will take $\cF$ to be a unit ball of a suitable normed space of observables.

We add a few comments.

\paragraph{Properties of the Koopman operator.}
The integral in \eqref{koop_op} is well defined for all \(f \in \cM_b(\cX)\), and moreover
\(
\nor{A_\star f}_\infty \le \nor{f}_\infty.
\)
Then, \(A_\star\) is a bounded linear operator on \(\cM_b(\cX)\), with \(\nor{A_\star}=1\). In particular, this ensures that the error measure~\eqref{eq:excess-risk-koop_vv} is well defined whenever \(\wh A : \cF\to \cM_b(\cX)\) is also a bounded linear operator.

\paragraph{Optimality of Koopman operator.}
The Koopman operator~\eqref{koop_op} can be written as 
\be\label{koop_op2}
A_\star f(x)= \bbE[f(X_{t+1}) \mid X_t=x].
\ee
In particular, for every \(f\in \cM_b(\cX)\) and every \(x\in \cX\), \(A_\star f(x)\) minimizes the {\em instantaneous} least squares risk
\be\label{koop_inst_risk}
\bbE[(f(X_{t+1})-a)^2 \mid X_t=x]
= \int (f(x')-a)^2 \, P(x,dx'),
\qquad a\in \bbR.
\ee
This risk and \(A_\star\) are well defined for all \(f\in \cM_b(\cX)\) as discussed next.

\paragraph{Risks and ergodicity.}
We next discuss analogues of the expected risk for the Koopman operator. For an ergodic process with invariant measure \( \pi \), and for every \(f \in \cF\), let
\begin{equation}
\label{eq:risk-def-koop}
\cR_f(A)
= \int \left( f(x') - A f(x) \right)^2 \,P(x, dx')\, \pi(dx).
\end{equation}
If \(A : \cF \to \cM_b(\cX)\) is bounded and linear, then the above functional is well defined. Moreover, since for every \(f \in \cF\) and every \(x \in \cX\), \(A_\star f(x)\) minimizes the instantaneous least squares risk~\eqref{koop_inst_risk}, we have
\[
\int (f(x')-A_\star f(x))^2\,P(x,dx')
\le \int (f(x')-A f(x))^2\,P(x,dx'),
\qquad x\in \cX.
\]
Integrating both sides with respect to \(\pi\), we obtain
\[
\cR_f(A_\star)\le \cR_f(A),
\]
so that \(A_\star\) minimizes \(\cR_f\) for every \(f\in \cF\). Moreover, taking the supremum over \(f \in \cF\), we can define the class-wise risk on \(\cF\) by
\be\label{F_risk}
\cR_\cF(A)=\sup_{f\in\cF}\cR_f(A).
\ee
Since \(\cR_f(A_\star)\le \cR_f(A)\) for every \(f\in\cF\), it follows that \(A_\star\) is a minimizer of \(\cR_\cF\). In this view, we could consider the error measure
\[
\cR_\cF(A)-\cR_\cF(A_\star).
\]
This choice and the error measure in~\eqref{eq:excess-risk-koop_vv} are related by the following inequality
\be\label{koop_excess}
\cR_\cF(A)-\cR_\cF(A_\star)\le \cE_\cF(A,A_\star).
\ee
Indeed, it is then easy to verify that, for every \(f \in \cF\),
\[
\int \left(Af(x)-A_\star f(x)\right)^2\,\pi(dx)
= \cR_f(A)-\cR_f(A_\star), 
\]
so that,
\[
\cE_\cF(A,A_\star)
= \sup_{f\in\cF}\big(\cR_f(A)-\cR_f(A_\star)\big).
\]
Then~\eqref{koop_excess} follows by noting that
\[
\sup_{f\in\cF}\cR_f(A)-\sup_{f\in\cF}\cR_f(A_\star)
\le \sup_{f\in\cF}\big(\cR_f(A)-\cR_f(A_\star)\big).
\]

Finally, we can also show that
\begin{equation}\label{ex2risk}
\cE_\cF(A,A_\star)
= \sup_{f\in\cF}
\bigl(\cR_f(A)-\cR_f(A_\star)\bigr)
\le \sup_{f\in\cF}\cR_f(A)
= \cR_\cF(A).
\end{equation}
Indeed,  for every \(f\in\cF\)
\[
\cR_f(A)-\cR_f(A_\star)\le \cR_f(A),
\]
because \(\cR_f(A_\star)\ge 0\),and  taking the supremum over \(f\in\cF\) yields inequality~\eqref{ex2risk}.

\paragraph{Risks and norms.}
We  note that if \(\cF\) is the unit ball in a normed space $\overline{\cF}$, then the metric~\eqref{eq:excess-risk-koop_vv} corresponds to the squared operator norm, that is,
\[
\cE_\cF(\wh A,A_\star)
= \nor{\wh A-A_\star}_{\cL(\overline{\cF},L^2_\pi)}^2.
\]
In this view, even stronger norms such as the Hilbert--Schmidt norm could be considered.

\paragraph{Observables evolution.}
The Koopman operator can be seen to describe the evolution of \emph{all} observables of the system. In this view, an observable is any \(f \in \cM_b(\cX)\), and its evolution is given exactly by \eqref{koop_op2}. Indeed, the Koopman operator uniquely determines the transition kernel, in the sense that \( P(x, B) = A_\star \mathbf 1_B(x) \); see Appendix~\ref{sec:backMP}.

\paragraph{Koopman mode decomposition.}
The study of Koopman operators allows one to analyze nonlinear dynamical systems similarly to linear ones. In particular, the spectral properties of the Koopman operator can be used to perform dimensionality reduction of the corresponding dynamical system, an approach called Koopman mode decomposition. Learning guarantees for the Koopman operator in sufficiently strong norms directly yield guarantees for the corresponding spectral quantities. In this work we focus on \(\cE_\cF\), deferring the analysis of other norms and spectral quantities.

\subsection{Least squares estimator} 

Before describing the least squares estimators, we provide some context for their derivation. 

Given a measurable \( \Phi : \cX \to \cH \), such that for all \(x\in \cX\)
\be\label{boundphi}
\nor{\Phi(x)}\le \kappa,
\ee
we restrict the space of possible observables to functions of the form 
\be\label{obs_koop_phi}
f(x) = \scal{v}{\Phi(x)}, \qquad x \in \cX.
\ee
We will assume  the map 
\begin{equation}\label{inj}
v\to \scal{v}{\Phi(\cdot)}
\end{equation}
to be injective. Note that a sufficient condition is
\[
\overline{\operatorname{span}\{ \Phi(x): x \in \cX \}}
= \cH.
\]
Then we consider estimators of the form 
\be\label{lin_est_kopp}
Af(x) = \scal{Wv}{\Phi(x)}, \qquad x \in \cX,
\ee
for every \(f\) as in~\eqref{obs_koop_phi}, and where \( W \in \cL_2(\cH)\).

With these choices we can derive an expression for the risk~\eqref{F_risk}, which is easier to approximate from data. Indeed, if for any \(f\) as in~\eqref{obs_koop_phi} we consider the risk \(\cR_f\), see~\eqref{eq:risk-def-koop}, then for \(A\) as in~\eqref{lin_est_kopp}
\[
\cR_f(A)
= \int \left( \scal{v}{\Phi( x')} - \scal{Wv}{\Phi(x)} \right)^2 \,P(x, dx')\, \pi(dx)
=  \int \left( \scal{v}{\Phi( x') - W^*\Phi(x)} \right)^2 \,P(x, dx')\, \pi(dx). 
\]
Next, let 
\be\label{F}
\cF = \{f:\cX\to \bbR~|~f= \scal{v}{\Phi(\cdot)}, ~\nor{v}\le 1\}.
\ee 
Then Condition~\eqref{boundphi} implies that 
\[
\sup_{x\in \cX}|f(x)|\le \kappa \nor{v},
\]
so that \(\cF\subset \cM_b(\cX)\).
\begin{remark}[Observables in RKHS]
Note that \( \cF \) can be seen as the unit ball in the space of considered observables~\eqref{obs_koop_phi}. Indeed, the latter is a normed space, since under the injectivity assumption~\eqref{inj}, $\nor{v}$ defines a norm for the corresponding observable~\eqref{obs_koop_phi}. It can be shown that this construction is equivalent to considering as observables the reproducing kernel Hilbert space  space of observables defined by $\Phi$, see \cite{christmann2008support}.
\end{remark}

With the above choice of \(\cF\) we can derive an upper bound for \(\cR_\cF\).

\begin{lemma}
Let \(\cF\) be as in~\eqref{F}. Then, for every \(A\) of the form~\eqref{lin_est_kopp},
\be\label{upper_riskH}
\cR_\cF(A)
\le \int \nor{\Phi(x')- W^* \Phi(x)}^2 \, P(x,dx')\,\pi(dx).
\ee
\end{lemma}

\begin{proof}
For every \(f\in \cF\), there exists \(v\in \cH\) with \(\nor{v}\le 1\) such that
\[
\cR_f(A)
= \int \left( \scal{v}{\Phi(x')-W^*\Phi(x)} \right)^2 \,P(x,dx')\,\pi(dx).
\]
By Cauchy--Schwarz,
\[
\left( \scal{v}{\Phi(x')-W^*\Phi(x)} \right)^2
\le \nor{v}^2 \nor{\Phi(x')-W^*\Phi(x)}^2
\le \nor{\Phi(x')-W^*\Phi(x)}^2.
\]
Integrating both sides gives
\[
\cR_f(A)
\le \int \nor{\Phi(x')- W^* \Phi(x)}^2 \, P(x,dx')\,\pi(dx).
\]
Taking the supremum over \(f\in \cF\) concludes the proof.
\end{proof}


The above discussion suggests considering the regularized least squares 
\[
\frac{1}{T} \sum_{t=0}^{T-1} \nor{\Phi(x_{t+1}) - W^*\Phi(x_t)}^2 + \la \nor{W}_{\cL_2(\cH)}^2,
\qquad \la > 0 .
\]
The error is measured by the norm of $\cX$ and the regularization term is the (squared) Hilbert--Schmidt norm. The unique solution of the above problem is denoted by \( \wh W_\la \), and the corresponding estimator by \( \wh A_\la \). By standard computations \( \wh W_\la \) can be written as
\be\label{ridgesol_koop}
\wh W_\la = (\wh \Sigma + \la I)^{-1} \wh h ,
\ee
where
\be\label{sigmah_koop}
\wh \Sigma = \frac{1}{T} \sum_{t=0}^{T-1} \Phi(x_t) \otimes \Phi(x_t),
\qquad
\wh h = \frac{1}{T} \sum_{t=0}^{T-1} \Phi(x_t) \otimes \Phi(x_{t+1}).
\ee
Note that the only difference with the estimator for {one-step prediction of} vector-valued states~\eqref{rls_vv} is the expression for $\wh h$.

\subsection{Learning guarantees}

The assumptions and corresponding learning guarantees are analogous to those for vector-valued states, with suitable adaptations to the Koopman operator setting considered in this section.

Let $g_\star:\cX\to\cH$ be defined as 
\begin{equation}
\label{eq:koop-g-star}
g_\star(x)
= \int \Phi(x')\,P(x,dx'),
\qquad x\in\cX.
\end{equation}
Under Condition~\eqref{boundphi}, $g_\star$ is well defined and $\nor{g_\star(x)}\le\kappa$, for each $x\in\cX$. Then, consider the following moment condition.

\begin{assumption}[Moment condition]
\label{bound_koop}
Assume that there exist constants $M>0$ and $\sigma>0$ such that, for all $x\in\cX$,
\be\label{res_bound_koop}
\int
\nor{\Phi(x')-g_\star(x)}^m
\,P(x,dx')
\le
\frac{m!}{2}\sigma^2M^{m-2},
\qquad
\forall m\ge 2.
\ee
\end{assumption}
The above assumption is the direct analogue of the moment condition in Assumption~\ref{bound_vv}, with the next state $x'$ replaced by  $\Phi(x')$ and $f_\star$ is replaced by $g_\star$. The boundedness condition~\eqref{boundphi} implies that
$\nor{\Phi(x')-g_\star(x)} \le 2\kappa$, for all $x\in \cX$,  and therefore implies~\eqref{res_bound_koop}, for example with $M=\sigma=2\kappa$. Nevertheless, the moment condition  is weaker and provides a more refined description of the  stochastic transition effect. In particular, for deterministic systems this effect vanishes, and the condition is satisfied trivially.

We also require the following well-specification assumption for the Koopman operator.
\begin{assumption}[Well specification]
\label{wellspec_koop}
Assume there exists $W_\star\in\cL_2(\cH)$ such that, for all $f$ as in~\eqref{obs_koop_phi},
$$
A_\star f(x)
=
\scal{W_\star v}{\Phi(x)},
\qquad
\forall x\in\cX.
$$
\end{assumption}

The above assumption is related to an invariance property of the class of observables~\eqref{obs_koop_phi} induced by $\Phi$. Indeed, it states that every observable of the form $f(x)=\scal{v}{\Phi(x)}$, with $v\in\cH$, is mapped by $A_\star$ to another observable of the same form. More precisely, the injectivity assumption introduced above, see Equation~\eqref{inj}, implies that
$A_\star f=f'$, where \( f'(x)=\scal{W_\star v}{\Phi(x)} \) for each \( x \in \cX \). This invariance assumption is strengthened by requiring \( W_\star \) to be Hilbert--Schmidt rather than merely bounded. Note that Assumption~\ref{wellspec_koop} is also equivalent to assuming that \( g_\star(x)=W_\star^\ast\Phi(x) \), for all \( x\in\cX \).

We are ready to state the learning guarantees for the least-squares estimator in this setting.

\begin{theorem}[Rates for Koopman operators]
\label{thm:wellspecsys_koop}
Let Assumptions~\ref{bound_koop}, \ref{wellspec_koop}, and~\ref{UGE} hold, and let Condition~\eqref{boundphi} be satisfied. Let $\delta\in(0,1)$. If
$$
\la
\ge
\frac{8C\kappa^2}{1-\rho}
\left(
\sqrt{\frac{2\log(6/\delta)}{T}}
+
\frac{1}{T}
\right),
$$
then, with probability at least $1-\delta$,
\begin{equation}
\label{eq:sample-rate-bound-simplified_koop}
\cE_\cF\bigl(\wh A_\la,A_\star\bigr)
\lesssim
\frac{C^2D}{(1-\rho)^2}
\left(\frac{2\log(6/\delta)}{\la T}
+ \frac{1}{\la T^2}\right)
+ \lambda \nor{W_\star}_{\cL_2(\cH)}^2,
\end{equation}
where the constant $D$ can be derived from the proof. In particular, setting
$\la=\la_T\asymp T^{-1/2}$, with probability at least $1-\delta$,
\[
\cE_\cF\bigl(\wh A_{\la_T},A_\star\bigr)
\lesssim
\frac{C^2D}{(1-\rho)^2}
\left(\frac{2\log(6/\delta)}{\sqrt{T}}+\frac{1}{T^{3/2}}\right)
+ \frac{\nor{W_\star}_{\cL_2(\cH)}^2}{\sqrt{T}}.
\]
\end{theorem}

The proof of the above result is given in Appendix~\ref{sec:furth_proofs} and follows from the  vector-valued analysis. Indeed, Koopman operator learning can be seen  as one-step prediction in a suitable {\em lifted} space as explained next. 

{\begin{remark}[Koopman operator learning as one-step prediction in a lifted space]
The function \(g_\star\) defined in~\eqref{eq:koop-g-star} minimizes the \emph{surrogate} risk appearing on the right-hand side of~\eqref{upper_riskH}, namely,
\[
\int \nor{\Phi(x')-g(x)}^2 \,P(x,dx')\,\pi(dx).
\]
Restricting \(g\) to functions of the form~\eqref{lin_est_kopp}, reduces the problem to minimizing
\[
\int \nor{\Phi(x')-W^*\Phi(x)}^2 \,P(x,dx')\,\pi(dx).
\]
Thus, keeping in mind Inequality~\eqref{ex2risk}, Koopman operator learning can be viewed as linear one-step prediction for  the \emph{lifted} states \( x \mapsto \Phi(x)\in\cH \). Similar to the finite-state space learning  setup~\eqref{sec:surrogate-least-squares}, the feature map \(\Phi\) acts as an \emph{encoding} of the state, now taking values in a possibly infinite-dimensional Hilbert space.
\end{remark}}

Given the above remark, it is not surprising that the resulting bound has the same structure as the one obtained for vector-valued states, up to the definition of the constants and the dependence on $\nor{W_\star}_{\cL_2(\cH)}$. The structure and the nature of the different contributions have already been discussed. In particular, the bound can be compared with existing guarantees for Koopman operator learning under i.i.d.\ sampling or dependent but stationary observations, see e.g.  \cite{kostic2022learning,mollenhauer2022kernel,philipp2024error}. While the dependence on the trajectory length is analogous, our result applies to a single trajectory initialized from an arbitrary distribution and makes explicit the additional dependence on the uniform geometric ergodicity constants $C$ and $\rho$.

We note that further results concerning the estimation of Koopman eigenvalues and eigenfunctions are obtained by \citet{kostic2023sharp} under additional structural assumptions. Developing  spectral guarantees  in the present nonstationary setting is possible but left for future work.

Finally, we note that the results on universal learning in previous sections  do not extend naturally to  Koopman operator learning. Also this question is left for future work.

\section{Numerical results}

In this section, we consider several geometrically ergodic stochastic dynamical systems belonging to the class of nonlinear autoregressive processes; see Remark~\ref{rem:ar1}. We present numerical experiments illustrating the behavior of a one-step prediction estimator learned from a single trajectory.

\paragraph{Model description.} We consider a dynamical system corresponding to a scalar nonlinear autoregressive process
\[
X_{t+1}=F(X_t)+N_t,
\qquad t\in\bbN_0,
\]
where \(X_0 \sim \mu_0\) and \((N_t)_{t \in \bbN_0}\) is a sequence of i.i.d.\ random variables, independent of \( X_0 \), with common distribution \(\nu\). We take \(\cX=\bbR\), \(\nu=\cU(-1,1)\), and consider the nonlinear drift function
\[
F(x)
= \alpha x
+ \sum_{k=1}^K \bigl( a_k\sin(kx)+b_k\cos(kx)\bigr),
\qquad x \in \bbR,
\]
where \( \alpha \in \bbR \) and  \( (a_k)_{k=1}^K,(b_k)_{k=1}^K \subset \bbR \).

Since the noise is centered, the optimal one-step prediction function satisfies
\(
f_\star(x)
= F(x),
\)
which allows us to compare the learned estimator directly with the true drift. 
Under suitable contractivity assumptions on \( F \), for example when the linear part dominates at infinity with \( |\alpha|<1 \), the induced Markov process admits a unique invariant probability measure \( \pi \) and is geometrically ergodic; see, e.g., \citet[Example~11.4.3]{douc2018markov}.

\paragraph{Experimental setup.} The following protocol was used to assess the performance of the least squares estimator~\eqref{eq:krr-estimator-r}:
\begin{itemize}
    \item \textbf{Trajectory datasets.} We generate a pair of trajectories of length 15{,}000 using fixed parameters, one to be used for training and the other held out for testing.
    
    \item \textbf{i.i.d.\ datasets.} We generate long trajectories of length 150{,}000 steps and approximate the invariant distribution \( \pi \) via uniform subsampling to obtain training and test datasets of 15{,}000 samples each. Separately, to ensure convergence to invariant, multiple trajectories are generated from uniformly sampled initial conditions.

    \item \textbf{Testing procedure.} We fit 25 least squares estimators with a Gaussian kernel, using subsets of the training data ranging from 100 to 10{,}000 samples (sequentially). The kernel bandwidth and the regularization parameter \(\lambda\) were selected by grid search using cross-validation on the training data. Each estimator is evaluated in terms of mean squared error  on both the test trajectory and the independent i.i.d.~test set drawn from \( \pi \). In addition, we include an oracle predictor using the known drift function \( F \), providing a lower bound on the achievable MSE.
\end{itemize}

\begin{figure}[htbp]
  \centering
    \includegraphics[width=0.99\linewidth]{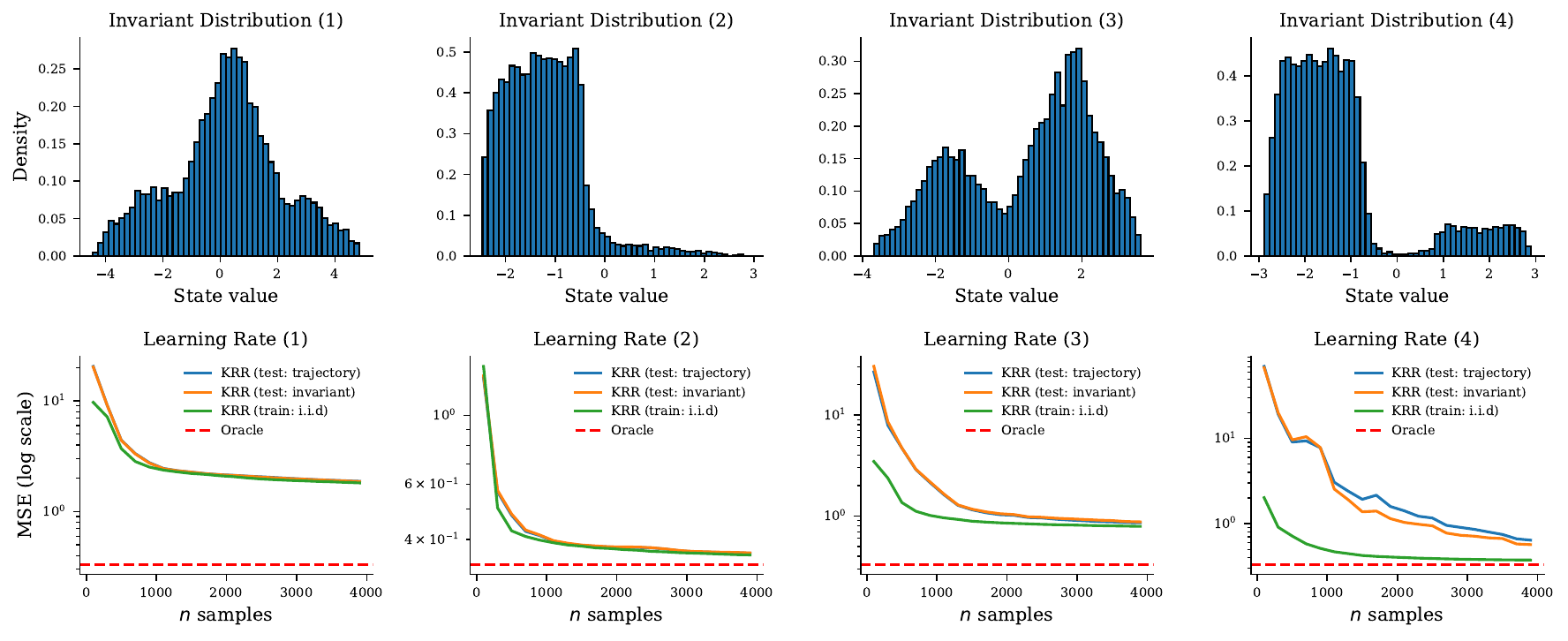}
    \caption{Dependency of the mean squared error (MSE) on the number of training samples from a single trajectory. The \emph{Learning Rate} plot shows three variants of KRR estimators trained on trajectory data and evaluated either on a separate trajectory or on i.i.d. samples. As a baseline, we include a KRR estimator trained on i.i.d. samples drawn from the invariant distribution.}
  \label{fig:ar1-learning-rate-comparison}
\end{figure}

\paragraph{Summary of observations.} We highlight several distinct scenarios (see Figure~\ref{fig:ar1-learning-rate-comparison}) that give rise to slightly different training dynamics, depending on the structure of the invariant distribution and the speed of mixing:
\begin{itemize}
    \item \textbf{Unimodal, fast mixing:} When the invariant distribution is unimodal and the chain mixes rapidly, there is virtually no difference between learning from i.i.d.\ samples and from a single trajectory. Both settings yield similar estimation errors across all sample sizes.
    \item \textbf{Bimodal, fast mixing:} When the invariant distribution has two modes but the mixing remains fast, the performance gap becomes noticeable. Estimators trained on a single trajectory tend to lag behind those trained on i.i.d.\ samples, though they eventually reach the same asymptotic error level. The trajectory-based learning still captures the overall structure but exhibits slower convergence.
    \item \textbf{Bimodal, slow mixing:} In the most challenging regime, the invariant distribution has two modes with one being dominant and the chain exhibiting slow transitions between modes. In this case, the training dynamics from a single trajectory are no longer smooth: the estimator may remain biased toward the dominant mode for a long time before it captures the full distributional structure. This leads to unstable error behavior and significantly delayed convergence.
\end{itemize}
\begin{remark}
To amplify these effects and construct more complex learning scenarios, we extended the experiment to vector-valued regression problems where each coordinate is governed by a different bimodal invariant distribution. This setup allows us to build high-dimensional stationary distributions with varying modal structure and mixing rates along each axis, resulting in rich and heterogeneous training dynamics.
\end{remark}

\section{Conclusions}
In this paper, we studied learning from a single finite trajectory of an ergodic stochastic dynamical system.
Using {non-linear} least squares estimators, we derived {one-step prediction} guarantees under uniform geometric ergodicity and showed how the analysis extends to vector-valued, higher-order, and finite-state systems, as well as to learning Koopman operators.
A main goal was to isolate the  statistical issues arising in the dynamical-systems setting while keeping the assumptions, proofs, and exposition as simple as possible.

In doing so, we  left a number of developments for future work.
Several of them require direct extensions of ideas from learning with i.i.d. data, such as faster rates under source conditions or estimates in different norms \citep{blanchard2018optimal,fischer2020sobolev,zhang2025optimal,li2024towards}.
Other estimators could also be considered, including spectral regularization methods \citep{bauer2007regularization,logerfo2008spectral}, stochastic gradient methods \citep{rosasco2015learning,lin2017optimal,pillaudvivien2018statistical}, and approaches based on random projections \citep{rudi2015less} or random features \citep{rudi2017generalization}.
A more substantial development would be to consider neural-network approaches, where the theory of linearly parametrized least squares estimators cannot be applied directly; convex neural networks may provide a useful intermediate setting \citep{bach2017breaking}. Another possible direction would be to extend analyses of benign overfitting to the dynamical-systems setting \citep{bartlett2020benign,hastie2022surprises,richards2021asymptotics} and investigate whether scaling laws analogous to supervised learning \citep{paquette2024phases} emerge.

Finally, more substantial developments would concern the nature of the systems considered, including for example non-uniform geometric ergodicity and continuous-time systems, or even non-autonomous and partially observable systems.

\section*{Acknowledgments}
S. V. has
been supported by the MUR Excellence Department Project awarded to Dipartimento di Matematica, Universita di Genova, CUP D33C23001110001 and the PNRR project “Harmonic Analysis and Optimization in Infinite-Dimensional Statistical - Future Artificial Intelligence  Fair – Spoke 10” (CUP J33C24000410007). S. V. is a member of the Gruppo Nazionale per l’Analisi Matematica, la Probabilità e le loro Applicazioni (GNAMPA) of the Istituto Nazionale di Alta Matematica (INdAM). O.\ K.\ and L.\ R.\ acknowledge the financial support of the European Commission (Horizon Europe grant ELIAS 101120237). L.\ R.\ also acknowledges the financial support of the Ministry of Education, University and Research (FARE grant ML4IP R205T7J2KP).

\bibliography{references}
\bibliographystyle{apalike}

\appendix

\section{Useful facts about  Markov processes}\label{app:SP}
\label{sec:backMP}
In this appendix, we collect the basic definitions and facts about Markov processes used throughout the paper. We further establish several auxiliary results, with self-contained proofs, that support the subsequent learning-theoretic analysis.

\subsection{Conditional expectations and conditional probabilities}

Let \((\Omega, \cA, \bbP)\) be a probability space, \((\cX,\cB(\cX))\) a measurable space, and let \(\cP(\cX)\) be the space of probability measures on \((\cX,\cB(\cX))\).  A \emph{random variable} is a map \(Y:(\Omega, \cA) \to(\cX,\cB(\cX))\) such that \(Y^{-1}(B)\in\cA\), for all \(B\in\cB(\cX)\). The law of \(Y\) is the probability measure \(\mu=Y_\#\bbP\in\cP(\cX)\) defined as  \(Y_\# \bbP(B)=\bbP(Y\in B)=\bbP (Y^{-1}(B))\), for each \(B\in\cB(\cX)\).

Let \((\cH,\scal{\cdot}{\cdot})\) be a separable Hilbert space, and denote by  \(\cB(\cH)\) its Borel \(\sigma\)-algebra. A random variable \(Y:(\Omega, \cA) \to(\cH,\cB(\cH))\) is \emph{Bochner integrable} if \(\int \|Y\|\,d\bbP<\infty\). In this case, its \emph{expectation} is \(\bbE[Y]=\int Y\,d\bbP=\int y d\mu\), where the integral is understood in the Bochner sense. For \(\cH=\bbR\), this coincides with the usual Lebesgue integral and standard integrability.

Let \(Y:(\Omega, \cA) \to(\cH,\cB(\cH))\) be Bochner integrable and  \(\cG\subseteq\cA\) be a sub-\(\sigma\)-algebra. A \emph{conditional expectation} of \(Y\) given \(\cG\) is any Bochner integrable random variable \(\bbE[Y\mid\cG]:(\Omega,\cG)\to(\cH,\cB(\cH))\) such that
\[
\int_G \bbE[Y\mid\cG]\,d\bbP
= \int_G Y\,d\bbP,
\quad \forall G\in\cG.
\]
Any two conditional expectations of \(Y\) given \(\cG\) are equal \(\bbP\)-almost surely.  

For each \(A\in\cA\), any \emph{conditional probability} of \(A\) given \(\cG\) is a random variable
\(
\bbP(A\mid \cG):(\Omega,\cG)\to[0,1]
\)
defined by
\[
\bbP(A\mid \cG)=\bbE[\mathds{1}_A\mid \cG].
\]
Let \((\Omega, \cA, \bbP)\) be a probability space, and  \((\cX,\cB(\cX)), (\cY,\cB(\cY))\) measurable spaces. 
Let \(Y:(\Omega,\cA)\to(\cY,\cB(\cY))\) and \(X:(\Omega,\cA)\to(\cX,\cB(\cX))\) be random variables.  For each \(B\in \cB(\cY)\),  define
\[
\bbP(Y\in B \mid X )
=
\bbP(Y^{-1}(B) \mid \sigma(X)),
\]
where \(\sigma(X)=\{X^{-1}(C):C\in\cB(\cX)\}\) is the \(\sigma\)-algebra generated by \(X\). A \emph{regular conditional distribution} of \(Y\) given \(X\) is a map \(Q:\cX\times\cB(\cY)\to[0,1]\) such that \(Q(x,\cdot)\in\cP(\cY)\) for all \(x\in\cX\), \(Q(\cdot,B)\) is measurable for all \(B\in\cB(\cY)\), and
\[
\bbP(Y\in B\mid X)(\omega)=Q(X(\omega),B)
\quad \bbP\text{-a.s.},
\qquad \forall B\in\cB(\cY).
\]
We then write
\(
\bbP(Y\in B\mid X=x)=Q(x,B),
\)
and regard \(Q(x,\cdot)\) as a version of the conditional probability of \(Y\) given \(X=x\).

Regular conditional distributions exist when the  measurable spaces  \((\cX,\cB(\cX)), (\cY,\cB(\cY))\)  are standard Borel spaces. A measurable space is called a \emph{standard Borel space} if it is measurably isomorphic to the Borel space of a complete separable metric space. For example, \(\bbR^d\) with its Borel \(\sigma\)-algebra is a standard Borel space.

\begin{remark}[Polish spaces are standard Borel spaces]
A \emph{Polish space} is a topological space whose topology is induced by a complete separable metric, equivalently a separable completely metrizable topological space. Thus every Polish space \(E\), equipped with its Borel \(\sigma\)-algebra \(\cB(E)\), is a standard Borel space. Conversely, every standard Borel space can be realized as the Borel space of some Polish topology. Hence standard Borel spaces retain only the measurable structure associated with Polish spaces, not a distinguished topology.
\end{remark}

\subsection{Stochastic processes and filtrations}

A \emph{stochastic process} \((X_t)_{t \in \bbN_0}\) is a family of random variables
\[
X_t : (\Omega,\cF) \to (\cX,\cB(\cX)), \quad  t \in \bbN_0.
\]
For each \(t \in \bbN_0\), the law of \(X_t\) is
\[
\mu_t(A) = (X_t)_{\#}\bbP(A)=\bbP(X_t \in A),
\quad A \in \cB(\cX),
\]
which is the marginal distribution of the process at time \(t\).

A \emph{filtration} \((\cF_t)_{t \in \bbN_0}\) is a family of sub-\(\sigma\)-algebras of \(\cF\) such that \(\cF_s \subseteq \cF_t\) for all \(s \leq t\). A process \((X_t)_{t \in \bbN_0}\) is \emph{adapted} to the filtration \((\cF_t)_{t \in \bbN_0}\) if, for every \(t \in \bbN_0\), the random variable \(X_t\) is \(\cF_t/\cB(\cX)\)-measurable. The natural filtration associated with \( (X_t)_{t \in \bbN_0} \) is \(\cF_t = \sigma(X_s : s \leq t)\), i.e., the smallest \(\sigma\)-algebra with respect to which the random variables \(X_0, \ldots, X_t\) are measurable.

\subsection{Markov processes}

Let \((\Omega, \cF, \bbP)\) be a probability space, \((\cX,\cB(\cX))\) a measurable space, and let \((X_t)_{t\in\bbN_0}\) be a stochastic process with values in \((\cX,\cB(\cX))\). By construction, the process is adapted to its  natural filtration \((\cF_t)_{t \in \bbN_0}\).

The process \((X_t)_{t\in\bbN_0}\) is called a \emph{Markov process} if it satisfies the \emph{Markov property}: that is  for all \(t\in\bbN\) and all \(A\in\cB(\cX)\),
\be\label{MP}
\bbP(X_t \in A \mid \cF_{t-1}) 
=\bbP(X_t\in A\mid X_{t-1}).
\ee
As explained before, we might use the notation 
\(\bbP(X_t\in A\mid X_{t-1},\dots,X_0)=\bbP(X_t \in A \mid \cF_{t-1})\). If the conditional probabilities in \eqref{MP} do not depend on \(t\), the process is called \emph{time-homogeneous}. Markov processes are tightly connected to the notion of transition kernel.

A \emph{transition kernel} on \((\cX,\cB(\cX))\) is a map
\[
P: \cX \times \cB(\cX) \to [0,1]
\]
such that, for every \(x\in\cX\), the map \(A \mapsto P(x,A)\) is a probability measure on \((\cX,\cB(\cX))\), and, for every \(A \in \cB(\cX)\), the map \(x \mapsto P(x,A)\) is measurable~\citep[Chapter~1]{douc2018markov}.

If \((\cX,\cB(\cX))\) is a standard Borel space, every time-homogeneous Markov process admits a transition kernel \(P\) such that
\(
\bbP(X_t\in A\mid X_{t-1})=P(X_{t-1},A),
\quad t\in\bbN,\ A\in\cB(\cX).
\)
Equivalently, one writes
\[
P(x,A)=\bbP(X_t\in A\mid X_{t-1}=x),
\quad x\in\cX,\ A\in\cB(\cX),
\]
where the right-hand side is understood as a regular conditional distribution. Conversely, given an initial distribution \(\mu_0\in\cP(\cX)\) and a transition kernel \(P\), the Ionescu--Tulcea extension theorem~\citep{kallenberg2002foundations} yields a unique probability measure on the canonical path space \(\cX^{\bbN_0}\), equipped with its product \(\sigma\)-algebra, under which the coordinate process is a time-homogeneous Markov process \((X_t)_{t\in\bbN_0}\) with initial distribution \(\mu_0\) and transition kernel \(P\), i.e.,
\[
\bbP(X_0\in A)=\mu_0(A),
\quad \bbP(X_t\in A\mid X_{t-1})=P(X_{t-1},A),
\]
for all \(t\in\bbN\) and \(A\in\cB(\cX)\). Thus the pair \((\mu_0,P)\) determines the law of the process \emph{uniquely}, although the process may admit many realizations on different probability spaces. Therefore, a time-homogeneous Markov process is often identified with the pair \((P,\mu_0)\) of  transition kernel and initial distribution. See, for example,~\citet{eberle2009markov,douc2018markov}.

\subsection{Stochastic dynamical systems}

Next,  we provide an alternative construction of the stochastic dynamical system~\eqref{eq:ds-def}, which makes it easier to investigate the connection to Markov processes and transition kernels. 

Let \((\Omega,\cF,\bbP)\) be a probability space, let \((\cX,\cB(\cX))\) and \((\cN,\cB(\cN))\) be measurable spaces, and let
\[
f:\cX\times\cN\to\cX
\]
be \(\cB(\cX)\otimes\cB(\cN)\)-measurable. Let \(X_0:(\Omega,\cF)\to(\cX,\cB(\cX))\) be a random variable, and let \(N_t:(\Omega,\cF)\to(\cN,\cB(\cN))\), \(t\in\bbN_0\), be identically distributed random variables such that \(X_0,N_0,N_1,\dots\) are mutually independent.

A \emph{stochastic dynamical system} can be  defined recursively by
\be\label{RV_SDS}
X_{t+1}=f(X_t,N_t),
\quad t\in\bbN_0.
\ee
Stochastic dynamical systems defined as above are also known as iterated random functions~\citep{diaconis1999iterated}, random
recursions~\citep{damek2017stochastic}, or stochastic recurrence
equations~\citep{vervaat1979stochastic}.  In engineering applications, a stochastic dynamical system is often written as in equation~\eqref{eq:ds-def}, namely as
\be\label{PW_SDS}
x_{t+1}=f(x_t,\eta_t),
\quad t\in\bbN_0,
\ee
where \(x_0\) is sampled according to a probability measure \(\mu_0\) on \((\cX,\cB(\cX))\), and \((\eta_t)_{t\in\bbN_0}\) are independent samples from a probability measure \(\nu\) on \((\cN,\cB(\cN))\). Here  \(x_0\) is also assumed to be  independent of \((\eta_t)_{t\in\bbN_0}\).

We don't provide an in-depth discussion of two formulations we consider for stochastic dynamical systems, that is Equation~\eqref{RV_SDS} and Equation~\eqref{PW_SDS},  but add one comment about their relation.
The  recursion~\eqref{PW_SDS} describes sample paths of the stochastic dynamical system~\eqref{RV_SDS}. Indeed, let \(\mu_0\) be the law of \(X_0\), and let \(\nu\) be the common law of the random variables \(N_t\), \(t\in\bbN_0\). If \(X_0,N_0,N_1,\dots\) are mutually independent, then {\(\bbP\)-almost every \(\omega\in\Omega\)} determines a sequence
\[
x_t=X_t(\omega),
\qquad
\eta_t=N_t(\omega),
\]
satisfying~\eqref{PW_SDS}.
Thus, \eqref{PW_SDS} is the sample-path notation, while
~\eqref{RV_SDS} is the corresponding random-variable formulation. As shown next, the latter yields an immediate connection to Markov processes.
We first add a remark regarding the existence of probability spaces supporting the random variables defining a stochastic dynamical system~\eqref{RV_SDS}.

\begin{remark}[Autoregressive models]
\label{rem:ar1}
A classical special case of~\eqref{RV_SDS} is obtained when \(\cX=\cN=\bbR^d\) and the noise enters additively. More precisely, let \(F:\bbR^d\to\bbR^d\) be measurable and set
\[
f(x,\eta)=F(x)+\eta,
\qquad x, \eta \in \bbR^d.
\]
Then the recursion~\eqref{RV_SDS} becomes
\[
X_{t+1}=F(X_t)+N_t,
\qquad t \in \bbN_0,
\]
which is a nonlinear autoregressive model of order one~\citep{Brockwell2016-re}. Note that this is the same dynamical system as in Example~\ref{ex:noisy-evolution}, now viewed from a stochastic-process perspective.
\end{remark}

\begin{remark}[Existence of the probability space]
The  formulation~\eqref{RV_SDS} assumes the existence of a probability space carrying the random variable describing the initial condition and a countable family of identically distributed, mutually independent noise random variables. Here, we recall that this assumption is not restrictive. Indeed, let
\(
\xi_0:(\Omega_X,\cF_X,\bbP_X)\to(\cX,\cB(\cX))
\)
be a random variable describing the initial condition, and let
\(
\xi_\eta:(\Omega_N,\cF_N,\bbP_N)\to(\cN,\cB(\cN))
\)
be a random variable describing the noise input. By the countable product construction, see, e.g., \citet[Chapter~36]{billingsley1995probability}, one may
consider the product probability space
\(
(\Omega,\cF,\bbP)
= \Big(
\Omega_X\times \Omega_N^{\bbN_0},
\cF_X \otimes \cF_N^{\otimes \bbN_0},
\bbP_X \otimes \bbP_N^{\otimes \bbN_0}
\Big).
\)
Writing
\[
\omega=(\omega_0,\omega_1,\omega_2,\dots),
\qquad \omega_0 \in \Omega_X, \quad \omega_j \in \Omega_N \quad \text{for } j \ge 1,
\]
define
\[
X_0(\omega)=\xi_0(\omega_0),
\qquad N_t(\omega)=\xi_\eta(\omega_{t+1}),
\quad t\in\bbN_0.
\]
Then \(X_0,N_0,N_1,\dots\) are mutually independent, \(X_0\) has law
\((\xi_0)_\#\bbP_X\), and each \(N_t\) has law \({\xi_\eta}_\#\bbP_N\). Hence the
probability space used to describe the stochastic dynamical system~\eqref{RV_SDS} above can always be realized as a product space carrying the initial condition together with countably many independent copies of the
noise.
\end{remark}

\subsection{Stochastic dynamical systems and Markov processes}

In this section we derive a basic result relating stochastic dynamical systems and Markov processes.

\begin{proposition}\label{prop:SDS_MP}
A stochastic dynamical system~\eqref{RV_SDS} defines a {time-homogeneous} Markov process
\((X_t)_{t\in\bbN_0}\) and the associated transition kernel \(P:\cX\times\cB(\cX)\to[0,1]\) is given by 
{
\be\label{SDS_TK}
P(x,B)=\int \mathbf 1_B(f(x,\eta))\,\nu(d \eta),
\qquad x\in\cX,\ B\in\cB(\cX),
\ee
}
where \(\nu\) is the common law of the random variables \(N_t\). 
\end{proposition}
The above facts are quite standard; see, e.g., \cite{diaconis1999iterated}.
We include the proof for completeness.

\begin{proof}
Fix \(B\in\cB(\cX)\). Let \(h_B:\cX\to[0,1]\) be defined by
{
\be\label{hB}
h_B(x)=\int \mathbf 1_B(f(x,\eta))\,\nu(d \eta),
\qquad x\in\cX .
\ee
}
Since \(f\) is measurable, the map
\(
(x,\eta) \mapsto \mathbf{1}_B(f(x,\eta))
\)
is nonnegative and \(\cB(\cX)\otimes\cB(\cN)\)-measurable. Hence, by Tonelli's theorem~ \citep{billingsley1995probability,eisner2015operator},
the map \( h_B \) is \(\cB(\cX)\)-measurable and takes values in \([0,1]\).

Then, fix \(t\in\bbN_0\). We claim that
\be\label{eq:markov-hb}
\bbP(X_{t+1}\in B\mid \cF_t)=h_B(X_t)
\qquad \bbP\text{-a.s.}
\ee
Let
\[
\cG_t=
\begin{cases}
\sigma(X_0), & t=0,\\
\sigma(X_0,N_0,\dots,N_{t-1}), & t\ge 1.
\end{cases}
\]
By Equation~\eqref{RV_SDS}, \(X_s\) is \(\cG_t\)-measurable for every \(s\le t\). Hence
\[
\cF_t=\sigma(X_0,\dots,X_t)\subseteq \cG_t .
\]
By mutual independence of \(X_0,N_0,N_1,\dots\), \(N_t\) is independent of
\(\cG_t\), and therefore of \(\cF_t\).
Since \(X_t\) is \(\cF_t\)-measurable and \(N_t\) is independent of \(\cF_t\), for every \(A\in\cF_t\),
\[
\int_A \mathbf 1_B(f(X_t(\omega),N_t(\omega)))\, \bbP(d \omega)
= \int_A \left(\int \mathbf 1_B(f(X_t(\omega),\eta))\,\nu(d \eta)\right) \bbP(d \omega)
= \int_A h_B(X_t(\omega))\, \bbP(d \omega).
\]
Since \(h_B(X_t)\) is \(\cF_t\)-measurable, this means
\[
\bbE[\mathbf 1_B(f(X_t,N_t))\mid \cF_t]
= h_B(X_t), \qquad \bbP\text{-a.s.}.
\]
Note that
\[
\bbP(X_{t+1}\in B\mid \cF_t)
= \bbE[\mathbf 1_B(X_{t+1})\mid \cF_t]
= \bbE[\mathbf 1_B(f(X_t,N_t))\mid \cF_t]
= h_B(X_t), \qquad \bbP\text{-a.s.},
\]
which proves~\eqref{eq:markov-hb}.

Since \(\sigma(X_t)\subseteq\cF_t\), the tower property and \eqref{eq:markov-hb} give
\[
\bbP(X_{t+1}\in B\mid X_t)
= \bbE\!\left[
    \bbE[\mathbf{1}_B(X_{t+1})\mid\cF_t]
    \,\middle|\,X_t
\right] \\
= \bbE[h_B(X_t)\mid X_t]
= h_B(X_t)
\qquad \bbP\text{-a.s.}
\]
Consequently,
\[
\bbP(X_{t+1}\in B\mid\cF_t)
=
\bbP(X_{t+1}\in B\mid X_t)
=
P(X_t,B),
\]
which proves the Markov property and time homogeneity.

Finally, the fact that~\eqref{SDS_TK} is the associated transition kernel follows from
the definition of \(h_B\). Indeed, for every \(x\in\cX\) and every
\(B\in\cB(\cX)\),
\[
P(x,B)=h_B(x).
\]
It remains only to check that \(P\) is a transition kernel. For every
\(B\in\cB(\cX)\), \(P(\cdot,B)=h_B\) is measurable. Moreover, for every fixed
\(x\in\cX\), \(P(x,\cdot)\) is the pushforward of \(\nu\) through the
measurable map \(f(x,\cdot)\). Hence \(P(x,\cdot)\) is a probability measure on
\((\cX,\cB(\cX))\). This proves the claim that \(P\) is the transition kernel associated with \((X_t)_{t\in\bbN_0}\).
\end{proof}
We end by noting that, under standard Borel regularity assumptions, converse results are also known; see, e.g., \cite{meyn2012markov}.

\subsection{Markov processes under measurable transformations}

In this subsection, we  show that an injective measurable transformation of a uniformly geometrically ergodic Markov process preserves its Markov structure, invariant measure, and the convergence rate. This observation will be crucial for deriving the learning guarantees in the finite-state setting and for Koopman operator learning.

\paragraph{Setup.} Let \((\cX,\cB(\cX))\) and \((\cY,\cB(\cY))\) be Borel measurable spaces, let \(T:\cX \to \cY\) be  measurable, and let \(\mu \in \cP(\cX)\). Recall that the \emph{push-forward} of \(\mu\) through \(T\), denoted by \(T_{\#}\mu\), is the probability measure on \((\cY,\cB(\cY))\) defined by
\begin{equation}
\label{eq:push-forward-def}
T_{\#}\mu(A)
= \mu\bigl(T^{-1}(A)\bigr),
\qquad A \in \cB(\cY).
\end{equation}
Equivalently, for every bounded measurable function \(g: \cY \to \bbR\),
\[
\int g(y)\,T_{\#}\mu(dy)
= \int g\bigl(T(x)\bigr)\,\mu(dx).
\]

\begin{lemma}[Push-forward of invariant measure and convergence rate]
\label{lem:emb-markov-conv}
Let \(X=(X_t)_{t \in \bbN_0}\) be a Markov process taking values in \((\cX,\cB(\cX))\), with transition kernel \(P\) and invariant measure \(\pi \in \cP(\cX)\). Suppose that 
\begin{equation}
\label{eq:pushforward-geometric-rate} \sup_{x\in\cX} \nor{P^t(x,\cdot)-\pi}_{\mathrm{TV}}
\leq C\rho^t,
\qquad t\in\bbN_0,
\end{equation}
for some \(C>0\) and \(\rho \in (0,1)\).
Let \(T:\cX \to \cY\) be injective and measurable. Define
\[
Y_t=T(X_t),
\qquad t \in \bbN_0,
\]
and set \(\widetilde\pi=T_{\#}\pi\). Then, \((Y_t)_{t \in \bbN_0}\) is a Markov process with associated transition kernel \(Q\), \(\widetilde\pi\) is its unique invariant measure, and
\[
\sup_{x\in\cX} \left\| Q^t\bigl(T(x),\cdot\bigr)-\widetilde\pi \right\|_{\mathrm{TV}} 
\leq C\rho^t.
\]
\end{lemma}
\begin{proof}
Define \(\widetilde \cY:=T(\cX)\), equipped with the trace \(\sigma\)-algebra
\[
\cB(\widetilde \cY)
:= \left\{A \cap \widetilde \cY: A \in \cB(\cY)\right\};
\]
see, e.g.,~\citet[Section~3]{schilling2017measures}. Since \(T\) is injective, the inverse map \(T^{-1}: \widetilde \cY \to \cX\) is well defined. Moreover, by Lusin--Suslin theorem for standard Borel spaces \( \cX \) and \( \cY \), \( T^{-1} \) is measurable; see, e.g.,~\citet[Theorem~15.1]{kechris2012classical}. Define \(Q:\widetilde \cY \times \cB(\widetilde \cY) \to [0,1]\) by 
\begin{equation}
\label{eq:markov-factor-condition} 
Q\bigl(T(x),A\bigr)
= P\bigl(x,T^{-1}(A)\bigr),
\qquad x\in\cX,\; A \in \cB(\widetilde \cY). 
\end{equation}
This definition is unambiguous because \(T\) is injective. Moreover, for every \(y=T(x) \in \widetilde \cY\), the map \(A \mapsto Q(y,A)\) is a probability measure, while, for every \(A \in \cB(\widetilde \cY)\), the map 
\[
y \mapsto Q(y,A)
= P\bigl(T^{-1}(y),T^{-1}(A)\bigr)
\]
is measurable. Thus, \(Q\) is a proper transition kernel on \(\widetilde \cY\); see Section~\ref{sec:MP-dynamical} for details. It remains to verify that \(Q\) is the transition kernel associated with \((Y_t)_{t\in\bbN_0}\).

Let \(\cG_t=\sigma(Y_0,\ldots,Y_t)\) be a natural filtration of the process \( (Y_t)_{t \in \bbN_0} \). Since \(\cG_t \subseteq \sigma(X_0,\ldots,X_t)\), for every \(A \in \cB(\cY)\), the Markov property of \((X_t)_{t \in \bbN_0}\) and \eqref{eq:markov-factor-condition} give
\begin{align*}
\bbP\bigl(Y_{t+1}\in A\mid\cG_t\bigr)
&= \bbE\left[ \bbP\bigl(T(X_{t+1})\in A\mid X_t\bigr) \,\middle|\,\cG_t \right]
= \bbE\left[ P\bigl(X_t,T^{-1}(A)\bigr) \,\middle|\,\cG_t \right] \\
&= \bbE\left[ Q\bigl(T(X_t),A\bigr) \,\middle|\,\cG_t \right]
= Q(Y_t,A). 
\end{align*}
Therefore, \((Y_t)_{t \in \bbN_0}\) is a Markov process with transition kernel \(Q\).  Iterating~\eqref{eq:markov-factor-condition} and applying~\eqref{eq:push-forward-def} yields
\[
Q^t\bigl(T(x),A\bigr)
= P^t\bigl(x,T^{-1}(A)\bigr)
= T_{\#}P^t(x, A).
\]
Finally, let \(g:\cY \to \bbR\) be bounded measurable. Writing
\[
(Qg)(y)=\int g(y')\,Q(y,dy'),
\]
we obtain
\begin{align*}
\int Qg(y)\,\widetilde\pi(dy)
= \int Qg\bigl(T(x)\bigr)\,\pi(dx)
= \int P(g\circ T)(x)\,\pi(dx)
= \int g\bigl(T(x)\bigr)\,\pi(dx)
= \int g(y)\,\widetilde\pi(dy),
\end{align*}
where the third equality follows from invariance of \(\pi\). Thus, \(\widetilde\pi\) is invariant for \(Q\).

Fix \(x\in\cX\) and \(t\in\bbN_0\). By the definition of \( Q \) and \( \widetilde\pi \), 
\[
\nor{Q^t(x, \cdot) - \widetilde\pi}_\mathrm{TV}
= \nor{T_{\#}P^t(x,\cdot)-T_{\#}\pi}_\mathrm{TV}
= \nor{ P^t(x,\cdot)-\pi }_\mathrm{TV},
\]
by the data-processing inequality for the total variation distance under measurable
transformations; see, e.g.,~\citet[Section~7.2]{polyanskiy2025information}. Since \(T\) is a measurable isomorphism onto its image \( \widetilde \cY \), equality holds. Taking the supremum over \(x\in\cX\) and applying \eqref{eq:pushforward-geometric-rate} gives the claimed geometric rate.
\end{proof}

\begin{remark}[Assumptions on $T$]
An additional assumption beyond measurability on the map \(T:\cX \to \cY\) cannot, in general, be omitted, since a measurable transformation of a Markov process need not itself be Markov. As is apparent from the proof, it is sufficient that there exists a Markov kernel \(Q\) on \( T(\cX) \) satisfying~\eqref{eq:markov-factor-condition}. We state the lemma under the stronger assumption that \(T\) is measurable and injective, as this covers all applications we considered in this work.
\end{remark}

\subsection{Example: bounded drift with Gaussian noise}
\label{sec:MPExamples}

In this section, we provide the derivation of the bounded-drift Gaussian-noise Example~\ref{the_example} introduced in Section~\ref{sec:ergodic-measure}.

\begin{proposition}
Let \( \cX=\bbR^d \) and define
\[
X_{t+1} = b(X_t) + \sigma \xi_t,
\qquad \text{with } \xi_t \sim \cN(0,1), \, \sigma > 0,
\]
where \( \cN(0, I_d) \) is a standard Gaussian and \( b: \cX \to \cX \) is bounded. Denote \( B := \norm[\infty]{b} \) and \( a := B/\sigma \). Then, the chain is uniformly ergodic and Assumption~\ref{UGE} holds with \( C = 1 \). Moreover,
\[
\frac{1}{1 - \rho} \approx \begin{cases}
\sqrt{2 \pi} a \exp\!\left(a^2/2\right), & d = 1 \\[0.3em]
\exp\!\left( a\sqrt{d} + a^2/2 \right), & d \gg 1
\end{cases}
\]
\end{proposition}
\begin{proof}
We begin by constructing transition kernel defined by the process \( (X_t)_{t \ge 0} \).
\[
P(x, A) = \int_A \underbrace{\frac{1}{(2 \pi \sigma^2)^{d/2}} \exp \left( - \frac{\norm[\cX]{x' - b(x)}^2}{2 \sigma^2}\right)}_{p(x, x')} dx',
\qquad x \in \cX, A \in \cB(\cX).
\]
Note that \( \operatorname{supp} P(x, \cdot) = \cX \) for every \( x \in \cX \). Therefore, Markov chain \( (X_t)_{t \in \bbN_0} \) is \( \psi \)-irreducible with respect to Lebesgue measure and aperiodic. By boundedness,
\begin{equation}
\norm[\cX]{x' - b(x)} \leq \norm[\cX]{x'} + B,
\qquad \forall x, x' \in \cX.
\label{eq:tilde-p-bound}
\end{equation}
Define \emph{transition density} \( \tilde p: \cX \to [0,1] \) and measure \( \nu \), so that for any \( x \in \cX \) and \( A \in \cB(\cX) \),
\[
\tilde p(x) = \frac{1}{(2 \pi \sigma^2)^{d/2}} \exp \left(- \frac{\left( \norm[\cX]{x} + B \right)^2}{2 \sigma^2}\right)
\quad \text{and} \quad
\nu(A) := \int_A \tilde p(x) dx.
\]
For every \( x \in \cX \), \( \tilde p(x) > 0 \), so that \( \nu \) is a non-trivial (non-zero) measure. Moreover, by~\eqref{eq:tilde-p-bound}, for every \( x \in \cX \) it holds \( \tilde p \leq p(x, \cdot) \). So that
\[
\forall x \in \cX, \, A \in \cB(\cX):
\quad P(x, A) \ge \nu(A).
\]
Therefore, by~\citet[Theorem~16.2.2]{meyn2012markov} there exists unique \( \pi \in \cP(\cX) \) such that for every \( x \in \cX \) and every \( n \ge 1 \),
\[
\norm[\mathrm{TV}]{ P^n(x, \cdot) - \pi } \leq \rho^n,
\]
where \( \rho := 1 - \nu(\cX) \).

We proceed by approximating \( \nu(\cX) \). Since \( \tilde p(x) \) depends only on the magnitude \( r = \norm[\cX]{x} \), we can use spherical coordinates:
\[
\nu(\cX) = \int_0^{+\infty} \frac{1}{(2 \pi \sigma^2)^{d/2}} \exp\!\left(- \frac{(r + B)^2}{2 \sigma^2}\right) \mathcal S_{d-1}(r) \, dr,
\]
where \( \mathcal S_{d-1}(r) = 2 \pi^{d/2} \Gamma(d/2)^{-1} r^{d-1}\) is a surface area of the sphere.
For \( d = 1 \), with a change of variable \( z := (x + B) / \sigma \) and \( dx = \sigma dz \), we get
\[
\nu(\cX)
= 2 \int_{a}^{+\infty} \frac{1}{\sqrt{2 \pi}} \exp \left(- \frac{z^2}{2}\right) dz
= 2(1 - \Phi(a)),
\]
where \( \Phi \) is standard normal CDF.
\begin{remark}
Note that for \( B = 0 \), i.e., no drift, yields \( \rho = 0 \) as expected.
\end{remark}
A simple approximation of \( \Phi \) e.g. by \citet{feller1991introduction,choudhury2007approximating} yields,
\[
1 - \rho = \nu(\cX) = 2(1 - \Phi(a))
\le \frac{\sqrt{2 \pi}}{a} \exp\left(-a^2/2\right).
\]

For large \( d \gg 1 \), we use the following observation. Let \( Z \sim \cN(0, I_d) \) and \( R := \norm[\cX]{Z} \). Then \( R \) has density
\[
p_R(x) = \frac{2}{2^{d/2} \Gamma(d/2)} r^{d-1} \exp(-r^2/2),
\]
so that with the change of variables \( r = \sigma t \) and \( dr = \sigma d t \),
\[
\nu(\cX)
= \exp\!\left(- \frac{a^2}{2}\right) \int_0^{+\infty} \exp\!\left(-a t\right) p_R(t)dt
= \exp\!\left(- \frac{a^2}{2}\right) \bbE\exp\left( - a R \right).
\]
For large \( d \), \( R \) is concentrated around \( R \approx \sqrt{d} \). Therefore,
\[
\nu(\cX) \approx \exp\!\left( - \frac{a^2}{2} - a\sqrt{d}\right).
\]
The claim is retrieved by inverting both bounds.
\end{proof}

\section{Operator theory results}\label{app_operators}
We start collecting some basic definitions and results.
\subsection{Basic definitions and notation}

Let $(\cH,\scal{\cdot}{\cdot}_\cH)$ and $(\cG,\scal{\cdot}{\cdot}_\cG)$ be real separable Hilbert spaces. 
An operator $A:\cH\to\cG$ is bounded if $\|A\|=\sup_{\|f\|_\cH\le 1}\|Af\|_\cG<\infty$. 
The adjoint operator $A^*:\cG\to\cH$ is defined by 
$\scal{Af}{g}_\cG=\scal{f}{A^*g}_\cH$ for all $f\in\cH$, $g\in\cG$. 
An operator $A:\cH\to\cH$ is called self-adjoint if $A=A^*$. It is positive if \(
\scal{Af}{f}_\cH \ge 0\),  \(\forall f\in\cH\). For $A:\cH\to\cG$ bounded, let $|A|=(A^*A)^{1/2}$.  Then there exists a partial isometry $U:\cH\to\cG$ such that $A=U|A|$, with  $\ker U=\ker A$, $U^*U$ the orthogonal projection onto 
$\overline{\operatorname{ran}(|A|)}\subseteq\cH$, and $UU^*$ the orthogonal projection onto 
$\overline{\operatorname{ran}(A)}\subseteq\cG$.
The trace of a positive operator 
$A$ is $\tr{A}=\sum_i \scal{Ae_i}{e_i}_\cH$, with $(e_i)$  any orthonormal basis of $\cH$. 
For $g\in\cG$ and $f\in\cH$, define the rank-one operator 
$g\otimes f:\cH\to\cG$ by
\(
(g\otimes f)h=\scal{h}{f}_\cH\,\)\(g\), \(~
 h\in\cH\).
If $f\in\cH$, then $f\otimes f:\cH\to\cH$ is a positive operator. 
If $\|f\|_\cH=1$, then $f\otimes f$ is an orthogonal projection. If $A\in\mathcal{L}(\mathcal{H})$, then $\sigma(A)$ denotes its spectrum. Recall that if $A$ is a bounded, positive and self-adjoint operator, then $\sigma(A)\subset[0,+\infty)$ and $\|A\|=\sup_{\sigma\in\sigma(A)} \sigma$.

An operator $A:\cH\to\cG$ is called Hilbert--Schmidt if $\tr{A^*A}<\infty$.
An operator $A:\cH\to\cH$ is trace class if $\tr{|A|}<\infty$.
The space of Hilbert--Schmidt operators is a Hilbert space with inner product 
$\scal{A}{B}_{\cL_2(\cH, \cG)}=\tr{A^*B}$ and is denoted by $\cL_2(\cH,\cG)$. When \(\cH=\cG\), we use the shorthand notation \(\cL_2(\cH)\).
The spaces of bounded and trace class operators are Banach spaces with norms 
$\nor{\cdot}$ and $\nor{\cdot}_{\mathrm{tr}}=\tr{|\cdot|}$, and are denoted by 
$\cL(\cH,\cG)$ and $\cL_1(\cH)$,  respectively.

\subsection{Operator estimates}
The following bound is useful
\begin{lemma}\label{lem:op_est}
Let $T\in\cL(\cH)$ be bounded, positive, and self-adjoint and let $\lambda>0$. 
For $0\le b \le a \le 1$,
\[
\|(T+\lambda I)^{-a}T^{\,b}\|\le \lambda^{\,b-a}.
\]
\end{lemma}
\begin{proof}
Since $T$ is positive and self-adjoint, by the spectral theorem
\[
\|(T+\lambda I)^{-a}T^{\,b}\|
=
\sup_{t\in\sigma(T)}\frac{t^b}{(t+\lambda)^a}
\le
\sup_{t\ge0}\frac{t^b}{(t+\lambda)^a}.
\]

Let $t=\lambda s$,  then
\[
\sup_{t\ge0}\frac{t^b}{(t+\lambda)^a}
\leq
\lambda^{\,b-a}\sup_{s\ge0}\frac{s^b}{(1+s)^a}\le 1
\]
where the last inequality follows noting that 
{$s^b\le (1+s)^b\leq (1+s)^a$} for all $s\ge0$, since $b\le a$.
\end{proof}
In particular we have the following corollary.
\begin{corollary}\label{cor:op_est}
Let $A\in\cL(\cH)$ and $\lambda>0$. For $ a\in [1/2,1]$,
\[
\|(A^*A+\lambda I)^{-a}A^*\|
\le
\lambda^{\frac12-a}.
\]
\end{corollary}

\begin{proof}
Since $A^*=|A|U^*$ for a partial isometry $U$, we have
\[
\|(A^*A+\lambda I)^{-a}A^*\|
\le
\|(A^*A+\lambda I)^{-a}|A|\|.
\]
Applying the lemma with $T=A^*A$ and $b=\tfrac12$ gives the claim. 
\end{proof}
 
 The following application of the Neumann series will also be useful. 

\begin{lemma}\label{lem:op_neu_est}
Let $A,B\in\cL(\cH)$ be positive self-adjoint operators, and for  $\lambda>0$ denote \(A_\lambda=A+\lambda I\) and \(B_\lambda=B+\lambda I\).
If  $\|(B-A)A_\lambda^{-1}\|<1$,  then
\[
\|A^{1/2}B_\lambda^{-1}\|
\le \frac{\|A^{1/2}A_\lambda^{-1}\|}{1-\|(B-A)A_\lambda^{-1}\|}.
\]
In particular, if $\|(B-A)A_\lambda^{-1}\|\le \tfrac12$, then
\[
\|A^{1/2}B_\lambda^{-1}\|\le 2\|A^{1/2}A_\la^{-1}\| \leq 1/\sqrt{\la}. \]
\end{lemma}

\begin{proof}
Note that 
\[
B_\lambda
= A_\lambda + (B-A)
= \bigl(I + (B-A)A_\lambda^{-1}\bigr)A_\lambda.
\]
 Note that $\bigl(I + (B-A)A_\lambda^{-1}\bigr)=B_\la A_\la^{-1}$ and is therefore  invertible.
Hence,
\[
B_\lambda^{-1}
=
A_\lambda^{-1}\bigl(I+(B-A)A_\lambda^{-1}\bigr)^{-1}.
\]
Multiplying by $A^{1/2}$ gives
\[ A^{1/2}B_{\la}^{-1}=A^{1/2}A_{\lambda}^{-1}\bigl(I+(B-A)A_\lambda^{-1}\bigr)^{-1}.
\]
Hence
\begin{align}\label{eq:ABl}
\|A^{1/2}B_\lambda^{-1}\|
\le
\|A^{1/2}A_{\lambda}^{-1}\|\,
\|(I+(B-A)A_\lambda^{-1})^{-1}\|
\end{align}
Then,  recall the Neumann series: if $T\in\cL(\cH)$ and $\|T\|<1$, then $I+T$ is invertible and
\[
(I+T)^{-1}=\sum_{k=0}^{\infty}(-T)^k .
\]
Moreover,
\[
\|(I+T)^{-1}\|
\le
\sum_{k=0}^{\infty}\|T\|^k
=
\frac{1}{1-\|T\|}.
\]
Applying the above result to \(T=(B-A)A_\lambda^{-1}\), since $\|(B-A)A_\lambda^{-1}\|<1$, from \eqref{eq:ABl} we obtain 
\[\|A^{1/2}B_\lambda^{-1}\|
\le
\frac{\|A^{1/2}A_{\lambda}^{-1}\|}{1-\|(B-A)A_\lambda^{-1}\|}.
\]
If $\|(B-A)A_\lambda^{-1}\|\le \tfrac12$, then $(1-\|(B-A)A_\lambda^{-1}\|)^{-1}\le 2$, which gives the second bound. To conclude note that by  spectral calculus we have $\|A^{1/2}A_{\lambda}^{-1}\|\le 1/(2\sqrt{\la})$. 
\end{proof}

Finally, we will use the following convergence result.

\begin{lemma}\label{dom_conv}
Let $T\in\cL(\cH)$ be positive, self-adjoint and compact, and let 
$T_\lambda=T+\lambda I$ for $\lambda>0$. Then for every 
\(z\in \overline{\operatorname{ran} (T)}=(\operatorname{ker} (T))^\perp\),
\[
\lim_{\lambda\to0}\lambda\| T_\lambda^{-1}z\|=0 .
\]
\end{lemma}
\begin{proof}
Since $T$ is positive, self-adjoint and compact, by the spectral theorem there exist  
$(\mu_j)_{j\ge1}\subset[0,\infty)$ and an orthonormal basis $(e_j)_{j\ge1}$  of $\cH$ such that
 \( T e_j = \mu_j e_j \). Hence
\[
T_\lambda^{-1} z = \sum_{j \ge 0} \frac{\langle z, e_j \rangle }{\mu_j + \lambda}\, e_j
\]
Since $z\in (\operatorname{ker} (T))^\perp$
\[
\lambda^2\| T_\lambda^{-1}z\|^2
=
\sum_{j:\mu_j\neq 0}\Big(\frac{\lambda}{\mu_j+\lambda}\Big)^2 |\langle z,e_j\rangle|^2 .
\]
For every $j$ such that $\mu_j\neq 0$, 
\[
0 < \frac{\lambda}{\mu_j+\lambda}< 1 , \qquad 
\lim_{\lambda\to0}\frac{\lambda}{\mu_j+\lambda}=0.
\]
Since $\sum_{j\ge1}|\langle z,e_j\rangle|^2=\|z\|_{\mathcal{H}}^2<\infty$, by dominated convergence \(\lim_{\la\to 0} \la^2\| T_\lambda^{-1}z\|^2=0. \)
\end{proof}

\subsection{Covariance and related operators}\label{sec:covop}

We recall basic properties of covariance (second moment) operators and related quantities. We consider a  probability measure  $\rho\in \cP(\cX)$ that later could  be chosen as the invariant measure $\pi$ or a Dirac measure on a sample, trajectory $x_0, \dots, x_{T-1}$.

Let   $\Sigma_\rho:\cH\to \cH$ be defined as 
$\Sigma_\rho= \int \Phi(x)\otimes \Phi(x)\,\rho(dx)$. 
The integral is well defined in the Bochner sense in view of Condition~\eqref{bound_phi}, since 
\[
\int \|\Phi(x)\otimes \Phi(x)\|\,\rho(dx)
=
\int \|\Phi(x)\|^2\,\rho(dx)
<
\kappa^2.
\]
It is easy to see that $\Sigma_\rho$ is self-adjoint, since for all $f,g\in\cH$
\[
\scal{\Sigma_\rho f}{g}_\cH=\int \scal{f}{\Phi(x)}_\cH \scal{\Phi(x)}{g}_\cH\,\rho(dx)
=\scal{f}{\Sigma_\rho g}_\cH .
\]
It is positive since for all $f\in\cH$
\[
\scal{\Sigma_\rho f}{f}_\cH=\int |\scal{f}{\Phi(x)}_\cH|^2\,\rho(dx)\ge 0 .
\]
Moreover, $\Sigma_\rho$ is trace class since
\( \tr{\Sigma}=\int \|\Phi(x)\|^2\,\rho(dx)\le \kappa^2 \), because of Condition~\eqref{bound_phi}.
Let $L^2(\rho)$ be the Hilbert space of square-integrable functions with respect to $\rho$ with inner product $\scal{f}{g}_\rho=\int f(x)g(x)\rho(dx)$.
Let $S_\rho :\cH\to L^2(\rho)$ be defined as 
\begin{align}\label{eq:defS}
   S_\rho w=\scal{w}{\Phi(\cdot)}_\cH 
\end{align}
which is bounded because of Condition~\eqref{bound_phi}. Then, the  adjoint operator $S_\rho ^*:L^2(\rho)\to\cH$ is given by
\be\label{sstart}
S_\rho^*f=\int \Phi(x)f(x)\,\rho(dx),
\ee
and the covariance operator satisfies 
\be\label{cov_sstar}
\Sigma_\rho=S_\rho ^*S_\rho.
\ee
Let $L_\rho:L^2(\rho)\to L^2(\rho)$ be defined as 
\be\label{int_op}
L_\rho=S_\rho S_\rho^* .
\ee
It is easy to see that $L_\rho$ is self-adjoint and positive. Moreover, it is trace class and \( \tr{L_\rho}=\tr{\Sigma_\rho}.\)

\section{Concentration inequalities for Markov chains}
\label{sec:concentration-inequalities-markov}
\label{sec:proof-concentration-poisson-solution}

In this section, we consider Hilbert-space-valued observables $f$, and  derive concentration inequalities for partial sums of the type $1/T \sum_{t=0}^{T-1} f(X_t)$ under uniform geometric ergodicity (see Definition~\ref{UGE}) of the Markov chain $(X_t)_{t \in \bbN}$. The proof relies on a martingale decomposition induced by the Poisson equation for a \( \pi \)-centered observable \( f - \pi f \), where 
 \( \pi f = \int f d \pi \).
A bounded solution yields a decomposition of the partial sum into a martingale term and a remainder, enabling the use of concentration inequalities for Banach-valued martingales. We, therefore, first show that uniform ergodicity of $(X_t)_{t \in \bbN_0}$ guarantees the existence of such a bounded solution, and then derive the concentration bounds.

This approach extends ideas in \citet{gordin1969central} and \cite{glynn2002hoeffding} from scalar to Hilbert space valued observables.

\medskip \noindent
We introduce the setting of this section. Let \( \left( \cX, \cB(\cX) \right) \) be a measurable space and let \( \left( X_t \right)_{t \in \bbN_0} \) be a time-homogeneous Markov process taking values in \( \cX \) with transition kernel \( P: \cX \times \cB(\cX) \to [0,1] \).  We consider a real separable Hilbert space \( \left( \cH, \langle \cdot, \cdot \rangle_\cH \right) \), endowed with the \( \sigma \)-algebra $\cB(\cH)$ and let
\[
\cM_b(\cX,\cH)
:= \left\{f: \cX \to \cH \, \mid \, f \text{ measurable}, \sup_{x \in \cX} \norm[\cH]{f(x)} <+\infty \right\}.
\]
Note that here strong and weak measurability coincide, since $\cH$ is separable. The space $\cM_b(\cX,\cH)$ is a Banach space if endowed with the norm \( \norm[\infty]{ f } := \sup_{x \in \cX} \norm[\cH]{ f(x) } \).

\medskip \noindent
We are ready to state the main result of the section.

\subsection{Hoeffding inequality for Markov chains}

\begin{theorem}
\label{thm:ugeimplieshoeffding}
Let Assumption~\ref{UGE} hold.
Let \( f \in \cM_b(\cX,\cH) \) be such that \( \pi f = 0 \). Fix \( T \ge 1 \). Then, for every $\delta\in(0,1)$ the following holds with probability at least $1-\delta$
\begin{equation}
\label{eq:hoeffding-type-simple}
\norm[\cH]{\frac{1}{T} \sum_{t=0}^{T-1} f(X_t)}
\leq \frac{ 2C\|f\|_\infty }{(1-\rho)} \left(\sqrt{\frac{2 \log(2/\delta)}{T}} + \frac 1 T \right).
\end{equation}
\end{theorem}

\begin{remark}
In contrast to the classical Hoeffding inequality for i.i.d.\ samples, the bound in Theorem~\ref{thm:ugeimplieshoeffding} includes an additional term of order \( 1/T \). The term appears as the residual bias due to the truncation to a finite number of steps of the martingale constructed to represent the partial sum and vanishes as \( T \to \infty \) (as we will see in the proof). Moreover, the leading constants depend on the uniform geometric ergodicity assumption.
\end{remark}

\subsection{Poisson equation}

As mentioned above, the proof of Theorem~\ref{thm:ugeimplieshoeffding} is based on the existence of a solution of a corresponding  Poisson equation, see Definition\ref{eq:PE} We therefore first introduce  necessary notions and preliminary results , and postpone the proof of Theorem~\ref{thm:ugeimplieshoeffding}  to the end of the section.

With a slight abuse of notation we use \( P \) to denote the Markov operator 
\begin{align*}
P: \cM_b(\cX,\cH)& \to    \cM_b(\cX,\cH)&\\
f\quad &\mapsto  \quad P f, \qquad Pf(x) = \int f(x') P(x, dx'), \; x \in \cX. &
\end{align*}
Note that $Pf$ is measurable since $x \mapsto P(x,B)$ is measurable for every $B\in\cB({\cX})$. Moreover,
\begin{align}
\label{eq:nonexp}
\|Pf\|_{\infty}
:= \sup_{x \in \cX} \|Pf(x)\|_\cH
\leq \sup_{x \in \cX} \int \|f\|_{\infty} P(x, dx')=\|f\|_\infty,  
\end{align}
namely $P$ is nonexpansive on \( \cM_b(\cX,\cH) \), so that $Pf$ is well defined.

\begin{definition}[Poisson Equation]\label{eq:PE}
Let \(f \in \cM_b(\cX,\cH)\) be such that \( \pi f = 0 \). The \emph{Poisson equation} associated with \( f \) and \( P \) is the operator equation:
\begin{equation}
(I - P) g = f,
\label{eq:poisson-eq-def}
\end{equation}
where \( I \) denotes the identity operator on \(\cM_b(\cX,\cH)\).
\end{definition}
\noindent
See, e.g., \citet[Definition 21.2.1]{douc2018markov}.

A natural question is under which conditions on  $f$ and on the process $(X_t)_{t \in \bbN}$ the Poisson equation has a solution. In the next proposition we give an answer. 

\begin{proposition}[Poisson solution under uniform geometric  ergodicity]
\label{prop:poisson-geometric}
Under the conditions of Theorem~\ref{thm:ugeimplieshoeffding}, the Poisson equation~\eqref{eq:poisson-eq-def} admits a bounded solution given by the uniformly convergent von Neumann series
\[
g(x) = \sum_{k=0}^\infty P^k f(x), \quad x \in \cX.
\]
Moreover,
\[
\|g\|_{\infty} \leq \frac{C\|f\|_\infty }{1 - \rho}.
\]
\end{proposition}
\begin{proof}
{\emph Existence and boundedness of $g$.} For \( t \in \bbN \), define 
\[
g_t(x) = \sum_{k=0}^t P^kf(x), \quad x \in \cX.
\]
Fix \( m>t \). By the triangle inequality,
\[
\|g_m-g_t\|_\infty
\le \sup_{x \in \cX}\sum_{k=t+1}^{m}\gnorm{P^k f(x)}_{\cH}.
\]
Fix \(k \in \bbN\). Using \(\pi f = 0\) and the duality between the total variation norm  and the supremum norm on bounded measurable functions,
\begin{align}
\norm[\cH]{P^k f(x)}
&= \norm[\cH]{ P^k f(x) - \pi f } \notag \\
&= \norm[\cH]{\int f d \bigl(P^k(x,\cdot)-\pi\bigr)} \notag \\
&\le \nor{f}_\infty \, \norm[\mathrm{TV}]{P^k(x,\cdot)-\pi}
\le \nor{f}_\infty \cdot C \rho^k.
\label{eq:poisson-iterate-geom-bound}
\end{align}
Therefore,
\[
\|g_m-g_t\|_\infty
\le \|f\|_\infty \cdot C \left(\sum_{k=n+1}^{m}\rho^k\right)
\le  \|f\|_\infty \cdot \frac{C \rho^{n+1}}{1-\rho}\xrightarrow[]{n\to\infty}0,
\]
uniformly with respect to \(m\ge t\). Hence \((g_t)_t\) is a Cauchy sequence and converges in \(\|\cdot\|_\infty\) to some \(g\), and the convergence is uniform. In particular, \(g\) is  measurable, and for every \( x \in \cX \),
\[
g(x) = \sum_{k=0}^\infty P^k f(x), 
\]
so that
\[
\big\| g(x) \big\|_{\cH} \leq \sum_{k=0}^\infty \left\| P^k f(x) \right\|_{\cH}
\leq  \|f\|_\infty \cdot C \sum_{k=0}^\infty \rho^k
\leq \|f\|_\infty \cdot \frac{C}{1 - \rho}.
\]
Therefore, \( \|g\|_\infty \leq C\|f\|_{\infty}/{(1-\rho)} \).

\medskip \noindent
{\bf \( g \) solves Poisson equation.} Note that
\[
Pg_t(x)=\sum_{k=0}^t P^{k+1}f(x).
\]
Hence, for every \( x  \in \cX \),
\[
g_t(x)-Pg_t(x)
= \sum_{k=0}^t \left(P^{k}f(x)- P^{k+1}f(x)\right)
= f(x)-P^{t+1}f(x).
\]
Taking \(t \to \infty\), we have \(\|P^{t+1}f\|_\infty\le  \|f\|_\infty C\rho^{t+1} \to  0\), so that \(\|P^{t+1}f\|_{\infty} \to 0\). Also, since \(P\) is nonexpansive on \( \cM_b(\cX, \cH) \) with respect to \(\|\cdot\|_\infty\) (see equation~\ref{eq:nonexp}),
\[
\|Pg_t-Pg\|_\infty \le \|g_t-g\|_\infty \to 0.
\]
Thus \(Pg_t \to Pg\) uniformly and we can pass to the limit in \(g_t-Pg_t=f-P^{t+1}f\), obtaining
\[
g-Pg=f,
\]
as claimed.
\end{proof}

\subsection{Martingale decomposition}

The existence of a solution of the Poisson equation is crucial, since it allows to decompose the partial sum $\sum_{t=0}^{T-1} f(X_t)$ into a martingale plus a remainder, as we show below.  

\begin{lemma}\label{lem:mardec}
Let \(f \in \cM_b(\cX,\cH) \) be such that \( \pi f = 0 \). Suppose that $g \in \cM_b(\cX,\cH)$ is a solution of the Poisson equation \eqref{eq:poisson-eq-def}. Fix \( T \in \bbN \). Define
\[
M_T=\sum_{t=1}^{T} \big( g(X_{t}) - P g(X_{t-1}) \big)
\quad\text{ and } \quad
R_T = g(X_0) - g(X_T).
\]
Then,
\[
\sum_{t=0}^{T-1} f(X_t)= M_{T} + R_{T}.
\]
Moreover,
$(M_t)_{t \geq 1}$ is a martingale with respect to the natural filtration \( \cF_t = (\sigma(X_0,\ldots,X_t) )_{t \in \bbN}\).
\end{lemma}
\begin{proof}
Since \( g \) is a solution to the Poisson equation, for any \( t \in \bbN \),
\[
f(X_t) = g(X_t) - P g(X_t),
\]
so that
\begin{align*}
\sum_{t=0}^{T-1} f(X_t)
&= \sum_{t=0}^{T-1} \big( g(X_{t+1}) - P g(X_t) \big) + \sum_{t=0}^{T-1} \big( g(X_t)-g(X_{t+1})\big) \\
&= \underbrace{\sum_{t=1}^{T} \big( g(X_{t}) - P g(X_{t-1}) \big)}_{M_T} +  \underbrace{g(X_0)-g(X_{T})}_{R_T}.
\end{align*}
We next show that $(M_t)_{t \geq 1}$ is a martingale. Define the  \( \cH \)-valued process \( (d_t)_{t \geq 1} \) as
\begin{equation}
\label{eq:dt-def}
d_t := g(X_{t}) - P g(X_{t-1}).
\end{equation}
Clearly $M_T=\sum_{t=1}^{T} d_t$.  Consider the natural filtration \( \cF_{t} = \sigma(X_0, X_1, \ldots, X_t)\). For \( t \geq 1 \),
by the measurability of \( P g \) and the Markov property,
\begin{align*}
\bbE[d_t \mid \cF_{t-1}] 
= \underbrace{\bbE[g(X_t) \mid \cF_{t-1}]}_{P g(X_{t-1})} - \underbrace{\bbE[P g(X_{t-1}) \mid \cF_{t-1}]}_{P g(X_{t-1})}
= 0.
\end{align*}
Thus, the process \( (d_t)_{t \ge 1} \) is adapted to the filtration \( \cF = (\cF_t)_{t \ge 1} \) and satisfies the martingale difference property. Consequently, \( M_T \) is a  martingale in \( \cH \) with respect to \( \cF \). 
\end{proof}

\begin{remark}
This decomposition originates from \citet{gordin1969central} in the study of stationary Markov chains.
\end{remark}

\subsection{Proof of Theorem~\ref{thm:ugeimplieshoeffding}.}

Provided with the martingale  decomposition in  Lemma~\ref{lem:mardec} we can  apply known concentration results to prove Theorem~\ref{thm:ugeimplieshoeffding}.

Theorem~\ref{thm:ugeimplieshoeffding}.
\begin{proof}[Proof of Theorem~\ref{thm:ugeimplieshoeffding}]
The Poisson equation~\ref{eq:poisson-eq-def} with right hand side $f$ has a solution \( g \in \cM_b(\cX, \cH) \) due to  Proposition~\ref{prop:poisson-geometric}. Moreover $\|g\|_\infty \leq C\|f\|_\infty/(1-\rho)$. Using Lemma~\ref{lem:mardec}, we have 
\[
\sum_{t=0}^{T-1} f(X_t)=M_{T}+R_{T}=\sum_{t=1}^{T} d_t +R_{T},
\]
where $(d_t)_{t \ge 1}$ is defined in~\eqref{eq:dt-def}. Since \( P \) is nonexpansive on \( \cM_b(\cX, \cH) \), we have \( \|P g\|_{\infty} \leq \|g\|_{\infty} \). Thus, for any \( t \in \bbN \), by the triangle inequality,
\[
\big\|d_t\big\|_{\cH}
\leq \big\|g(X_t)\big\|_{\cH} + \big\|P g(X_{t-1})\big\|_{\cH}
\leq 2 \|g\|_{\infty}.
\]
We now apply, a concentration inequality for bounded Banach space-valued martingales \citet[Theorem 3.5]{pinelis1994optimum}, for every \( \varepsilon > 0 \),
\[
\bbP \left( \norm[\cH]{M_T} \ge \varepsilon \right)
\leq 2 \exp \left\{ - \frac{\varepsilon^2}{2 T \cdot (2 \| g \|_\infty)^2} \right\}.
\]
Denoting the right hand side with \( \delta \) and solving for \( \varepsilon \), we get for any \( \delta \in (0, 1) \) that the   following inequality holds with probability at least \( 1 - \delta \),
\[
\norm[\cH]{M_T} \leq \sqrt{8 T \| g\|_\infty^2 \log(2/\delta)}.
\]

Next, observe that by the triangle inequality, the remainder term satisfies
\[
\norm[\cH]{R_{T}}
\leq \norm[\cH]{g(X_1)} + \norm[\cH]{g(X_T)}
\leq 2 \|g\|_{\infty}.
\]
Thus, applying the triangle inequality one more time, we conclude that
\[
\norm[\cH]{\sum_{t=0}^{T-1} f(X_t)}
\leq \norm[\cH]{M_{T}} + \norm[\cH]{R_{T}} 
\leq 2 \|g\|_{\infty} \sqrt{2 T \log(2/\delta) } + 2 \|g\|_{\infty}
\]
with probability at least \( 1 - \delta \). Finally, dividing both sides by \( T \), we get
\[
\left\| \frac{1}{T} \sum_{t=0}^{T-1} f(X_t) \right\|_{\cH}
\leq 2 \|g\|_{\infty} \left( \sqrt{\frac{2 \log\left( 2 / \delta \right)}{T}} + \frac{1}{T} \right)
\leq \frac{2 C \| f \|_\infty}{ 1 - \rho } \left( \sqrt{\frac{2 \log\left( 2 / \delta \right)}{T}} + \frac{1}{T} \right),
\]
as claimed.
\end{proof}

\subsection{Pinelis--Bernstein inequality}

In the non-asymptotic analysis of the least squares estimator, we will also use another concentration inequality for Hilbert-valued martingales. The result is a direct consequence of the Bernstein-type tail inequality established by \citet[Theorem~3.3]{pinelis1994optimum}. For clarity, we reproduce below a slightly simplified formulation of this theorem, tailored to the setting and notation used in our analysis. See also \citet[Proposition~6]{DeMol2009} for an analogous inequality for sums of independent Hilbert-valued random variables.

\begin{theorem}\label{thm:pinelis} Let $(d_t)_{t \in \bbN_0}$ be a martingale difference sequence with respect to filtration \( (\cF_t)_{t \in \bbN_0} \) taking values in a real separable Hilbert space $\cH$ and satisfying for every \( t \in \bbN_0 \)
\[
\bbE\left[\|d_t\|_\cH^m \mid \cF_{t-1} \right] \leq \frac{m!}{2}\sigma^2 M^{m-2}, \quad \forall m\geq 2
\]
for some constants $M, \sigma < \infty$. Then, for every $\delta\in(0,1)$,
\begin{equation*}
\bbP \left\{ \left\|\frac{1}{T}\sum_{t=0}^{T-1} d_t\right\|_{\cH}
\leq \sigma\sqrt{\frac{2\log(2/\delta)}{T}} +\frac{M\log(2/\delta)}{T}\right\} \ge 1 - \delta.
\end{equation*}
\end{theorem}

\section{Learning guarantees for scalar states: proofs}\label{sec:learning_guar}
In this section we analyze approximation and estimation error and the proof of Theorem~\ref{thm:wellspecsys}. We first collect a few basic facts.

\subsection{Preliminaries}\label{sec:D1pre}
We begin reporting some basic formulas for the different minimizers we consider.  With the definition in Section~\ref{sec:covop}, let $S=S_\pi$, and $S^*, L, \Sigma$ be defined accordingly. 

First, recall the  estimator expression  \(\wh w_\la = (\wh  \Sigma +\la I) ^{-1} \wh h\), see~\eqref{ridgesol} and~\eqref{sigmah_hat}. Then,  note that 
\be\label{pop_ridgesol}
 w_\la = ( \Sigma + \la I)^{-1}  h ,
\ee
with 
\be\label{sigmah}
 \Sigma = \int  \Phi(x) \otimes \Phi(x) \, \pi(dx),
\qquad
 h = \int x' \, \Phi(x) \, P(x,dx') \, \pi(dx).
\ee
Note that under  Assumption~\ref{wellspec}, optimality conditions imply that 
\be\label{norm_eq}
\Sigma w_\star = h.
\ee
Recall from Appendix~\ref{sec:covop}, that $\wh \Sigma, \Sigma: \cH \to \cH$ are self-adjoint,  positive, trace-class operators, hence compact. To lighten the notation we let 
$$
\wh \Sigma_\la = \wh \Sigma+\la I, \quad \quad \Sigma_\la=  \Sigma+\la I.
$$
Finally, let $L^2(\pi)$, be the Hilbert space of square-integrable functions with respect to $\pi$ with inner product $\scal{f}{g}_\pi=\int f(x)g(x)\pi(dx)$. Then 
$\cE(f, f')=\nor{f-f'}_\pi^2$, and if for all $x\in \cX$ $f(x)= \scal{w}{\Phi(x)}$, $f'(x)= \scal{v}{\Phi(x)}$ for any $w, v \in \cH$, then
\be\label{norm_equiv}
\nor{f- f'}^2_\pi=  \nor{\Sigma^{1/2}(w- v)}^2,
\ee
noting that $f= Sw, f'=Sw'$ and using the polar decomposition of $S$.

\subsection{Approximation error analysis}
In this section we report the short proof of Proposition~\ref{prop:apprx}.

\begin{proof}
With the definition in Section~\ref{sec:covop}, let $S=S_\pi$, and $S^*, L, \Sigma$ be defined accordingly. 
Then, Assumption~\ref{universal_phi} can be written as
\[
\inf_{w\in \cH}\nor{Sw-f_\star}_\pi^2=0.
\]
Since $\operatorname{ran}(S)=\{Sw:\, w\in\cH\}$, this is equivalent to $f_\star\in \overline{\operatorname{ran}(S)}$. 
Moreover, since $L=SS^*$, we have \( \ker(L)=\ker(S^*) \),
and therefore
\(
\overline{\operatorname{ran}(L)}=(\ker L)^\perp
=(\ker S^*)^\perp
=\overline{\operatorname{ran}(S)}
\).
Then Assumption~\ref{universal_phi} is also equivalent to 
$$
f_\star\in \overline{\operatorname{ran}(L)}= (\ker L)^\perp.
$$

Now, note that it follows directly from~\eqref{sigmah} and~\eqref{sstart} that $h=S^*f_\star$. Recalling the definition of \( S \) in~\eqref{eq:defS} and~\eqref{pop_ridgesol} we have
$f_\la= S w_\la= S (\Sigma+\la I)^{-1}S^* f_\star= LL_\la^{-1}f_\star$, with $L_\la = L+\la I$ and $L=L_\pi$. Then 
$f_\la- f_\star= ( LL_\la^{-1}- I)f_\star= - \la  L_\la^{-1} f_\star$, and $\cE(f_\la, f_\star)= \la^2 \nor{L_\la^{-1}f_\star}_{\pi}^2$. Finally, the limit in~\eqref{apprx_conv} follows from Lemma~\ref{dom_conv}.

To prove the bound~\eqref{apprx_bound}, we begin 
combining~\eqref{norm_eq} and \eqref{pop_ridgesol} so that $w_\la = \Sigma_\la^{-1}\Sigma w_\star$. Then,  $w_\la - w_\star= ( \Sigma_\la^{-1}\Sigma-I)w_\star= - \la  \Sigma_\la^{-1}w_\star$ and using~\eqref{norm_equiv}, $\cE(f_\la, f_\star)= \la^2 \nor{\Sigma^{1/2}\Sigma_\la^{-1}w_\star}^2$
Finally, the bound in~\eqref{apprx_bound} follows from Lemma~\ref{lem:op_est} with \( T = \Sigma \), \( a=1\) and \(b = 1/2\).
\end{proof}

\subsection{Estimation error analysis}
The proof of the estimation error bounds can be  decoupled in an analytic and a probabilistic part.

\paragraph{Analytic bounds.}

Using purely analytic arguments we can derive the following bound. 
\begin{proposition} \label{prop:est_one}
Assume that \be\label{neumann_cond}
\la\ge  2 \nor{\wh \Sigma-\Sigma},
\ee 
and suppose that Assumptions~\eqref{bound}, and~\eqref{UGE} hold. Then,
\[
\sqrt{\cE(\wh f_\la, f_\la)}
\le
 \frac{1}{\sqrt{\la}}
\left(\nor{\wh h - h}  + 
\nor{\wh \Sigma - \Sigma}\frac{R}{\sqrt{\la}} \right)
\]
If in addition Assumption~\eqref{wellspec} holds, 
 then 

\begin{equation}\label{est_bound}
\sqrt{\cE(\wh f_\la, f_\la)}
\le  
 \frac{1}{\sqrt{\la}}
\left(\nor{\wh h - h}  + 
\nor{\wh \Sigma - \Sigma}\nor{w_\star}\right).
\end{equation}

\end{proposition}
The proof is standard and is reported for completeness. 
\begin{proof}
Using  Equations~\eqref{ridgesol} and \eqref{pop_ridgesol}, and adding and subtracting $\wh \Sigma_\la ^{-1} h$ we get
\[
\wh w_\la - w_\la =  \wh \Sigma_\la ^{-1} \left[\Big(\wh h - h\Big)+ \Big(\Sigma- \wh \Sigma \Big) w_\la\right],
\]
so that using Equation~\eqref{norm_equiv}
\begin{equation}
\nor{\wh f_\la - f_\la}_\pi \le \nor{\Sigma^{1/2}  \wh \Sigma_\la ^{-1}}\left[\nor{\wh h - h}+ \nor{\Sigma- \wh \Sigma} \nor{w_\la}\right].
\label{eq:pre-bound-template}
\end{equation}

Then, under Assumption~\eqref{bound} we have the bound
\be\label{pop_bound1}
\nor{w_\la}\le \frac{R}{\sqrt{\la}}.
\ee
This bound is derived  using the definitions in Appendices~\ref{sec:covop} and~\ref{sec:D1pre} to note that  $h=S^*f_\star$,  and $w_\la =  \Sigma_\la ^{-1} S^*f_\star$, then using Condition~\eqref{2mom}  so that $\nor{f_\star}_\pi\le R$, and the operator norm estimate $\nor{\Sigma_\la ^{-1} S^*}\le 1/\sqrt{\la}$ in Corollary~\ref{cor:op_est}.

\medskip \noindent
Under Assumption~\eqref{wellspec} we derive the  improved bound
\be\label{pop_bound2}
\nor{w_\la}\le \nor{w_\star}.
\ee
This bound is derived noting that from Equation~\eqref{norm_eq} {$w_\la =  \Sigma_\la ^{-1} \Sigma w_\star$}, and then using  the operator norm estimate $\nor{\Sigma_\la ^{-1} \Sigma}\le 1$ from Lemma~\ref{lem:op_est}. The  operator norm bound 
\be\label{op_est_neu}
\nor{\Sigma^{1/2}\wh \Sigma_\la^{-1}}\le \frac{1}{\sqrt{\la}}
\ee
follows by Lemma~\ref{lem:op_neu_est} thanks to condition~\eqref{neumann_cond}. {
The claims follow by substituting the bounds in Equations~\eqref{pop_bound1}, \eqref{pop_bound2}, and~\eqref{op_est_neu} into Equation~\eqref{eq:pre-bound-template}.
}
\end{proof}

\paragraph{Probabilistic bounds.}
Here we provide the bounds of the terms  $\|\wh h-h\|$ and $\nor{\wh \Sigma - \Sigma}$.
\begin{lemma}\label{lem:probbounds} Let $\delta\in (0,1)$. Under Assumptions~\ref{UGE} and \ref{bound}, each of the following estimates holds 
\begin{align*}
&\nor{\wh h - h}
\leq \frac{4 C \kappa R}{1-\rho} \left(\sqrt{\frac{2\log (6/\delta)}{T}}+\frac{1}{T}\right)
+ \frac{2 \kappa M \log(6/\delta)}{T} + \kappa \sigma \sqrt{\frac{2\log(6/\delta)}{T}}
\text{ with probability } 1-2\delta/3
 \\
\label{eq:empr-opr-diff-inequality}
&\nor{\wh \Sigma - \Sigma}_{\cL_2(\cH)}
\leq \frac{4 C \kappa^2 }{1-\rho} \left(\sqrt{\frac{2\log (6/\delta)}{T}}+\frac{1}{T}\right) \text{ with probability } 1-\delta/3. 
\end{align*}
\end{lemma}
\begin{proof}
To bound the term $\|\wh h-h\|$, consider
\begin{equation}\label{xi1}
\xi_1:\cX \times\cX \to \cH,
\quad
\xi_1(x,x')
= \Phi(x) \, x'- h,
\end{equation}
which is  is measurable. Moreover, by adding and subtracting \( \Phi(x) f_\star(x) \), it can be written as
\begin{equation}\label{z1_eta1}
    \xi_1(x, x') = \underbrace{\Phi(x)(x'-f_\star(x))}_{\zeta_1(x, x')} + \underbrace{\left( \Phi(x)f_\star(x)-h \right)}_{\eta_1(x)}.
\end{equation}

For every $t \geq 0$, define $d_t=\zeta_1(X_t,X_{t+1})$. Then $d_t$
 is measurable with respect to the $\sigma$-algebra $\cF_t=\sigma\left((X_0,X_1),\ldots,(X_t,X_{t+1})\right)=\sigma(X_0,\ldots,X_{t+1})$. In addition, by Markovianity and measurability,
\begin{align*}
\bbE[d_t \mid \cF_{t-1}]
&=\bbE[\Phi(X_t)(X_{t+1}-f_*(X_t)) \, \mid \, X_t]\\ 
&=\Phi(X_t) \, \Big(\bbE[X_{t+1} \mid X_t]-f_\star(X_t))\Big)\\
&=0.
\end{align*}
Therefore, \( (d_t)_t \) is a martingale difference sequence with respect to the filtration \( (\cF_t)_t \). Moreover, under Assumption~\ref{bound} for every \( m \ge 2 \),
\begin{align*}
\bbE[\| d_t \|^m \mid \cF_{t-1}]
&= \bbE[\| \Phi(X_t) \|^m \left| X_{t+1} - f_\star(X_t) \right|^m \mid X_t] \\
&= \| \Phi(X_t) \|^m \cdot \bbE[\left| X_{t+1} - f_\star(X_t) \right|^m \mid X_t] \\
&\leq \kappa^m \cdot \frac{m!}{2} \sigma^2 M^{m-2} = \frac{m!}{2} \left(\kappa \sigma\right)^2 \left(\kappa M\right)^{m-2}.
\end{align*}
Applying Theorem~\ref{thm:pinelis} we get with probability at least $1-\delta/3$,
\begin{equation} \label{eq:xi1}
\nor{\frac{1}{T}\sum_{t=0}^{T-1} \zeta_1(X_t,X_{t+1})}
= \nor{\frac{1}{T}\sum_{t=0}^{T-1} d_t}
\leq \frac{2 \kappa M \log(6/\delta)}{T} + \kappa \sigma \sqrt{\frac{2\log(6/\delta)}{T}}.
\end{equation}
Next, Assumptions~\ref{bound} and \ref{2mom} imply that $|\eta_1(x)|\leq 2\sup_{s\in \cX}|\Phi(s) f_\star(s)| \leq 2 \kappa R$. Moreover, 
\begin{equation*}
\pi(\eta_1)
=\int \left(\Phi(x)f_\star(x)-h\right) \, \pi(dx)
=\int \Phi(x)f_\star(x) \, \pi(dx)-\int f_\star(s)\Phi(s) \, \pi(ds)
=0.
\end{equation*}
Hence, Theorem~\ref{thm:ugeimplieshoeffding} yields that with probability at least $1-\delta/3$ it holds
\begin{align}\label{eq:eta1}
\nor{\frac{1}{T} \sum_{t=0}^{T-1} \eta_1(X_t)} 
\leq \frac{4 C \kappa R}{1-\rho}   \left(\sqrt{\frac{2\log(6/\delta)}{T}}+\frac{1}{T}\right).
\end{align}
Applying union bound to~\eqref{eq:xi1} and~\eqref{eq:eta1} we get that with probability at least $1-2\delta/3$ 
\begin{equation}\label{eq:boundhath}
\left\| \wh h - h \right\|
\leq \frac{4 C \kappa R }{1-\rho} \left(\sqrt{\frac{2\log (6/\delta)}{T}}+\frac{1}{T}\right)
+ \frac{2 \kappa M \log(6/\delta)}{T} + \kappa \sigma \sqrt{\frac{2\log(6/\delta)}{T}}.
\end{equation}
Next we provide a bound for $\nor{\wh \Sigma - \Sigma}$. We first note that
\[
\left\| \Sigma - \widehat{\Sigma} \right\| \leq \left\| \widehat{\Sigma} - \Sigma \right\|_{\cL_2(\cH)}.
\]
To bound the term on the right hand side, define the observable
\begin{equation}
\label{eq:bound-xi-1}
\xi_2: \cX \to \cL_2(\cH),
\quad \xi_2(x) = \Phi(x) \otimes \Phi(x) - \Sigma.
\end{equation}
By definition of \( \Sigma \), we have \( \pi(\xi_2) = 0 \). Since \( \nor{\Phi(x) \otimes \Phi(x)}_{\cL_2(\cH)} = \nor{\Phi(x)}^2 \leq \kappa^2 \) by Assumption \ref{bound} (equation~\ref{bound_phi}), it follows by the triangle inequality and convexity of the norm that
\[
\nor{\xi_2}_\infty
:= \sup_{x \in \cX} \nor{\xi_2(x)}_{\cL_2(\cH)}
\leq 2 \kappa^2.
\]
\Cref{thm:ugeimplieshoeffding} yields that for every \( \delta \in (0,1) \),
\begin{equation}
\label{eq:empr-opr-diff-inequality}
\nor{\wh \Sigma - \Sigma}_{\cL_2(\cH)}
= \nor{\frac{1}{T} \sum_{t=0}^{T-1} \xi_2(X_t)}_{\cL_2(\cH)}
\leq \frac{4 C \kappa^2 }{1-\rho} \left(\sqrt{\frac{2\log (6/\delta)}{T}}+\frac{1}{T}\right) 
\end{equation}
with probability at least \( 1 - \delta/3 \). 
\end{proof}

\begin{proof}[Proof of Proposition~\ref{prop:est-error}]
For \( \delta \in (0,1) \), on the event
\[
\nor{\wh \Sigma - \Sigma}_{\cL_2(\cH)}
\leq \frac{4 C \kappa^2 }{1-\rho} \left(\sqrt{\frac{2\log (6/\delta)}{T}}+\frac{1}{T}\right),
\]
the condition \(\lambda\ge \lambda_{T,\delta}\) implies
\[
\lambda_{T,\delta} \ge 2 \nor{ \wh \Sigma - \Sigma }_{\cL_2(\cH)},
\]
so that \( \lambda \ge 2 \nor{ \wh \Sigma - \Sigma } \).
The statement follows by inserting the probabilistic estimates of Lemma~\ref{lem:probbounds} into the bound derived in Proposition~\ref{prop:est_one} followed by application of \( (a + b + c)^2 \leq 3 a^2 + 3 b^2 + 3 c^2 \) inequality.
\end{proof}

The main results are a direct consequence of the previous results. 

\begin{proof}[Proof of Theorem~\ref{thm:universal}]
We first prove the almost sure convergence of the estimation error. 
For \(T \geq 1\), set \( \delta_T := T^{-2} \), so that \( \sum_{T=1}^\infty \delta_T < \infty \). Moreover,
\[
\log(6/\delta_T)=\log(6T^2)\lesssim \log T .
\]
To invoke Proposition~\ref{prop:est-error}, the choice of \( \lambda_T \) should satisfy
\[
\lambda_T \geq \lambda_{T,\delta_T}
= \frac{8 C\kappa^2}{1-\rho} \left(\sqrt{\frac{2\log(6T^2)}{T}} + \frac{1}{T}\right),
\]
Since \( \log T / (\lambda_T^2 T) \to 0 \),
we have, in particular, \( \lambda_T^2 T \gg \log T \). Hence
\[
\lambda_T
\gg \sqrt{\frac{\log T}{T}},
\]
since \( \log T/T \to 0 \). Moreover,
\[
\frac{1}{T}=o\!\left(\sqrt{\frac{\log T}{T}}\right).
\]
Therefore
\[
\lambda_{T,\delta_T} =
O\!\left(\sqrt{\frac{\log T}{T}}+\frac{1}{T}\right)
= O\!\left(\sqrt{\frac{\log T}{T}}\right).
\]
It follows that, for all sufficiently large \(T\), \( \lambda_T \geq \lambda_{T,\delta_T} \), thus Proposition~\ref{prop:est-error} is applicable for all sufficiently large \(T\).

\medskip \noindent
The dominant term in the bound given in Proposition~\ref{prop:est-error} is
\[
\cE\big(\wh f_{\lambda_T}, f_{\lambda_T}\big)
= O\!\left( \frac{\log T}{\lambda_T^2 T} \right)
\]
on event of probability at least \(1 - \delta_T\).

\medskip \noindent
By the Borel--Cantelli lemma, since \(\sum_T \delta_T<\infty\), with probability one the bound holds for all sufficiently large \(T\). Consequently,
\[
\cE\big(\wh f_{\lambda_T}, f_{\lambda_T}\big)
= O_{\mathrm{a.s.}}\!\left(
\frac{\log T}{\lambda_T^2 T}
\right).
\]
Therefore,
\[
\frac{\log T}{\lambda_T^2 T} \to 0
\implies \cE\big(\wh f_{\lambda_T}, f_{\lambda_T}\big) \to 0
\quad\text{a.s.}.
\]

\medskip \noindent
Using the  decomposition~\eqref{err_dec}, we have
$$
    \cE\big(\wh f_{\lambda_T}, f_\star\big)
\leq 2 \cE\big(\wh f_{\lambda_T}, f_{\lambda_T}\big) + 2\cE\big(f_{\lambda_T}, f_\star\big).
$$

By Proposition~\ref{prop:apprx} and Assumption~\ref{universal_phi}, \( \cE\big(f_{\lambda_T}, f_\star\big) \to 0 \) because \(\lambda_T \to 0\). Combining the two limits gives \( \cE\big(\wh f_{\lambda_T}, f_\star\big) \to 0 \quad \text{a.s.} \), as claimed.
\end{proof}

\begin{proof}[Proof of Theorem~\ref{thm:wellspecsys}]  The statement follows directly by the inequality
\begin{equation}\label{err_dec2}
\cE\big(\wh f_\la, f_\star\big)
\leq 2\cE\big(\wh f_\la, f_\la\big) + 2 \cE\big( f_\la, f_\star\big)
\end{equation}
and by Propositions~\ref{prop:apprx}  and \ref{prop:est-error} after discarding higher-order terms. Furthermore, for the choice of \( \la_T = \gamma T^{-1/2} \), for any \( \delta \in (0,1) \) there exists admissible \( \gamma \) such that \( \la_T \ge \Lambda_{T,\delta} \).
\end{proof}

\section{Further proofs}\label{sec:furth_proofs}

The proof for scalar-valued states was structured so that the proofs for the  various extensions  follow straightforwardly. We collect these proofs in the following sections, starting with vector-valued states.

\subsection{Proof of the learning guarantees for vector states}
In this section, we prove Theorem~\ref{thm:wellspecsys_vv}.
The proof follows, with minor modifications, the one for scalar-valued states.
In particular, estimators are now parametrized by Hilbert--Schmidt operators
\(W\in \cL_2(\cX,\cH)\), instead of vectors $w\in \cH$. Moreover, if \(f(x)=W^*\Phi(x)\) and \(f'(x)=V^*\Phi(x)\), with
\(W,V\in\cL_2(\cX,\cH)\), then
\be\label{norm_equiv_vv}
\cE(f,f')
= \nor{f-f'}_\pi^2
= \nor{\Sigma^{1/2}(W-V)}_{\cL_2(\cX,\cH)}^2.
\ee
This identity, analogous to Equation~\eqref{norm_equiv}, follows by noting
that, if \(f(x)=W^*\Phi(x)\), then
\[
\nor{f}_\pi^2
= \int \scal{\Phi(x)}{WW^*\Phi(x)}_\cH\,\pi(dx)
= \int \tr{WW^*\Phi(x)\otimes\Phi(x)}\,\pi(dx)
= \tr{WW^*\Sigma}
= \nor{\Sigma^{1/2}W}_{\cL_2(\cX,\cH)}^2.
\]
We are ready to give the proof of Theorem~\ref{thm:wellspecsys_vv}.
\begin{proof}
We follow the same proof structure as in the scalar case, highlighting which steps are different.
Let
\[
W_\la=\argmin_{W\in \cL_2(\cX,\cH)}
\int \nor{W^*\Phi(x)-x'}_{\cX}^2 P(x,dx')\,\pi(dx)
+\la \nor{W}_{\cL_2(\cX,\cH)}^2,
\]
and let $f_\la$ be the corresponding prediction function. Then
$W_\la=(\Sigma+\la I)^{-1}h$, where
\[
\Sigma=\int \Phi(x)\otimes\Phi(x)\,\pi(dx),
\qquad
h=\int \Phi(x)\otimes x'\,P(x,dx')\,\pi(dx).
\]
It is easy to check that $h\in\cL_2(\cX,\cH)$ under Conditions~\eqref{2mom_vv} and~\eqref{bound_phi_vv}. Then, we can consider an error decomposition as in the scalar case, see~\eqref{err_dec} and~\eqref{err_dec2}, to obtain
\begin{equation}\label{dec_vv}    
\cE\big(\wh f_{\la},f_\star\big)
\leq 2\cE\big(\wh f_{\la},f_{\la}\big)
+2\cE\big(f_{\la},f_\star\big).
\end{equation}
For the  approximation error $\cE\big(f_{\la},f_\star\big)$, we can  follow the argument in Proposition~\ref{prop:apprx} to derive the bound
\begin{equation}\label{apprx_bound_vv}
\cE\big(f_{\lambda}, f_\star\big)\le \la^2 \nor{\Sigma^{1/2} \Sigma_\la^{-1} W_\star}_{\cL_2(\cX, \cH)}^2\le \la \nor{W_\star}_{\cL_2(\cX, \cH)}^2.
\end{equation}
For the estimation error $\cE\big(\wh f_{\la},f_{\la}\big)$, adapting Proposition~\ref{prop:est_one} we find that if 
$$
\la\ge  2 \nor{\wh \Sigma-\Sigma},
$$
and Assumptions~\eqref{bound_vv}, and  Assumption~\eqref{wellspec_vv} hold, 
 then 
\begin{equation}\label{est_bound_vv}
\sqrt{\cE(\wh f_\la,f_\la)}
\le
\frac{1}{\sqrt{\la}}
\left(
\nor{\wh h-h}_{\cL_2(\cX,\cH)}
+
\nor{\wh\Sigma-\Sigma}\,
\nor{W_\star}_{\cL_2(\cX,\cH)}
\right).
\end{equation}
Note that compared to Proposition~\ref{prop:est_one}, here the deviation $\wh h - h$ is measured in $\cL_2(\cX, \cH)$  rather that in $\cH$, and $\nor{w_\star}$ is replaced  by $\nor{W_\star}_{\cL_2(\cX, \cH)}$.
 
Then, we need to adapt Lemma~\ref{lem:probbounds}. The analysis of the term
\(\nor{\wh \Sigma-\Sigma}\) is unchanged, but we need to adapt the proof to
control \(\nor{\wh h-h}_{\cL_2(\cX,\cH)}\). Extending the ideas in
Equations~\eqref{xi1} and~\eqref{z1_eta1}, we consider
\[
\xi_1:\cX\times\cX\to\cL_2(\cX,\cH),
\qquad
\xi_1(x,x')
=
\Phi(x)\otimes x'-h.
\]
Adding and subtracting \(\Phi(x)\otimes f_\star(x)\), we obtain
\[
\xi_1(x,x')
=
\underbrace{\Phi(x)\otimes\bigl(x'-f_\star(x)\bigr)}
_{\zeta_1(x,x')}
+
\underbrace{\bigl(\Phi(x)\otimes f_\star(x)-h\bigr)}
_{\eta_1(x)}.
\]
The analysis in Lemma~\ref{lem:probbounds} can then be adapted directly under Assumption ~\ref{bound_vv}, since
it is based on probabilistic bounds for Hilbert-space-valued random variables,
which can be applied in \(\cL_2(\cX,\cH)\). The resulting bound is
\[
\nor{\wh h-h}_{\cL_2(\cX,\cH)}
\leq
\frac{4C\kappa R}{1-\rho}
\left(
\sqrt{\frac{2\log(6/\delta)}{T}}+\frac{1}{T}
\right)
+
\frac{2\kappa M\log(6/\delta)}{T}
+
\kappa\sigma\sqrt{\frac{2\log(6/\delta)}{T}},
\]
with probability at least \(1-2\delta/3\). This bound has the same form as in
the scalar-valued setting, up to the choice of norm and the definitions of the
constants.
The above bound , together with the bound for \(\nor{\wh\Sigma-\Sigma}\) in Lemma~\ref{lem:probbounds} and the analytic estimate in~\eqref{est_bound_vv}, yields the final estimation error bound. In particular, if Assumption~\ref{wellspec} holds and
\[
\la \ge \frac{8 C \kappa^2}{1-\rho}
\left(
\sqrt{\frac{2\log(6/\delta)}{T}}+\frac{1}{T}
\right),
\]
then, with probability at least \(1-\delta\),
\begin{equation}\label{est_err_vv}
\cE\bigl(\wh f_\la,f_\la\bigr)
\le 
\frac{C^2D_3}{(1-\rho)^2}
\left(
\frac{2\log(6/\delta)}{\la T}+\frac{1}{\la T^2}
\right)
+\frac{D_1\log^2(6/\delta)}{\la T^2}
+\frac{D_2\log(6/\delta)}{\la T},
\end{equation}
where \(D_1:=12\kappa^2M^2\), \(D_2:=6\kappa^2\sigma^2\), and \(D_3:=96\kappa^2\bigl(R+\kappa\nor{W_\star}_{\cL_2(\cX,\cH)}\bigr)^2\).
Plugging this estimation error bound and the  approximation error bound~\eqref{apprx_bound_vv} into the error decomposition~\eqref{dec_vv} yields the error bound~\eqref{eq:sample-rate-bound-simplified_vv}. The final error rate is obtained by choosing $\la$ to balance the error contributions in the bound. 
\end{proof}

\subsection{Proof of the learning guarantees for higher-order systems}

In this section, we prove Theorem~\ref{thm:wellspecsys_p}. The proof is obtained by applying the vector-valued analysis to the lifted process $(\overline X_t)_{t\in\bbN_0}$. The main differences are that the inputs belong to $\cX^p$, the invariant measure is $\overline\pi$, and the empirical averages contain $T-p+1$ terms.

\begin{proof}
Let
$$
W_\la
=
\argmin_{W\in\cL_2(\cX,\cH)}
\left\{
\int
\nor{W^*\Phi(\overline x)-x'}_{\cX}^2
\,P(\overline x,dx')\,\overline\pi(d\overline x)
+
\la\nor{W}_{\cL_2(\cX,\cH)}^2
\right\},
$$
and let $f_\la(\overline x)=W_\la^*\Phi(\overline x)$. Then
$
W_\la=(\Sigma+\la I)^{-1}h,
$
where
\[
\Sigma
=
\int
\Phi(\overline x)\otimes\Phi(\overline x)
\,\overline\pi(d\overline x),
\quad
h
=
\int
\Phi(\overline x)\otimes x'
\,P(\overline x,dx')\,\overline\pi(d\overline x).
\]
Under Conditions~\eqref{2mom_p} and~\eqref{bound_phi_p},
$h\in\cL_2(\cX,\cH)$.

The deterministic part of the argument is unchanged from the vector-valued case. Indeed, for
$f(\overline x)=W^*\Phi(\overline x)$ and
$f'(\overline x)=V^*\Phi(\overline x)$,
\[
\cE(f,f')
= \nor{\Sigma^{1/2}(W-V)}_{\cL_2(\cX,\cH)}^2.
\]
Therefore, as in~\eqref{dec_vv},
\begin{equation}\label{dec_p}
\cE\bigl(\wh f_\la,f_\star\bigr)
\le 2\cE\bigl(\wh f_\la,f_\la\bigr)
+ 2\cE\bigl(f_\la,f_\star\bigr).
\end{equation}
Moreover, the approximation bound~\eqref{apprx_bound_vv} gives
\begin{equation}\label{apprx_bound_p}
\cE\bigl(f_\la,f_\star\bigr)
\le
\la\nor{W_\star}_{\cL_2(\cX,\cH)}^2.
\end{equation}
Similarly, the deterministic estimation bound~\eqref{est_bound_vv} yields, whenever
\[
\la\ge 2\nor{\wh\Sigma-\Sigma},
\]
\begin{equation}\label{est_bound_p}
\sqrt{\cE\bigl(\wh f_\la,f_\la\bigr)}
\le
\frac{1}{\sqrt{\la}}
\left(\nor{\wh h-h}_{\cL_2(\cX,\cH)} + \nor{\wh\Sigma-\Sigma}
\nor{W_\star}_{\cL_2(\cX,\cH)}
\right).
\end{equation}

It remains to control the two empirical deviations. The bound for
$\nor{\wh\Sigma-\Sigma}$ follows directly from Lemma~\ref{lem:probbounds}, applied to the lifted process under Assumption~\ref{UGE_p}.
We can also control $\nor{\wh h-h}_{\cL_2(\cX,\cH)}$ with minor adaptations to the arguments seen for scalar and vector states. Consider
\[
\xi_1:\cX^p\times\cX\to\cL_2(\cX,\cH),
\qquad \xi_1(\overline x,x')
= \Phi(\overline x)\otimes x'-h.
\]
Adding and subtracting
$\Phi(\overline x)\otimes f_\star(\overline x)$ gives
\[
\xi_1(\overline x,x')
= \underbrace{
\Phi(\overline x)\otimes
\bigl(x'-f_\star(\overline x)\bigr)
}_{\zeta_1(\overline x,x')}
+ \underbrace{
\left(
\Phi(\overline x)\otimes f_\star(\overline x)-h
\right)
}_{\eta_1(\overline x)}.
\]
As in the vector-valued case, the first term yields a martingale difference sequence, while the second is an additive functional of the lifted Markov process. Hence, under Assumptions~\ref{UGE_p} and~\ref{bound_p}, the argument of Lemma~\ref{lem:probbounds}, applied to the lifted process, gives
\[
\nor{\wh h-h}_{\cL_2(\cX,\cH)}
\le \frac{4C\kappa R}{1-\rho}
\left(
\sqrt{\frac{2\log(6/\delta)}{T-p+1}}
+ \frac{1}{T-p+1}
\right)
+ \frac{2\kappa M\log(6/\delta)}{T-p+1}
+ \kappa\sigma
\sqrt{\frac{2\log(6/\delta)}{T-p+1}},
\]
with probability at least $1-2\delta/3$.

Thus, the probabilistic bounds are exactly those obtained for vector-valued states, with $T$ replaced by $T-p+1$ and with the constants $C$ and $\rho$ associated with the lifted process. Combining these bounds with~\eqref{est_bound_p} shows that, if
\[
\la
\ge \frac{8C\kappa^2}{1-\rho}
\left( \sqrt{\frac{2\log(6/\delta)}{T-p+1}} + \frac{1}{T-p+1} \right),
\]
then the estimation error satisfies the same bound as in~\eqref{est_err_vv}, with $T$ replaced by $T-p+1$.

Combining the resulting estimation bound with~\eqref{apprx_bound_p} in~\eqref{dec_p} yields~\eqref{eq:sample-rate-bound-simplified_p}. The final rate follows by choosing
\[
\la\asymp(T-p+1)^{-1/2}.
\]
\end{proof}

\subsection{Proof of the learning guarantees for finite-state spaces}
\label{sec:appendix-finite-states}

Before proving the main theorem, we provide the proof the comparison Lemma~\ref{comp_ineq1}.

\begin{proof}[Proof of Lemma~\ref{comp_ineq1}]
Equation~\eqref{fisher} follows by noting that 
\[
D(g_\star(x))
\in \argmax_{x'\in \cX} \langle E(x'), g_\star(x)\rangle
= \argmax_{x'\in \cX} P(x,x'),
\]
where we used that, for one-hot encoding, the vectors \(E(x')\) form an orthonormal basis, so that
\[
\langle E(x'), E(z)\rangle =
\begin{cases}
1 & \text{if } z=x',\\
0 & \text{otherwise}, 
\end{cases}
\]
and hence
\[
\langle E(x'), g_\star(x)\rangle
= \sum_{z\in \cX} \langle E(x'), E(z)\rangle P(x,z)
= P(x,x').
\]
To prove~\eqref{eq:boundE}, fix \(x\in\cX\). If \(f(x)=f_\star(x)\), the claim is trivial. Assume \(f(x)\neq f_\star(x)\). Applying decoder \( D \), for every $x\in \cX$
\[
f(x) \in \argmax_{x'\in\cX} \langle E(x'), g(x)\rangle.
\]
Hence, taking \(x' = f_\star(x)\), we get
\[
\langle E(f(x)), g(x)\rangle
\ge \langle E(f_\star(x)), g(x)\rangle.
\]
and
\begin{equation}\label{eq:dec}
\langle E(f_\star(x)) - E(f(x)), g(x)\rangle \le 0.
\end{equation}
Next, using the definition $g_\star(x):=\sum_{x'\in\cX} E(x')\,P(x,x')\)
and the fact that one hot encoding vectors are orthonormal, we get
\[
\langle E(f_\star(x)), g_\star(x)\rangle
= P(x,f_\star(x)),
\qquad \langle E(f(x)), g_\star(x)\rangle
= P(x,f(x)).
\]
Moreover, using~\eqref{gap} and the fact that \(f(x)\neq f_\star(x)\), we get
\begin{equation}\label{eq:gap-point}
\langle E(f_\star(x)) - E(f(x)), g_\star(x)\rangle
= P(x,f_\star(x)) - P(x,f(x)) \ge \gamma.
\end{equation}
From~\eqref{eq:dec}, adding and subtracting $\langle E(f_\star(x)) - E(f(x)), g(x)\rangle$, we get
\[
\langle E(f_\star(x)) - E(f(x)), g_\star(x)\rangle
- \langle E(f_\star(x)) - E(f(x)), g(x)\rangle
\ge \langle E(f_\star(x)) - E(f(x)), g_\star(x)\rangle,
\]
so that using~\eqref{eq:gap-point}
\[
\langle E(f_\star(x)) - E(f(x)), g_\star(x)-g(x)\rangle
\ge \gamma.
\]
By Cauchy--Schwarz and \(\nor{E(f_\star(x)) - E(f(x))}_{\bbR^N}\le \sqrt{2}\),
\[
\gamma
\le \nor{g(x)-g_\star(x)}_{\bbR^N}\,\nor{E(f_\star(x)) - E(f(x))}_{\bbR^N}
\le \sqrt{2}\,\nor{g(x)-g_\star(x)}_{\bbR^N}.
\]
Since \(f(x)\neq f_\star(x)\), then 
\[
\mathds{1}\{f(x)\neq f_\star(x)\}
\le \frac{\sqrt{2}}{\gamma}\,\nor{g(x)-g_\star(x)}_{\bbR^N}
\le \frac{2}{\gamma^2} \nor{g(x)-g_\star(x)}^2.
\]
Integrating with respect to \(\pi\) yields
\[
\begin{aligned}
\cE(f,f_\star)
&= \sum_{x\in\cX} \mathds{1}\{f(x)\neq f_\star(x)\} \, \pi(x)\\
&\leq \frac{2}{\gamma^2} \sum_{x \in \cX} \nor{g(x)-g_\star(x)}^2\pi(x) \\
&= \frac{2}{\gamma^2}\nor{g-g_\star}_\pi^2
=: \frac{2}{\gamma^2} \left(\cI(g)-\cI(g_\star)\right),
\end{aligned}
\]
as claimed.
\end{proof}

We are now ready to prove Theorem~\ref{thm:fs}.

\begin{proof}[Proof of Theorem~\ref{thm:fs}]
Consider the encoded process
\[
Y_t=E(X_t),
\qquad t\in\bbN_0.
\]
The one-hot encoder is injective and measurable. Therefore, Lemma~\ref{lem:emb-markov-conv} shows that \((Y_t)_{t\in\bbN_0}\) is a Markov process with invariant measure \(E_{\#}\pi\) and satisfies the same uniform geometric convergence bound as \((X_t)_{t\in\bbN_0}\), with constants \(C\) and \(\rho\).

For the identity feature map,
\[
\nor{\Phi(E(x))}_{\bbR^N} = \nor{E(x)}_{\bbR^N} = 1,
\qquad x \in \cX,
\]
so that the bounded-feature condition holds with \(\kappa=1\); see also Remark~\ref{rem:effective-boundedness}. Moreover,
\[
\sum_{x'\in\cX} \nor{E(x')}_{\bbR^N}^2 P(x,x') 
= 1,
\qquad x \in \cX,
\] and hence the second-moment condition holds with \(R=1\). The remaining residual moment condition is precisely Assumption~\ref{bound_finite}.

Under one-hot encoding,
\[
g_\star(x)
:= \sum_{x'\in\cX}E(x')P(x,x')
= P^\top E(x),
\qquad x \in \cX.
\]
Thus, the regression problem is well specified. More precisely, under the parameterization \(g(x)=W^\top E(x)\),  \(W_\star=P\), and  \( \nor{W_\star}_{\cL_2(\bbR^N)} = \nor{P}_{\cL_2(\bbR^N)} \).

Fix \( \delta \in (0,1) \). Applying Theorem~\ref{thm:wellspecsys_vv} to the encoded process gives, for a sufficiently large \( \lambda \), with probability at least \( 1 - \delta \),
\[
\nor{\wh g_\lambda-g_\star}_\pi^2
\lesssim \frac{ 2  C^2 D }{(1-\rho)^2} \left(\frac{2\log (6/\delta)}{\la T}+\frac{1}{\la T^2}\right) + 2 \lambda\nor{P}_{\cL_2(\bbR^N)}^2,
\]
with constants in~\eqref{est_err_vv} given by \( D_1=12M^2 \), \( D_2=6\sigma^2 \), and \( D_3=96\bigl(1+\nor{P}_{\cL_2(\bbR^N)}\bigr)^2 \). Under Condition~\eqref{unique}, Lemma~\ref{comp_ineq1} yields
\[
\cE(\wh f_\lambda,f_\star)
\leq \frac{ 4  C^2 D }{ \gamma^2 (1-\rho)^2} \left(\frac{2\log (6/\delta)}{\la T}+\frac{1}{\la T^2}\right) + \frac{4 \lambda}{\gamma^2} \nor{P}_{\cL_2(\bbR^N)}^2.
\]
The stated rates follow by taking \(\lambda_T\asymp T^{-1/2}\).
\end{proof}

\begin{remark}[Boundedness on the effective state space]
\label{rem:effective-boundedness}
Strictly speaking, the identity feature map
\[
\Phi:\bbR^N \to \bbR^N,
\qquad \Phi(z)=z,
\]
is not bounded on \( \bbR^N \), and hence does not satisfy Assumption~\ref{bound_vv}, specifically~\eqref{bound_phi_vv} on all of \(\bbR^N\). However, in the application of Theorem~\ref{thm:wellspecsys_vv}, the feature map is evaluated only along the encoded process \(E(X_t)\), whose state space is contained in \(E(\cX)\). It is apparent from the proof of Theorem~\ref{thm:wellspecsys_vv} that it is sufficient to require
\[
\sup_{x\in\cX}\nor{\Phi(E(x))}_{\bbR^N}
= \sup_{z\in E(\cX)}\nor{\Phi(z)}_{\bbR^N} 
\leq \kappa.
\]
Equivalently, one may think about the encoded process as taking values in the state space \(E(\cX)\), on which the \emph{restriction} of \(\Phi\) is bounded.
\end{remark}

\subsection{Proof of the learning guarantees for Koopman operators}

In this section, we prove Theorem~\ref{thm:wellspecsys_koop}. Like for vector-valued states and higher-order process the proof follows the key idea explained for scalar-valued states. In particular, we can follow the changes done for vector-valued states adding the following key observation: under Assumption~\ref{wellspec_koop} and $\cF$ as in~\eqref{F}, for every $A$ of the form~\eqref{lin_est_kopp}, the following inequality holds, 
\begin{equation}\label{koop_error_comparison}
\cE_\cF(A,A_\star)
\le
\nor{\Sigma^{1/2}(W-W_\star)}_{\cL_2(\cH)}^2.
\end{equation}
Indeed, first, for every $v\in\cH$,
\[
\int
\left|
\scal{(W-W_\star)v}{\Phi(x)}
\right|^2
\,\pi(dx)
=
\nor{\Sigma^{1/2}(W-W_\star)v}^2.
\]
Therefore,
\[
\cE_\cF(A,A_\star)
=
\sup_{\nor{v}\le 1}
\nor{\Sigma^{1/2}(W-W_\star)v}^2
=
\nor{\Sigma^{1/2}(W-W_\star)}_{\cL(\cH)}^2
\le
\nor{\Sigma^{1/2}(W-W_\star)}_{\cL_2(\cH)}^2,
\]
where the last inequality follows since the operator norm is bounded by the Hilbert--Schmidt norm.
We are then ready to prove Theorem~\ref{thm:wellspecsys_koop}.

\begin{proof}[Proof of Theorem~\ref{thm:wellspecsys_koop}]
Let
\[
W_\la
=
\argmin_{W\in\cL_2(\cH)}
\left\{
\int
\nor{\Phi(x')-W^*\Phi(x)}_\cH^2
\,P(x,dx')\,\pi(dx)
+
\la\nor{W}_{\cL_2(\cH)}^2
\right\},
\]
and let $A_\la $ be the corresponding operator on the considered observables. Then
\(
W_\la=(\Sigma+\la I)^{-1}h,
\)
where
\[
\Sigma
=
\int
\Phi(x)\otimes\Phi(x)
\,\pi(dx),
\qquad
h
=
\int
\Phi(x)\otimes\Phi(x')
\,P(x,dx')\,\pi(dx).
\]
Condition~\eqref{boundphi} implies that $h\in\cL_2(\cH)$.

Under Assumption~\ref{wellspec_koop}, we have
\[
g_\star(x)=W_\star^*\Phi(x),
\qquad x\in\cX,
\]
and therefore
\[
\Sigma W_\star=h.
\]
Using~\eqref{koop_error_comparison}, we obtain the error decomposition
\begin{align}
\cE_\cF\bigl(\wh A_\la,A_\star\bigr)
&\le
\nor{\Sigma^{1/2}(\wh W_\la-W_\star)}_{\cL_2(\cH)}^2
\nonumber\\
&\le
2\nor{\Sigma^{1/2}(\wh W_\la-W_\la)}_{\cL_2(\cH)}^2
+
2\nor{\Sigma^{1/2}(W_\la-W_\star)}_{\cL_2(\cH)}^2.
\label{dec_koop}
\end{align}

The approximation term is controlled as in the vector-valued case. Indeed,
\[
W_\la
= \Sigma_\la^{-1}\Sigma W_\star,
\qquad \Sigma_\la=\Sigma+\la I,
\]
and hence
\begin{equation}
\label{apprx_bound_koop}
\nor{\Sigma^{1/2}(W_\la-W_\star)}_{\cL_2(\cH)}^2
\le \la\nor{W_\star}_{\cL_2(\cH)}^2.
\end{equation}

The deterministic estimation bound~\eqref{est_bound_vv} also applies directly. In particular, whenever
\[
\la\ge 2\nor{\wh\Sigma-\Sigma},
\]
we have
\begin{equation}
\label{est_bound_koop}
\nor{\Sigma^{1/2}(\wh W_\la-W_\la)}_{\cL_2(\cH)}
\le \frac{1}{\sqrt{\la}}
\left(
\nor{\wh h-h}_{\cL_2(\cH)}
+ \nor{\wh\Sigma-\Sigma} \nor{W_\star}_{\cL_2(\cH)}
\right).
\end{equation}

It remains to control the empirical deviations. The analysis of
$\nor{\wh\Sigma-\Sigma}$ is unchanged from Lemma~\ref{lem:probbounds}. To control
$\nor{\wh h-h}_{\cL_2(\cH)}$, consider
\[
\xi_1:\cX\times\cX\to\cL_2(\cH),
\qquad \xi_1(x,x') = \Phi(x)\otimes\Phi(x')-h.
\]
Adding and subtracting $\Phi(x)\otimes g_\star(x)$ gives
\[
\xi_1(x,x')
= \underbrace{
\Phi(x)\otimes\bigl(\Phi(x')-g_\star(x)\bigr)
}_{\zeta_1(x,x')}
+ \underbrace{
\bigl(\Phi(x)\otimes g_\star(x)-h\bigr)
}_{\eta_1(x)}.
\]
As in the vector-valued case, the random variables
\[
d_t=\zeta_1(X_t,X_{t+1})
\]
form a martingale difference sequence in $\cL_2(\cH)$. Moreover, using
\(
\nor{u\otimes v}_{\cL_2(\cH)}
= \nor{u}\nor{v},
\)
Conditions~\eqref{boundphi} and~\eqref{res_bound_koop} imply that, for every $m\ge 2$,
\[
\bbE\left[\nor{d_t}_{\cL_2(\cH)}^m\mid\cF_{t-1}\right]
\le \kappa^m\bbE\left[\nor{\Phi(X_{t+1})-g_\star(X_t)}_\cH^m\mid X_t\right]
\le \frac{m!}{2}(\kappa\sigma)^2(\kappa M)^{m-2}.
\]
Then, Theorem~\ref{thm:pinelis} yields the same martingale bound as in Lemma~\ref{lem:probbounds}.
Furthermore, Condition~\eqref{boundphi} implies that for all $x\in \cX$
\[
\nor{g_\star(x)}_\cH\le\kappa,
\]
and hence
\[
\nor{\eta_1(x)}_{\cL_2(\cH)}
\le 2\kappa^2.
\]
Since $\pi(\eta_1)=0$, Theorem~\ref{thm:ugeimplieshoeffding} can be applied to the second term. Combining the two estimates gives, with probability at least $1-2\delta/3$,
\begin{equation}
\label{h_bound_koop}
\nor{\wh h-h}_{\cL_2(\cH)}
\le
\frac{4C\kappa^2}{1-\rho}
\left(
\sqrt{\frac{2\log(6/\delta)}{T}} + \frac{1}{T} \right)
+ \frac{2\kappa M\log(6/\delta)}{T}
+ \kappa\sigma \sqrt{\frac{2\log(6/\delta)}{T}}.
\end{equation}

Together with the bound for $\nor{\wh\Sigma-\Sigma}$ in Lemma~\ref{lem:probbounds} and the analytic estimate~\eqref{est_bound_koop}, this yields the final estimation error bound. In particular, if
\[
\la
\ge \frac{8C\kappa^2}{1-\rho} \left( \sqrt{\frac{2\log(6/\delta)}{T}} + \frac{1}{T} \right),
\]
then, with probability at least $1-\delta$,
\begin{equation}\label{est_err_koop}
\nor{\Sigma^{1/2}(\wh W_\la-W_\la)}_{\cL_2(\cH)}^2
\le \frac{C^2D_3}{(1-\rho)^2}
\left( \frac{2\log(6/\delta)}{\la T} + \frac{1}{\la T^2} \right)
+ \frac{D_1\log^2(6/\delta)}{\la T^2}
+ \frac{D_2\log(6/\delta)}{\la T}.
\end{equation}
where $D_1:=12\kappa^2M^2$, $D_2:=6\kappa^2\sigma^2$, and $D_3:=96\kappa^4\bigl(1+\nor{W_\star}_{\cL_2(\cH)}\bigr)^2$.
Finally, combining~\eqref{est_err_koop} and~\eqref{apprx_bound_koop} in~\eqref{dec_koop} yields~\eqref{eq:sample-rate-bound-simplified_koop}. The final rate follows by choosing
\(
\la\asymp T^{-1/2},
\)
to balance the error contributions.
\end{proof}
\end{document}

%% file: commands.tex
\def\mathscr{\EuScript}

\newcommand{\bbN}{\mathbb{N}}                               \newcommand{\bbR}{\mathbb{R}}                               
\newcommand{\abs}[1]{\left|#1\right|}                       

\newcommand{\argmin}{\mathop{\arg\min}}                     
\newcommand{\argmax}{\mathop{\arg\max}}                     

\makeatletter
\def\va@a{\boldsymbol{\va@arg^{\textstyle\text{\unboldmath$\scriptstyle\va@expo$}}_{\textstyle\text{\unboldmath$\scriptstyle\va@index$}}}}
\def\va#1{\def\va@expo{}\def\va@index{}\def\va@arg{\uppercase{#1}}%
  \@ifnextchar^{\va@h}{\@ifnextchar_\va@u\va@a}}
\def\va@h^#1{\def\va@expo{#1}\@ifnextchar_\va@hu\va@a}
\def\va@u_#1{\def\va@index{#1}\@ifnextchar^\va@uh\va@a}
\def\va@hu_#1{\def\va@index{#1}\va@a}
\def\va@uh^#1{\def\va@expo{#1}\va@a}
\makeatother

\ifdefined\proof\else \newenvironment{proof}{\small{\bf Proof.}}{\hfill$\Box$\normalsize\bigskip} \fi
\newcommand{\finpreuvesymb}{$\Box$}
\newcommand{\finremarksymb}{$\Diamond$}
\newcommand{\finexemplesymb}{$\triangle$}
\newcommand{\finpreuve}{\ \hspace*{\fill}\finpreuvesymb}
\newcommand{\finremark}{\ \hspace*{\fill}\finremarksymb}
\newcommand{\finexemple}{\ \hspace*{\fill}\finexemplesymb}

\makeatletter
\ifdefined\endproof\else \def\endproof{\finpreuve\@endtheorem}\fi
\def\endremark{\finremark\@endtheorem}
\def\endexample{\finexemple\@endtheorem}
\makeatother

